\documentclass[acmsmall,authorversion,screen=true,nonacm=true,timestamp=false]{acmart}
\setcopyright{none} 
\copyrightyear{2020}

\usepackage{multirow}
\newcommand{\jcal}{\mathcal J}
\newcommand{\fcal}{\mathcal F}
\newcommand{\gcal}{\mathcal G}
\newcommand{\acal}{\mathcal A}
\newcommand{\ncal}{\mathcal N}
\newcommand{\cH}{\mathcal {H}}
\newcommand{\cE}{\mathcal{E}}
\newcommand{\cF}{\mathcal{F}}
\newcommand{\cI}{\mathcal{I}}
\newcommand{\cX}{\mathcal{X}}
\newcommand{\cD}{\mathcal{D}}
\newcommand{\sR}{\mathcal{R}}
\newcommand{\sX}{\mathcal{X}}
\newcommand{\bN}{\mathbb{N}}
\newcommand{\bZ}{\mathbb{Z}}
\newcommand{\bR}{\mathbb{R}}
\newcommand{\R}{\mathbb{R}}
\newcommand{\N}{\mathbb{N}}

\DeclareMathOperator{\TV}{TV}
\DeclareMathOperator{\Span}{span}
\DeclareMathOperator{\polylog}{polylog}
\DeclareMathOperator{\poly}{poly}
\DeclareMathOperator{\trace}{Tr}
\newcommand{\eps}{\varepsilon}
\renewcommand{\epsilon}{\varepsilon}

\newcommand{\muhat}{\widehat{\mu}}
\newcommand{\vhat}{\widehat{v}}
\newcommand{\sigmahat}{\widehat{\sigma}}
\newcommand{\Psihat}{\widehat{\Psi}}
\newcommand{\equidist}{\overset{d}{=}}
\newcommand{\kmix}{\ensuremath{k\textnormal{\normalfont-mix}}}
\newcommand{\vc}{\operatorname{VC-dim}}
\newcommand{\transpose}{^{\mathsf{T}}}
\renewcommand{\det}{\operatorname{det}}
\newcommand{\conv}{\operatorname{conv}}

\newcommand{\prob}{\mathbf{Pr}}
\renewcommand{\Pr}{\mathbf{Pr}}
\newcommand{\prb}[1]{\mathbf{Pr}\left[{#1}\right]}
\newcommand{\expect}{\mathbf{E}}
\newcommand{\E}{\mathbf{E}}

\renewcommand{\tilde}[1]{\widetilde{#1}}
\renewcommand{\hat}[1]{\widehat{#1}}
\newcommand{\set}[1]{\left \{ #1 \right \}}                    
\newcommand{\abs}[1]{\lvert #1 \rvert}
\newcommand{\card}[1]{\abs{#1}}
\newcommand{\norm}[1]{\left\lVert #1 \right\rVert}

\newcommand{\setst}[2]{\left\{\; #1 \,:\, #2 \;\right\}} 
\newcommand{\smallfrac}[2]{{\textstyle \frac{#1}{#2}}}
\newcommand{\inner}[2]{\left\langle\: #1 ,\, #2 \:\right\rangle}
\renewcommand{\smallsum}[2]{{\textstyle \sum_{#1}^{#2}}}
\newcommand{\DKL}[2]{\operatorname{KL}\left(#1\parallel#2\right)}
\newcommand{\DLD}[2]{\operatorname{LD}\left(#1, #2\right)}
\newcommand{\DTV}[2]{\operatorname{TV}\left(#1, #2\right)}

\newcommand{\integral}[4]{ \int_{#1}^{#2} {#3} \, \mathrm{d} {#4} }

\theoremstyle{acmplain}
\newtheorem{theorem}{Theorem}[section]
\theoremstyle{acmdefinition}
\newtheorem{remark}[theorem]{Remark}
\newcommand{\LemmaName}[1]{\label{lem:#1}}
\newcommand{\Lemma}[1]{Lemma~\ref{lem:#1}}
\newcommand{\Theorem}[1]{Theorem~\ref{thm:#1}}
\newcommand{\TheoremName}[1]{\label{thm:#1}}
\newcommand{\Section}[1]{Section~\ref{sec:#1}}
\newcommand{\SectionName}[1]{\label{sec:#1}}

\begin{document}

\title[Sample Complexity of Learning Mixtures of Gaussians]
{Near-optimal Sample Complexity Bounds for Robust Learning of Gaussian Mixtures via Compression Schemes}
\thanks{A preliminary version of this paper appeared in the proceedings of NeurIPS 2018~\cite{gaussian_mixture_conference_version}. In this full version, we have strengthened the results from realizable learning to agnostic (robust) learning, improved polylogarithmic factors, and included all the proofs. {This technical report has two appendices~\ref{secupperboundsingle} and~\ref{sec:constantcompression} in addition to the journal version.}}

\author{Hassan Ashtiani}
\orcid{0000-0003-1758-7330}
\affiliation{
  \department{Department of Computing and Software}
  \institution{McMaster University}
  \streetaddress{1280 Main Street West}
  \city{Hamilton}
  \state{Ontario}
  \postcode{L8S 4L7}
  \country{Canada}
}
\affiliation{
	\institution{Vector Institute}
	\streetaddress{661 University Avenue}
	\city{Toronto}
	\state{Ontario} 
	\postcode{M5G 1M1}
	\country{Canada}
}
\email{zokaeiam@mcmaster.ca}

\author{Shai Ben-David}
\affiliation{
  \department{School of Computer Science}
  \institution{University of Waterloo}
  \streetaddress{200 University Avenue West}
  \city{Waterloo}
  \state{Ontario}
  \postcode{N2L 3G1}
  \country{Canada}
}
\email{shai@uwaterloo.ca}

\author{Nicholas J. A. Harvey}
\orcid{0000-0001-5593-9785}
\affiliation{
  \department{Department of Computer Science}
  \institution{University of British Columbia}
  \streetaddress{2366 Main Mall}
  \city{Vancouver}
  \state{British Columbia}
  \postcode{V6T 1Z4}
  \country{Canada}
}
\email{nickhar@cs.ubc.ca}

\author{Christopher Liaw}
\affiliation{
  \department{Department of Computer Science}
  \institution{University of British Columbia}
  \streetaddress{2366 Main Mall}
  \city{Vancouver}
  \state{British Columbia}
  \postcode{V6T 1Z4}
  \country{Canada}
}
\email{cvliaw@cs.ubc.ca}

\author{Abbas Mehrabian}
\orcid{0000-0002-0658-7709}
\affiliation{
  \department{School of Computer Science}
  \institution{McGill University}
  \streetaddress{3480 University Street}
  \city{Montr\'eal}
  \state{Qu\'ebec}
  \postcode{H3A 0E9}
  \country{Canada}
}
\email{abbas.mehrabian@gmail.com}

\author{Yaniv Plan}
\affiliation{
  \department{Department of Mathematics}
  \institution{University of British Columbia}
  \streetaddress{1984 Mathematics Road}
  \city{Vancouver}
  \state{British Columbia}
  \postcode{V6T 1Z2}
  \country{Canada}
}
\email{yaniv@math.ubc.ca}

\begin{abstract}
We introduce a novel technique for distribution learning
based on a notion of \emph{sample compression}.
Any class of distributions that allows such a compression scheme
can be learned with few samples.
Moreover, if a class of distributions has such a compression scheme, then so do the classes of \emph{products}
and \emph{mixtures} of those distributions.

As an application of this technique, we prove that $\tilde{\Theta}(k d^2 / \eps^2)$ samples are necessary and sufficient for learning a mixture of $k$ Gaussians in $\bR^d$,
up to error $\eps$ in total variation distance.
This improves both the known upper bounds and lower bounds for this problem.
For mixtures of axis-aligned Gaussians,
we show that $\tilde{O}(k d / \eps^2)$ samples suffice,
matching a known lower bound.
Moreover, these results hold in an agnostic learning
(or robust estimation) setting, in which the target distribution is only approximately a mixture of Gaussians.
Our main upper bound is proven by showing that the class of Gaussians in $\bR^d$ admits a small compression scheme.
\end{abstract}

\ccsdesc[500]{Mathematics of computing~Density estimation}
\ccsdesc[500]{Theory of computation~Sample complexity and generalization bounds}
\ccsdesc[100]{Theory of computation~Unsupervised learning and clustering}
\ccsdesc[100]{Mathematics of computing~Multivariate statistics}

\keywords{compression schemes, density estimation, distribution learning, mixtures of Gaussians}

\maketitle
\renewcommand{\shortauthors}{Ashtiani, Ben-David, Harvey, Liaw, Mehrabian, and Plan}

\section{Introduction}

Estimating distributions from observed data is a fundamental
task in statistics, which has been studied for over a century.
This task frequently arises in applied machine learning, commonly assuming that the distribution can be modeled approximately by a mixture of Gaussians.
Popular software packages have implemented heuristics, such as the expectation maximization (EM) algorithm,
for learning a mixture of Gaussians.
The theoretical machine learning community  has a rich literature on distribution learning as well; for example, the recent survey~\cite{Diakonikolas2016} considers learning structured distributions,
and the survey~\cite{KMV} focuses on mixtures of Gaussians.

This paper develops a general technique for distribution learning,
then employs this technique in the
canonical setting of Gaussian mixtures.
The learning model we adopt is \emph{density estimation}:\
given i.i.d.\ samples from an unknown target distribution,
find a distribution that is close to the target  in \emph{total variation (TV) distance}.
Our analysis focuses on sample complexity rather than computational complexity.
That is, we seek a learning algorithm that obtains a good estimate of the target distribution using as few samples as possible, but we do not worry about its running time.
For background on this model, see, e.g., \cite{devroye_book,Diakonikolas2016}. 

Our new technique for proving upper bounds on the sample complexity involves a novel form of \emph{sample compression}: if it is possible to ``encode'' each member of a class of distributions using a carefully chosen subset of its samples,
then we obtain an upper bound on the sample complexity of distribution learning for that class.
In particular, by constructing compression schemes for mixtures of axis-aligned Gaussians and general Gaussians, we
obtain new upper bounds on sample complexities of learning with respect to these classes.
Furthermore, we prove that these new bounds are optimal up to polylogarithmic factors.

\subsection{The distribution learning framework}

A \emph{distribution learning method} or \emph{density estimation method} is an algorithm that takes as input a sequence of i.i.d.\ samples generated from a distribution $g$, and outputs (a description of) a distribution $\hat{g}$  as an estimate for $g$.
We work with absolutely continuous distributions in this paper (i.e., distributions that have a density with respect to the Lebesgue measure), so we identify a probability distribution with its probability density function.
The \emph{total variation (TV) distance} 
between two probability distributions $f_1$ and $f_2$ over $\R^d$
is defined as
\begin{equation}
\label{eq:TVdef}
\DTV{f_1}{f_2}
 \:\coloneqq\: \sup_{B \subseteq \R^d} \int_B \big(f_1(x) - f_2(x)\big) \,\mathrm{d} x
 \:=\:
 \frac{1}{2}\|f_1 - f_2\|_1 \:,
\end{equation}
where 
\(\|f\|_1\coloneqq \int_{\bR^d} |f(x)|\,\mathrm{d} x\)
is the \emph{$L^1$ norm} of $f$,
and $\|f_1 - f_2\|_1$ is the \emph{$L^1$ distance} between $f_1$ and $f_2$.
(Strictly speaking, the supremum in~\eqref{eq:TVdef} should be over the Borel sigma algebra on $\R^d$;
however, we do not worry about such measure-theoretic issues in this paper.)
Sometimes we  write $\DTV{X}{Y}$, where $X$ and $Y$ are random variables rather than distributions.
In the following definitions,  $\mathcal{F}$ is a class of probability distributions, and $g$ is a distribution not necessarily in $\mathcal{F}$. 

\begin{definition}[$\eps$-approximation, $\eps$-close, $(\eps, C)$-approximation]
We say a distribution $\hat{g}$ is an \emph{$\eps$-approximation} for $g$,
or $\hat{g}$ is \emph{$\eps$-close} to $g$, if $ \|\hat{g}- g\|_1 \leq \eps$.
	A distribution $\hat{g}$ is an \emph{$(\eps, C)$-approximation} for $g$ with respect to $\mathcal{F}$ if 
	\[ \|\hat{g}- g\|_1 ~\leq~ C\cdot \inf_{f\in \mathcal{F}}\|f-g\|_1 + \eps.\]
\end{definition}

\begin{definition}[PAC-learning distributions, realizable setting]\label{def:realizablelearning}
	A distribution learning method is called a \emph{(realizable) PAC-learner} for $\mathcal{F}$ with sample complexity $m_{\mathcal{F}}(\eps, \delta)$ if, for all distributions $g\in\mathcal F$ and all $\eps, \delta \in(0,1)$, given $\eps$, $\delta$, and an i.i.d.\ sample of size $m_{\mathcal{F}}(\eps, \delta)$ from  $g$, outputs an $\eps$-approximation of $g$, with probability at least $1-\delta$ (over the samples and the algorithm's randomness).
\end{definition}

\begin{definition}[PAC-learning distributions, agnostic setting]\label{def1}
	A distribution learning method is called a \emph{$C$-agnostic PAC-learner for $\mathcal{F}$} with sample complexity $m_{\mathcal{F}}^C(\eps, \delta)$ if, for all distributions $g$ and all $\eps, \delta \in(0,1)$, given $\eps$, $\delta$, and a sample of size $m_{\mathcal{F}}^C(\eps, \delta)$ generated i.i.d.\ from $g$, outputs an $(\eps, C)$-approximation of $g$ with respect to ${\mathcal F}$, with probability at least $1-\delta$.
\end{definition}

The statement that a class can be ``$C$-learned in the agnostic setting'' means there exists a $C$-agnostic PAC-learner for the class. The case $C>1$ is sometimes called \emph{semi}-agnostic learning in the learning theory literature.
Note that minimizing the $L^1$ distance is equivalent to minimizing the TV distance, as the former is twice the latter.
Let 
\(
\Delta_n \coloneqq \{\: (w_1,\dots,w_n) \in \R^n \,:\, w_i\geq 0 ,\, \sum w_i=1 \:\}
\)
denote the $n$-dimensional simplex.

\begin{definition}[$\kmix(\mathcal{F})$]\label{kmixdef}
	Let $\mathcal{F}$ be a class of probability distributions. Then the class of $k$-mixtures of $\mathcal{F}$, written $\kmix(\mathcal{F}$), is defined as 
	$$
	\kmix(\fcal) ~\coloneqq~ \{\: \smallsum{i=1}{k} w_{i}f_{i} \::\:
	(w_1,\dots,w_k)\in \Delta_k ,\, f_1,\dots,f_k\in\mathcal F \:\}.
	$$
\end{definition}

\subsection{Sample compression schemes}
\SectionName{compression:intro}

In this paper, we introduce a method for learning distributions via a novel form of \emph{compression}. Given a class $\fcal$ of distributions, suppose there is a method for ``compressing'' information about any distribution $f\in\fcal$ using a subset of samples from $f$ and some additional bits.
Further, suppose there exists a fixed, deterministic \emph{decoder} for $\fcal$, which, given the subset of samples and the additional bits, approximately recovers $f$.
If the size of the subset and the number of bits is guaranteed to be small, we show that the sample complexity of learning $\fcal$ is small as well.

More precisely, we say class $\fcal$  \emph{admits $(\tau, t, m)$ compression} if there exists a \emph{decoder function} such that, upon generating $m$ i.i.d.\ samples from any $f\in\fcal$, we are guaranteed, with probability at least $2/3$, to have $\tau$ data points from the sample and a sequence of at most $t$ bits on which the decoder outputs a distribution that is within total variation distance of $\eps$ from $f$.
Note that $\tau,t,$ and $m$ may be functions of $\eps$, the accuracy parameter.
The decoder function is specific to the class $\fcal$ but  does not depend on the particular $f$.

This definition is further generalized to a stronger notion of \emph{robust compression}, formally defined in Definition~\ref{def_robustcompression}, where the target distribution is to be encoded using samples that are not necessarily generated from the target itself but are generated from a distribution  close to the target.
More precisely,  class $\fcal$ \emph{admits $(\tau, t, m)$ robust compression} if there exists a \emph{decoder function}, such that for any $f\in\fcal$, upon generating $m$ i.i.d.\ samples from any distribution that is ``close'' to $f$, we are guaranteed, with probability at least $2/3$, to have $\tau$ data points from the sample and a sequence of at most $t$ bits on which the decoder outputs a distribution that is within total variation distance of $\eps$ from $f$.
We prove that robust compression implies agnostic learning. In particular, if $\fcal$ admits $(\tau, t, m)$ robust compression, then the sample complexity of agnostic learning with respect to $\fcal$ is bounded by 
${\widetilde O}(m + (\tau + t) / \eps^2)$ (Theorem~\ref{thm:compression}).
(${\widetilde O}$ allows for polylogarithmic factors.)

To illustrate the compression technique, it is instructive to compare it with another density estimation technique, called the {\em cover method} in~\cite[Section~1.5]{Diakonikolas2016}.
Let $(\cD, d)$ be a metric space and let $\fcal\subseteq \cD$. Given $\eps>0$, a subset ${C}\subseteq\cD$ is called an \emph{$\eps$-net} for $\fcal$ if for any $f\in \fcal$ there exists some $c\in C$ such that $d(f,c) \leq \eps$. The size of the smallest $\eps$-net is called the \emph{covering number} of $\fcal$ and its logarithm is called the \emph{metric entropy} of $\fcal$. (In this informal discussion, we are ignoring the dependence on $\eps$ for brevity.) The metric entropy is a measure of ``dimension'' for the set $\fcal$.

The cover method for learning a class $\fcal$ of distributions over domain $Z$ works as follows: view $\fcal$ as a subset of the class of all distributions over $Z$ equipped with the $L^1$ metric.
Consider a small $\eps$-net of $\fcal$ and then choose the distribution in the net that is closest to the target distribution in the $L^1$ distance.
If the net is finite, then this last step is equivalent to learning the closest distribution among finitely many candidates to a target distribution, which has an algorithm whose  sample complexity scales logarithmically with the net size (see Theorem~\ref{thm:candidates}).

The issue with using the cover method for learning the class of Gaussians is that the covering number and the metric entropy of the class are infinite with respect to the $L^1$ metric. 
In fact, even for the class of mean-zero Gaussians with entry-wise bounded  covariance matrices, the metric entropy is
infinite unless we assume a bound on the condition number of the covariance matrices
(consider 2-dimensional singular Gaussians each supported on a different 1-dimensional subspace with unit variance on that subspace).

The power of compression schemes is that they take a data-dependent approach: a first round of sampling is used to shrink the space of feasible distributions significantly, making the metric entropy of the resulting feasible set finite. This allows us to achieve bounds for Gaussians which do not depend on the condition number or the size of the parameters (Lemma~\ref{lem:coreformixtures}).

An additional attractive property of compression is that it enjoys two closure properties.
Specifically, if a base class admits compression, 
then 
the class of products of the base class, 
as well as 
the class of mixtures of that base class, 
are compressible (lemmata~\ref{lem:product_compress} and~\ref{lem:compressmixtures}). 

As an application of this technique, we prove tight (up to logarithmic factors) sample complexity upper bounds of $\widetilde{O}(kd^2/\eps^2)$ for learning mixtures of $k$ Gaussians  over $\R^d$ (\Theorem{upper_bound})
and
$\widetilde{O}(kd/\eps^2)$ for learning mixtures of $k$ axis-aligned Gaussians (\Theorem{upper_bound_axis_aligned}).

In light of the closure properties of compression schemes, we need only provide a compression scheme for the class of Gaussian distributions to  obtain a compression scheme (and a sample complexity bound) for  mixtures of Gaussians. 
We prove that the class of $d$-dimensional Gaussian distributions admits $(\widetilde O(d), \widetilde O(d^2), \widetilde O(d)) $ robust compression (Lemma~\ref{lem:coreformixtures}).
The high-level idea is that by generating $\widetilde{O}(d)$ samples from a Gaussian, one can get a rough sketch of the geometry of the Gaussian.
In particular, the points drawn from a Gaussian concentrate around an ellipsoid centered at the mean and whose principal axes are the eigenvectors of the covariance matrix.
Using ideas from 
convex geometry
and random matrix theory, we show one can encode the center of the ellipsoid \emph{and} the principal axes using linear combinations of these samples.
Then we discretize the coefficients and obtain an approximate encoding.

Our compression framework is quite flexible and can be used to prove sample complexity upper bounds for other distribution classes as well. This is left for future work.

\subsection{Main results}\label{sec:main}
Our first main result is an upper bound for learning mixtures of multivariate Gaussians. This bound is tight up to logarithmic factors.
Let $k$ denote the number of mixture components and $d$ denote the dimension.
Henceforth, the notations $\widetilde O(\cdot)$ and $\widetilde \Omega(\cdot)$ suppress $\polylog(k d/\eps\delta)$ factors; there are no hidden dependencies on any other parameters (such as the condition number) when we use the $O$, $\Omega$, $\tilde{O}$ and $\tilde{\Omega}$ notation.

\begin{theorem}
\TheoremName{upper_bound}
The class of $k$-mixtures of $d$-dimensional Gaussians can be learned 
in the realizable setting,
and can be 12-learned in the agnostic setting,
using $\widetilde{O}(kd^2/\eps^2)$ samples.
\end{theorem}

We emphasize that the $\widetilde{O}(\cdot)$ notation
has \emph{no dependence whatsoever on the scaling, condition number, separation}, or any other structural property of the distribution.
Previously, the best known upper bounds on the sample complexity of this problem were
$\widetilde{O}(kd^2/\eps^4)$, due to \cite{ashtiani2017sample},
and $O(k^4d^4/\eps^2)$, based on a VC~dimension bound  that we discuss below.
For the case of a single Gaussian (i.e., $k=1$),
a bound of $O(d^2/\eps^2)$ is known, again using a VC~dimension bound discussed below.

Our second main result is a lower bound matching \Theorem{upper_bound} up to logarithmic factors.

\begin{theorem}
\TheoremName{lower_bound}
Any method for learning the class of $k$-mixtures of $d$-dimensional Gaussians in the realizable setting has sample complexity $\widetilde{\Omega}(kd^2/\eps^2)$.
\end{theorem}

Note that this is a worst-case (i.e., minimax) lower bound: for any estimation method, \emph{there exists at least one distribution} which requires that many samples.
Previously, the best known lower bound on the sample complexity was $\widetilde{\Omega}(kd/\eps^2)$ \cite{spherical}.
Even for a single Gaussian (i.e., $k=1$),
an $\widetilde\Omega(d^2/\eps^2)$ lower bound was not known prior to this work.

Our third main result is an upper bound for learning mixtures of \emph{axis-aligned} Gaussians,
i.e., Gaussians with diagonal covariance matrices. This bound is also tight up to logarithmic factors.

\begin{theorem}
\TheoremName{upper_bound_axis_aligned}
The class of $k$-mixtures of axis-aligned $d$-dimensional Gaussians can be learned 
in the realizable setting,
and can be 12-learned in the agnostic setting,
using $\widetilde{O}(kd/\eps^2)$ samples.
\end{theorem}

A matching lower bound of $\widetilde{\Omega}(kd/\eps^2)$ was proved in \cite{spherical}.
Previously, the best known upper bounds were $\widetilde{O}(kd/\eps^4)$,
due to \cite{ashtiani2017sample}, and $O((k^4d^2+k^3d^3)/\eps^2)$, based on a VC~dimension bound that we discuss below.

In the agnostic results of \Theorem{upper_bound} and \Theorem{upper_bound_axis_aligned}, the constant 12 can be decreased to any constant larger than 9.
One may verify this statement through a detailed inspection of our proofs.
We omit a full derivation of this improved constant in order to avoid 
tedious details in the proofs.

\paragraph{Our techniques.}
The upper bounds are proved using the compression technique discussed in \Section{compression:intro}.
Next we discuss the main ideas used in the proof of our lower bound, \Theorem{lower_bound}.
In order to prove our lower bound for mixtures of Gaussians, we first prove a lower bound of $\widetilde \Omega(d^2/\eps^2)$ for learning a single Gaussian.
The main step is to construct a large family, of size $2^{\Omega(d^2)}$, of covariance matrices such
that the associated Gaussian distributions are well-separated in terms of their total variation
distance, while simultaneously ensuring that their mutual Kullback-Leibler divergences are small.
Once this is established, we can then apply a generalized form of Fano's inequality to complete the proof.

To construct this family of covariance matrices, we sample $2^{\Omega(d^2)}$ matrices
from the following probabilistic process: start with an identity covariance matrix; then choose a uniformly random subspace of dimension $d/9$ and slightly increase the eigenvalues corresponding to this eigenspace.
It is easy to bound the KL divergences between the constructed Gaussians.
Quantifying this gap will then give the desired lower bound on the total variation distance.

\label{rates}
\paragraph{Minimax estimation rates.}
Our results are stated in terms of sample complexity, which is the terminology mostly used in the machine learning literature.
It is also possible to state our results in terms of minimax estimation rates,
which are often used in the statistics literature.
There is a direct connection between sample complexity bounds and estimation rates, although translating between them sometimes incurs logarithmic factors.

We recall some definitions from the minimax estimation framework (see, e.g., \cite[Chapter~2]{Tsybakov}).
Let $\fcal$ be a class of probability distributions defined on domain $\sX$.
The \emph{risk} of a density estimation method
$\hat{f}:\sX^n \to \fcal$ for this class is defined as 
\[
\sR_n(\hat{f}, \fcal) ~\coloneqq~
\sup_{f \in \fcal} 
\E \TV(\hat{f}(X_1,\dots,X_n), f)  ,
\]
where the expectation is over the i.i.d.\ samples $X_1,\dots,X_n$ from $f$, and
possible randomization of the estimator. 
The \emph{minimax estimation rate} for
$\fcal$ is the smallest risk over all possible estimators $\hat{f}:\sX^n \to \fcal$, i.e.,
\[
\sR_n(\fcal) ~\coloneqq~
\inf_{\hat{f}} \sR_n(\hat{f}, \fcal)
~=~
\inf_{\hat{f}} 
\sup_{f \in \fcal} 
\E \TV(\hat{f}(X_1,\dots,X_n), f) .
\]

Let $\gcal_{d,k}$ denote the class of $k$-mixtures of $d$-dimensional Gaussian distributions, 
and let $\acal_{d,k}$ denote the class of $k$-mixtures of $d$-dimensional axis-aligned Gaussian distributions. 
Then \Theorem{upper_bound} implies the minimax estimation rate of $\gcal_{d,k}$ is $\widetilde{O}(\sqrt{kd^2/n})$; indeed,
the theorem states that to get error $\leq \eps$ we need some $n\leq \polylog(kd/\eps) kd^2/\eps^2$ many samples,
hence, solving for $\eps$, we find that given $n$ samples we obtain an estimator with error
$\eps \leq \polylog(kdn) \sqrt{kd^2/n}$.
Similarly,
\Theorem{upper_bound_axis_aligned} implies the minimax estimation rate of $\acal_{d,k}$ is $\widetilde{O}(\sqrt{kd/n})$.
Note that these theorems indeed imply stronger statements than these risk bounds, since they give guarantees for the case where the target distribution does not necessarily belong to the known class $\fcal$, a setting not captured by the minimax framework.

Finally, the proof of \Theorem{lower_bound} (see Theorem~\ref{thm:lbmixture} below) implies that the minimax rate of $\gcal_{d,k}$ is 
$\Omega\left(\sqrt{k d^2 / n} \right)$,
improving the 
$\Omega\left(\sqrt{k d^2} /\sqrt n \log n \right)$
lower bound proved in the preliminary version of this paper~\cite{gaussian_mixture_conference_version}.

\paragraph{Computational efficiency.}
Although our approach for proving sample complexity upper bounds is algorithmic,
our focus is not on computational efficiency.
The resulting algorithms have nearly optimal sample complexities, but their running times are exponential in the dimension $d$ and the number of mixture components $k$.
More precisely, the running time is $2^{kd^2 \polylog(d, k, 1/\eps,1/\delta)}$ for mixtures of general Gaussians, and
$2^{kd \polylog(d, k, 1/\eps,1/\delta)}$ for mixtures of axis-aligned Gaussians.
The existence of an algorithm for density estimation that runs in time $\poly(k, d)$ is unknown even for the class of mixtures of axis-aligned Gaussians, see \cite[Question 1.1]{gaussian_mixture}.

Even for the case of a single Gaussian, the
published proofs of the $O(d^2/\eps^2)$ bound, of which we are aware, are not algorithmically efficient, e.g., \cite[Theorem 13]{ashtiani2017sample}.
Adopting ideas from our proof of \Theorem{upper_bound}, an algorithmically efficient learner for a single Gaussian can be obtained simply by computing the empirical mean and (an appropriate estimate of the) covariance matrix using $O(d^2/\eps^2)$ samples.
The details appear in Appendix~\ref{secupperboundsingle}.

\paragraph{Paper outline.}
Next, we review some related work. Then we introduce our notation and recall some standard facts in \Section{prelim}.
In \Section{justification}, we provide justification for our learning model.
In \Section{compression}, we formally define compression schemes for distributions, prove their closure properties, and show their connection with density estimation.
Theorems~\ref{thm:upper_bound} and~\ref{thm:upper_bound_axis_aligned} 
are proved in \Section{upper_bound_simpler}.
\Theorem{lower_bound} is proven in Section~\ref{sec:lower_bound}.
All omitted proofs appear in the appendices.

\subsection{Related work}
\SectionName{prior}

Distribution learning is a vast topic and many approaches have been considered in the literature.
This section reviews the approaches that are particularly relevant to our work.

For parametric families of distributions, a common approach is to estimate the parameters
of the distribution, in a maximum likelihood sense, or aiming to approximate the true parameters.
For mixtures of Gaussians, there is a rich theoretical literature on
algorithms that approximate the mixing weights, means and covariances (e.g., \cite{arora,Belkin,dasgupta1999learning,moitravaliant}); see
\cite{KMV} for a survey.
The strictness of this objective cuts both ways.
On the one hand, a successful learner uncovers substantial structure of the target distribution.
On the other hand, this objective is  impossible when the means and covariances are extremely close.
Thus, algorithms for parameter estimation of Gaussian mixtures necessarily require some separation assumptions on the target parameters.

Density estimation has a long history in  statistics, where the focus is on the sample complexity question;
see \cite{devroye_density_estimation_first,devroye_book,silverman} for general background.
Density estimation was first studied in the computational learning theory community under the name \emph{PAC learning of distributions}  in \cite{Kearns},
whose focus is on the computational complexity of the learning problem.

Various measures of dissimilarity between distributions have been considered in existing density estimation schemes.
One natural measure is the TV distance, which
has been adopted by several papers
on learning mixtures of Gaussians \cite{ashtiani2017sample,onedimensional,DK14}.
Another natural measure, which has also been considered for mixtures of Gaussians, is the Kullback-Leibler (KL) divergence~ \cite{axis_aligned}.
Some prior work has also used the $L^2$ distance for density estimation \cite{ADHLS,DGLNOS}.
This paper focuses on the TV distance (i.e., the $L^1$ distance), and we provide justification for this choice in \Section{justification}.

A popular method for distribution learning in practice is kernel density estimation (see, e.g., \cite[Chapter~9]{devroye_book}).
The rigorously proven sample complexity/estimation rate upper bounds for this method require either 
smoothness assumptions (e.g., \cite[Theorem~9.5]{devroye_book})
or boundedness assumptions
(e.g., \cite[Theorem~2.2]{ibragimov})
on the class of densities.
The class of Gaussians is not universally Lipschitz or universally bounded, so those results do not apply to the problems we consider.
Moreover, numerical calculations demonstrate that the number of samples required to estimate a standard Gaussian (within $L^2$ distance 0.1 using Gaussian kernels) grow exponentially with the dimension (see~\cite[Table~4.2 on page~94]{silverman}), which hints that this method suffers from the curse of dimensionality.

Another
approach for deriving  sample complexity upper bounds for distribution learning, called the \emph{minimum distance estimate} in~\cite[Chapter~6]{devroye_book},
is based on the uniform convergence theory and the notion of \emph{Vapnik-Chervonenkis dimension} (Definition~\ref{define:vc}).
It is proved in ~\cite{devroye_book} that an upper bound for a class of distributions can be obtained by bounding the VC~dimension of an associated set system, called the \emph{Yatracos family} (Definition~\ref{def:yatracos}).
For example, \cite{logconcave} used this method to bound the sample complexity of learning high-dimensional log-concave distributions.
For learning $d$-dimensional Gaussians, this approach leads to the optimal sample complexity upper bound of $O(d^2/\eps^2)$.
However, for mixtures of Gaussians and axis-aligned Gaussians, the best known VC~dimension bounds (see \cite[Theorem~8.14]{AB99} and \cite[Section~8.5]{devroye_book})
result in loose upper bounds of $O(k^4d^4/\eps^2)$ and $O((k^4d^2+k^3d^3)/\eps^2)$, respectively.

Another approach is to first approximate the Gaussian mixture class using a more manageable class such as piecewise polynomials, and then study the associated Yatracos family; see, e.g., \cite{onedimensional}.
However, piecewise polynomials do a poor job in approximating $d$-dimensional Gaussians, resulting in an exponential dependence on $d$.

For density estimation of mixtures of Gaussians
using the TV distance,
the best known sample complexity upper bounds (in terms of $k$ and $d$) are
$\widetilde{O}(kd^2/\eps^4)$ for general Gaussians
and $\widetilde{O}(kd/\eps^4)$ for axis-aligned Gaussians,
both due to \cite{ashtiani2017sample}.
For the general Gaussian case, their method takes an i.i.d.\ sample of size $\widetilde{O}(kd^2/\eps^2)$ and partitions this sample in every possible way into $k$ subsets.
Based on those partitions, $k^{\widetilde{O}(kd^2/\eps^2)}$ ``candidate distributions'' are generated.
The problem is then reduced to learning with respect to this finite class of candidates.
Their sample complexity has a suboptimal factor of $1/\eps^4$,
of which $1/\eps^2$ arises in their approach for choosing the best candidate,
and another factor $1/\eps^2$ is due to the exponent in the number of candidates.

Our approach via compression schemes also ultimately reduces the problem to learning with respect to finite classes, although yielding a more refined bound than \cite{ashtiani2017sample}.
In our sample complexity upper bounds, one factor of $1/\eps^2$ is again incurred due to learning with respect to finite classes.
The key is that the number of compressed samples does not depend on $\eps$,
so the overall sample complexity bound has only an $\widetilde{O}(1/\eps^2)$ dependence on $\eps$.

As for lower bounds on the sample complexity 
for learning mixtures of Gaussians under the TV distance, much fewer results are known. 
The only lower bound prior to this work is due to \cite{spherical}, which shows a bound of $\widetilde{\Omega}(kd/\eps^2)$ for learning mixtures of axis-aligned Gaussians
(and hence for general Gaussians as well).
This bound is tight for the axis-aligned case, as we show in \Theorem{upper_bound_axis_aligned},
but loose in the general case, as we show in \Theorem{lower_bound}.
After the preliminary version of this paper was completed~\cite{gaussian_mixture_conference_version}, 
an alternative construction was provided in~\cite{lower_bound_improved}
giving the same lower bound as ours using a deterministic construction.

A summary of known bounds on sample complexities for learning Gaussian mixtures and their subclasses is presented in Table~\ref{table}.

\begin{table}\centering
\caption{Bounds on the sample complexities of learning Gaussian mixtures and their subclasses.
The lower bounds are minimax (i.e., worst-case).
The bounds in the first two rows are well known; proofs can be found in~\cite{ashtiani2017sample}.}
\label{table}
\begin{tabular}{c  c  c  c   c  c }
    \toprule
& Number of Gaussians & Dimension & Axis-aligned & Sample complexity & Reference \\
    \cmidrule(lr){2-6}
\parbox[t]{2mm}{\multirow{5}{*}{\rotatebox[origin=c]{90}{Upper Bounds}}}
& 1   & $d$ & no     & ${O(d^2 / \eps^2)}$ & standard \\
& 1   & $d$ & yes    & $O(d / \eps^2)$ & standard  \\
& $k$ & $1$ & n/a & $\widetilde{O}(k/\eps^2)$&
                       \cite{onedimensional} \\
& $k$ & $d$ & no     & $\widetilde{O}(kd^2 / \eps^2)$ & this paper  \\
& $k$ & $d$ & yes    & $\widetilde{O}(kd /  \eps^2)$ & this paper  \\ 
    \cmidrule(lr){2-6}
\parbox[t]{2mm}{\multirow{5}{*}{\rotatebox[origin=c]{90}{Lower Bounds}}}
& 1   & $d$ & no     & $\widetilde\Omega(d^2 /  \eps^2)$&this paper\\
& 1   & $d$ & yes    &$\widetilde\Omega(d /  \eps^2)$ & \cite{spherical}\\
& $k$ & $1$ & n/a & $\widetilde\Omega(k/\eps^2)$& \cite{spherical} \\
& $k$ & $d$ & no     & $\widetilde\Omega(kd^2 /  \eps^2)$&this paper\\
& $k$ & $d$ & yes    & $\widetilde\Omega(kd /  \eps^2)$&\cite{spherical}\\
    \bottomrule
\end{tabular}
\end{table}

\section{Preliminaries}
\SectionName{prelim}
\label{sec:formal}

\paragraph{Basic notation.}
The notation $\log(\cdot)$ denotes logarithm in the natural base, $[M]$ denotes $\{1,2,...,M\}$, and $A^c$ denotes the complement of a set $A$.
Throughout the paper, $a/bc$ always means $a/(bc)$.
For any $p\geq1$, the $L^p$ distance between functions $f$ and $g$ over $\R^d$ is defined as
\(\|f-g\|_p\coloneqq \left(\int_{\bR^d} |f(x)-g(x)|^p\,\mathrm{d} x\right)^{1/p}\), and the $\ell_p$ distance between two vectors $(x_1,\dots,x_n)$ and $(y_1,\dots,y_n)$ is defined as 
\(
(\sum_{i=1}^{n} |x_i-y_i|^p)^{1/p},
\)
and their $\ell_{\infty}$ distance is defined as $\max_i |x_i-y_i|$.

\paragraph{Probability terminology.}
For random variables $X$ and $Y$, the notation $X \equidist Y$ means that $X$ and $Y$ have the same distribution.
For a distribution $g$, we write $X\sim g$ to mean $X$ is a random variable with distribution $g$, and  $S\sim g^m$  means that $S$ is an i.i.d.\ sample of size $m$ generated from $g$.

\begin{proposition}
	\label{fact:dtv}
	Let $X$ and $Y$ be random variables taking values in the same set.
	For any function $f$, we have
	\(
	\DTV{f(X)}{f(Y)} ~\leq~ \DTV{X}{Y}.
	\)
\end{proposition}
\begin{proof}
	For any set $A$, we have \[\prb{f(X) \in A} - \prb{f(Y) \in A} = \prb{X \in f^{-1}(A)} - \prb{Y \in f^{-1}(A)} \leq \DTV{X}{Y}.\] 
	Taking the supremum of the left-hand side proves the proposition.
\end{proof}

\paragraph{Matrix terminology.}
We use $\|v\|_2$ to denote the Euclidean norm of a vector $v$, 
$\|A\|_s$ to denote the \emph{spectral norm}, or the \emph{operator norm}, of a matrix $A$, 
that is, $\|A\|_s \coloneqq \sup_{\|v\|_2=1} \|Av\|_2$,
and $\|A\|_F \coloneqq \sqrt{\trace(A\transpose A)}$ to denote the \emph{Frobenius norm} of a matrix $A$.

\paragraph{Gaussian and $\chi_d$ distributions.}
Let $d$ denote the dimension.
A Gaussian distribution with mean $\mu\in \R^d$ and covariance matrix $\Sigma \in \R^{d\times d}$ is denoted by $\ncal(\mu,\Sigma)$.
If $\Sigma$ is a diagonal matrix, then $\ncal(\mu,\Sigma)$ is called an \emph{axis-aligned} Gaussian.
If $\Sigma$ is full-rank, then $\ncal(\mu,\Sigma)$ is called a \emph{full-rank} Gaussian.
The $\ncal(0,I)$ distribution is called the standard Gaussian distribution, and an $\ncal(0,1)$ random variable is called standard normal.
It is easy to check that if $g\sim\ncal(0,I)$ then $\mu+\Sigma^{1/2}g\sim\ncal(\mu,\Sigma)$.
Let $g_1,\dots,g_n$ be i.i.d.\ standard normal;
then $\sum_{i=1}^n g_i^2$ is said to have the \emph{chi-squared distribution} with parameter $n$ and is denoted by $\chi_n$. Observe that $\E \chi_n = n$.

\begin{lemma}[{\protect\cite[Lemma 1]{ML00}}]
	\LemmaName{ChiSquaredTail}
	For any positive integer $d$,
		$\prob[ \chi_d-d \geq 2\sqrt{dt}+2t ] ~\leq~ \exp(-t).$
\end{lemma}

\begin{corollary}\label{cor:chi}
	For any positive integer $d$,
	$\prob[\chi_d \geq 16d] \leq \exp(-3)$.	
\end{corollary}

\begin{theorem}[Chernoff bound, see Theorem~4.5 in~\cite{mitzenmacher}]\label{thm:chernoff}
	Let $X = \sum_{i=1}^{n} X_i$, where $X_1,\dots,X_n$ are independent $\{0,1\}$ random variables.
	For any $0<\delta<1$ we have
	\(
	\prb{X \leq (1-\delta)\E X} \leq \exp (- \delta^2 \E X/2).
	\)
\end{theorem}

\begin{definition}[Vapnik-Chervonenkis (VC) dimension~\cite{vapnik2015uniform}]
	\label{define:vc}
	Let
	$\mathcal{A} \subseteq 2^{\mathcal{X}}$ be a family of subsets of a set $\mathcal{X}$.  The
	\emph{VC~dimension} of $\mathcal{A}$, denoted by $\vc(\mathcal{A})$, is the size of
	the largest set $X \subseteq \mathcal X$ such that for each $Y\subseteq X$
	there exists $B \in \mathcal{A}$ with $X \cap B = Y$.  
\end{definition}

For examples and applications of the VC~dimension, see, e.g.,
\cite[Chapter~4]{devroye_book}.

\begin{definition}[Yatracos family~\cite{Yatracos}]\label{def:yatracos}
	For a class of densities $\fcal$ over $\R^d$, the associated
	\emph{Yatracos family} is the following family of subsets of $\R^d$:
	\[
	\Big\{\{x \in \R^d \colon f(x) > g(x)\} \colon f, g \in \fcal \Big \}.
	\]
\end{definition}

\subsection{KL divergence, log-det divergence, and total variation distance}

\begin{definition}[Kullback-Leibler (KL) divergence~\cite{kldivergence}]
The \emph{Kullback-Leibler (KL) divergence} between densities $f_1$ and $f_2$ is defined by
\[
\DKL{f_1}{f_2} ~\coloneqq~ \int_{\R^d} f_1(x) \log \frac{f_1(x)}{f_2(x)} \mathrm{d} x,
\]
where we define $\DKL{f_1}{f_2}=+\infty$
if the set $\{\, x \,:\, f_2(x)=0 < f_1(x) \,\}$ has positive Lebesgue measure.
\end{definition}

The KL divergence is a measure of distance between distributions, which is asymmetric and does not satisfy the triangle inequality.
However, it is always nonnegative (see, e.g., \cite[Theorem~2.6.3]{cover}) and can take value $+\infty$.

\begin{definition}[log-det divergence]
\label{def:LD}
Let $A$ and $B$ be symmetric positive definite matrices of the same size.
The \emph{log-det divergence} of $A$ and $B$ is defined as
$ \DLD{A}{B} \coloneqq \trace(B^{-1}A - I) - \log \det( B^{-1} A )$.
\end{definition}

From the definition, it is apparent that $\DLD{A}{B}$ only depends on the spectrum of $B^{-1}A$.
The log-det divergence is an asymmetric measure of distance between matrices and is closely related to the KL divergence between their corresponding Gaussian distributions, as can be seen from Lemma~\ref{lem:tvupbound} below, which illustrates that the log-det divergence is always nonnegative.

\begin{lemma}
\LemmaName{SpectralToDLD}
Let $A$ and $B$ be positive definite $d\times d$ matrices satisfying 
$\|B^{-1/2}AB^{-1/2} - I\|_s \leq \alpha$
for some $\alpha \in [0,1/2]$.
Then $\DLD{A}{B} \leq  d \alpha^2$.
\end{lemma}
\begin{proof}
Since $B^{-1/2}AB^{-1/2}$ is symmetric, its eigenvalues $\lambda_1,\ldots,\lambda_d$  are real, and since we have $\|B^{-1/2}AB^{-1/2} - I\|_s \leq \alpha$, each $\lambda_i \in [1-\alpha,1+\alpha]$.
So, 
\begin{align*}
\DLD{A}{B}&
    ~=~ \trace(B^{-1}A - I) - \log \det( B^{-1} A ) 
    ~=~ \sum_{i=1}^d (\lambda_i-1) ~-~ \log \prod_{i=1}^d \lambda_i \\
    &~=~ \sum_{i=1}^d (\lambda_i - 1 - \log (\lambda_i))
    ~\leq~ \sum_{i=1}^d (\lambda_i-1)^2
    ~\leq~ d \alpha^2.
   \end{align*}
The first inequality follows from $x - 1 - \log x \leq (x-1)^2$, valid for any $x \geq 1/2$.
\end{proof}

\begin{lemma}
\label{lem:tvupbound}
For any two full-rank Gaussians $\ncal(\mu_0,\Sigma_0)$ and $\ncal(\mu_1,\Sigma_1)$, we have
\begin{align*}
2 \DTV{\ncal(\mu_0,\Sigma_0)}{\ncal(\mu_1,\Sigma_1)}^2 & ~\leq~
\DKL{\ncal(\mu_0,\Sigma_0}{\ncal(\mu_1,\Sigma_1)}\\
&~=~ \frac{1}{2} \Big(\DLD{\Sigma_0}{\Sigma_1} +
        (\mu_0-\mu_1) \transpose \Sigma_1^{-1} (\mu_0-\mu_1) \Big).
\end{align*}
\end{lemma}
\begin{proof}
The inequality is Pinsker's inequality (see, e.g., \cite[Lemma 2.5]{Tsybakov})
applied to the distributions
${\ncal(\mu_0,\Sigma_0)}$ and ${\ncal(\mu_1,\Sigma_1)}$,
and the equality is a known formula for the KL divergence between Gaussians
(see, e.g., \cite[Equation A.23]{Rasmussen}).
\end{proof}

\begin{lemma}\label{gaussianTV1d}
For any $\mu, \sigma, \widehat{\mu}, \widehat{\sigma}\in\R$
with $|\widehat{\mu}-\mu| \leq \eps \sigma$
and $|\widehat{\sigma}-\sigma| \leq \eps \sigma$
and $\eps \in [0,2/3]$ 	we have
	$$\|\mathcal{N}(\mu, \sigma^2) - \mathcal{N}(\widehat{\mu}, \widehat{\sigma}^2)\|_1 ~\leq~ 2 \eps.$$
\end{lemma}

\begin{proof}
By Lemma~\ref{lem:tvupbound},
\[
4\DTV{\ncal(\muhat,\sigmahat^2)}{\ncal(\mu,\sigma^2)}^2 
~\leq~
\frac {\sigmahat^2}{\sigma^2}-1
- \log \Big( \frac {\sigmahat^2}{\sigma^2} \Big)
+ \frac{|\mu-\muhat|^2}{\sigma^2}
~\leq~
\Big(\frac {\sigmahat}{\sigma}\Big)^2 -1
- \log \left( \Big(\frac {\sigmahat}{\sigma}\Big)^2 \right)
+ \eps^2.
\]
Since $z\coloneqq {\sigmahat}/{\sigma} \in [1-\eps,1+\eps]$ and $\eps \leq 2/3$, using the inequality $x^2-1-\log(x^2)\leq 3(x-1)^2$ valid for all $x \in [1/3,5/3]$, we find
\[
\DTV{\ncal(\muhat,\sigmahat^2)}{\ncal(\mu,\sigma^2)}^2 
~\leq~
\frac 1 4
(
3 (z-1)^2
+ \eps^2
)
~\leq~ \frac 1 4
(
4 \eps^2
)
~=~ \eps^2.
\]
The lemma follows since the $L^1$ distance is symmetric and is equal to twice the TV distance.
\end{proof}

\subsection{Bounds on net sizes}

\begin{definition}[$\eps$-net]
	Let $\eps\geq 0$.
	We say $N\subseteq X$ is an $\eps$-net for $X$ in metric $d$ if for each $x\in X$ there exists some $y\in N$ such that $d(x,y)\leq \eps$.
\end{definition}

The following lemma is Corollary~4.2.13 in \cite{hdp-vershynin}; we include a proof for completeness.
The notation $B_2^d$ denotes the $d$-dimensional Euclidean ball, $B_2^d \coloneqq \setst{ x \in \bR^d }{ \norm{x}_2 \leq 1 }$.

\begin{lemma} 
\LemmaName{epsnetL2}
For any $\eps\in(0,1]$, there exists an $\eps$-net for $B_2^d$
in $\ell_2$ metric of size $(3/\eps)^d$.
\end{lemma}
\begin{proof}
Take a maximal set of points in the unit ball $B_2^d$ that are $\eps$-separated, i.e., the distance between every pair of points is greater than $\eps$. Such a set must be an $\eps$-net by maximality.
Moreover, by the triangle inequality, the balls of radius $\eps/2$ centered at the points in the $\eps$-net are disjoint and hence the sum of their volumes is not more than the volume of a ball of radius $1+\eps/2$.
The volume of a $d$-dimensional ball of radius $r$ is $c_d r^d$ for some constant $c_d$, thus the size of this $\eps$-net is at most $\frac{c_d(1+\eps/2)^d}{c_d(\eps/2)^d}\leq(3/\eps)^d$, as required.
\end{proof}

\begin{lemma}
	\LemmaName{epsnetLinfsimplex}
	For any $\eps\in(0,1]$ there exists an $\eps$-net for $\Delta_d$ in $\ell_{\infty}$ metric of size $\eps^{-d}$.
\end{lemma}
\begin{proof}
	We give an algorithm to construct the net:
	partition $[0,1]^d$ into $\eps^{-d}$ cubes of side-length $\eps$; for each cube that intersects $\Delta_d$, put one arbitrary point of the intersection in the net.
\end{proof}

\section{Justification for our model}
\SectionName{justification}

Some of the existing models for learning mixtures of Gaussians need  structural assumptions on the target distribution.
For example, learning under the parameter estimation model requires that the means are sufficiently separated and that the mixing weights are not too small; see the discussion after \cite[Definition~1]{KMV}.

A key motivation for our work is to study a model for learning mixtures of Gaussians that requires no structural assumptions at all.
Specifically, we would like to identify a model in which Gaussian mixtures can be learned up to error
$\epsilon$ with sample complexity depending only on $k$, $d$ and $\epsilon$, and then derive optimal
sample complexity bounds in that model.
Density estimation under the TV distance is one such model:
\cite[Theorem~14]{ashtiani2017sample} and \Theorem{upper_bound} in this paper 
show that mixtures of Gaussians can be learned up to error $\eps$ with sample complexity depending on $k,d,$ and $\eps$ only.
In this section we provide further justification for using this particular model.
In \Section{kl_lp} we argue that the TV distance is not an arbitrary choice. 
If instead we had used the KL divergence or any $L^p$ distance, with $p>1$, then the sample complexity must necessarily depend on the structural properties of the distribution.
Thus, TV distance is a natural choice.

It is also natural to wonder whether some of our results could be derived from existing results on parameter estimation. 
In \Section{parameter} we show that this is not the case: entry-wise estimation of the covariance matrices is quite unrelated to density estimation under the TV distance.
Thus our model is natural and our results are not subsumed by previous work.\label{sec:interesting}

\subsection{Comparison to KL divergence and $L^p$ distances}
\label{sec:kl_lp}
In this section we consider the problem of density estimation, for mixtures of Gaussians,
using a distance measure that is either the KL divergence or an $L^p$ distance with $p>1$.
Under these distance measures, we show that the sample complexity of this problem
must necessarily depend on the structural properties of the distribution---that is, it cannot be bounded purely as a function of $k$, $d$ and $\epsilon$.

First we consider using the KL divergence. Recall that KL divergence is not symmetric. 
We consider using KL divergence only in one direction and show that no algorithm can guarantee,
after receiving a uniformly bounded number of samples from the true distribution,
that the KL divergence between the true distribution and the output distribution is smaller than any finite number.
In fact, this holds even for mixtures of two one-dimensional Gaussians with unit variances.

\begin{theorem}
    \label{thm:kl_bad}
    Let $\cF$ be the class of mixtures of two Gaussians in $\bR$,
    both of which have unit variance.
    Let $\acal$ be any algorithm whose input is a finite sequence of real numbers
    and whose output is a (Lebesgue) measurable density function.
    Then, for every $m \in \bN$ and every $\tau > 0$, there exists a density $f \in \cF$ such that if $X_1', \ldots, X_m' \sim f$
    then $\DKL{f}{\acal(X_1', \ldots, X_m')} \geq \tau$ with probability at least $0.98$.
\end{theorem}

We present the proof idea here, leaving the formal argument to Appendix~\ref{app:kl_lp_bad}.
Let $a \in \bN$ and consider the set of distributions $(1-\delta) \cdot \ncal(0,1) + \delta \cdot \ncal(a, 1)$, where $\delta \ll 1/m$.
Any algorithm that draws $m$ samples from such a distribution will likely have all of its samples
come from $\ncal(0,1)$.
However, the only way for the KL divergence to be small is if the distribution returned by
$\acal$ has non-negligible mass near the $\ncal(a,1)$ distribution, which is impossible since
the samples provide no information about $a$.

Next we consider $L^p$ distances and prove a result analogous to Theorem~\ref{thm:kl_bad}.
The main difference is that the argument uses Gaussians with different variances,
which can strongly influence the $L^p$ distance.

\begin{theorem}
    \label{thm:lp_bad}
    Let $\fcal$ be the class of mixtures of two Gaussians in $\bR$.
    Let $\acal$ be any algorithm whose input is a finite sequence of real numbers and whose output is a (Lebesgue) measurable density function.
    Then, for every $p > 1$, every $m \in \bN$, and every $\tau > 0$, there exists a density $f \in \cF$ such that if $X_1', \ldots, X_m' \sim f$ then $\norm{f-\acal(X_1', \ldots, X_m')}_p \geq \tau$ with probability at least $0.98$.
\end{theorem}

The proofs of Theorems~\ref{thm:kl_bad} and~\ref{thm:lp_bad} appear in Appendix~\ref{app:kl_lp_bad}.
Note that the theorems hold even for randomized algorithms.

\subsection{Comparison to parameter estimation}
\label{sec:parameter}

In this section, we observe that neither our upper bound
(\Theorem{upper_bound}) nor our lower bound (\Theorem{lower_bound})
can directly follow from results about parameter estimation for Gaussian mixtures.
First, recall that our sample complexity upper bound in
\Theorem{upper_bound} has no dependence on the structural properties 
of the Gaussians in the mixture.
Next, consider an algorithm that learns a single $d$-dimensional Gaussian
and provides a proximity guarantee on the entries of the covariance matrix.
If we use this entrywise guarantee to infer closeness in either KL divergence or TV distance, 
then we argue that the error must depend on the condition number of the covariance matrix.

The \emph{condition number} of a matrix $\Sigma$, i.e., 
the ratio of its maximum and minimum eigenvalues, is denoted by $\kappa(\Sigma)$.

\begin{proposition}
\label{prop:parametervsdensity1}
Set $\eps = \frac{2}{\kappa(\Sigma)+1}$.
There exist two covariance matrices $\Sigma$ and $\hat{\Sigma}$
that are good entrywise approximations (both additively and multiplicatively):
$$
\abs{\Sigma_{i,j} - \hat{\Sigma}_{i,j}} \leq \eps
\qquad\text{and}\qquad
\hat{\Sigma}_{i,j} \in [1,1+2\eps] \cdot \Sigma_{i,j}
\qquad\forall i,j,
$$
but the corresponding Gaussian distributions are as far as they can get, i.e., 
$$
\DKL{ \ncal(0,\Sigma) }{ \ncal(0,\hat{\Sigma}) } = \infty
\qquad\text{and}\qquad
\DTV{ \ncal(0,\Sigma) }{ \ncal(0,\hat{\Sigma}) } = 1.
$$
\end{proposition}
\begin{proof}
Define
$$
\Sigma = \begin{bmatrix} 1 & -(1-\eps) \\ -(1-\eps) & 1 \end{bmatrix}
\qquad\text{and}\qquad
\hat{\Sigma} = \begin{bmatrix} 1 & -1 \\ -1 & 1 \end{bmatrix},
$$
where $\eps \in (0,1/2)$. 
The eigenvalues of $\Sigma$ are $2-\eps$ and $\eps$,
so $\kappa(\Sigma) = \frac{2}{\eps}-1$,
satisfying the stated condition for $\eps$.
Observe that $\Sigma$ and $\hat{\Sigma}$ satisfy the entrywise approximation guarantees in the statement of the theorem.
However, $\Sigma$ is non-singular and $\hat{\Sigma}$ is singular.
Thus, $\DTV{ \ncal(0,\Sigma) }{ \ncal(0,\hat{\Sigma}) } = 1$
(consider the event that a random variable lies in the range of $\hat{\Sigma}$), and
$\DKL{ \ncal(0,\Sigma) }{ \ncal(0,\hat{\Sigma}) } = \infty$
(recall the definition of KL divergence).
\end{proof}

Thus, given a black-box algorithm that provides an entrywise $\eps$-approximation to the true covariance matrix $\Sigma$,
if $\kappa(\Sigma) \geq 2/\eps$, it might output $\hat{\Sigma}$,
which does not approximate $\Sigma$ in KL divergence
or total variation distance. 
Thus \Theorem{upper_bound} is not a direct consequence of any parameter estimation algorithm
with entrywise covariance guarantees.

\begin{remark}\label{rem:singular}
The matrix $\widehat{\Sigma}$ in the above construction is singular, so $\ncal(0,\hat{\Sigma})$ is a singular Gaussian.
The point of this proposition is that an algorithm that provides a good entrywise approximation for the covariance matrix does not necessarily provide a good statistical approximation, as it may output a singular Gaussian for example.
Note that our algorithm handles singular Gaussians as is.
\end{remark}

One might wonder instead if our lower bound is a direct consequence of existing lower bounds on
parameter estimation.
We show that this is also not the case:
the next proposition shows that there exist Gaussians that are close in TV distance
but whose covariance matrices do not satisfy any (multiplicative) entrywise guarantee.
Thus, even if a lower bound concludes that a class of algorithms cannot provide entrywise
covariance guarantees, it is still possible  that an algorithm in that class
can provide guarantees on the TV distance.

\begin{proposition}
\label{prop:parametervsdensity2}
For any $\eps\in(0,1/2)$,
there exist two covariance matrices $\Sigma$ and $\hat{\Sigma}$
such that 
$
\DTV{ \ncal(0,\hat\Sigma) }{ \ncal(0,{\Sigma}) } \leq \eps,
$
but there exist $i,j$ such that, for any $c \geq 1$,
$
\hat{\Sigma}_{i,j} \not\in [1/c,c] \cdot \Sigma_{i,j}.
$
\end{proposition}
\begin{proof}
Define
$$
\Sigma = \begin{bmatrix} 1 & 0 \\ 0 & 1 \end{bmatrix}
\qquad\text{and}\qquad
\hat{\Sigma} = \begin{bmatrix} 1 & \eps \\ \eps & 1 \end{bmatrix},
$$
where $\eps \in [0,1/2]$. 
By Lemma~\ref{lem:tvupbound},
\begin{align*}
	2 \DTV{\ncal(0,\Sigma_0)}{\ncal(0,\Sigma_1)}^2 & ~\leq~
	\frac{\DLD{\Sigma_0}{\Sigma_1}}{2}
	= \frac{-\log(1-\eps^2)}{2} \leq \eps^2,
\end{align*}
so $\DTV{\ncal(0,\hat\Sigma)}{\ncal(0,{\Sigma})} \leq \eps$.
However, $\hat{\Sigma}$ is not an entrywise multiplicative approximation of $\Sigma$.
\end{proof}

\section{Compression schemes}
\SectionName{compression}

The main technique introduced in this paper is
using compression for density estimation.
An overview of this technique was given in \Section{compression:intro}.
In this section we provide formal definitions of compression schemes
and  their usage.

\subsection{Definition of a compression scheme}

Let $\mathcal{F}$ be a class of distributions over a domain $Z$.
Intuitively, a compression scheme for $\cF$ involves two agents:
an \emph{encoder} and a \emph{decoder}.
\begin{itemize}

\item
The encoder knows a distribution $g \in \cF$
and receives $m$ samples from this distribution.
She uses her knowledge of $g$ to construct and send a small \emph{message}
to the decoder, which will suffice for him to construct a distribution that is close to $g$.
From the $m$ samples, the encoder selects a subset of size $\tau$,
which  are somehow representative of $g$.
This subset, together with $t$ additional bits,
constitutes the message sent to the decoder.

\item
The decoder receives the message (the $\tau$ data points and the $t$ bits) and constructs a distribution that is close to $g$.
\end{itemize}

Of course, there is some probability that the samples are not representative
of the distribution $g$, in which case the compression scheme will fail.
Thus, we only require that the decoding succeed with constant probability.

We emphasize that including samples in the message is critical and they cannot be omitted. Using samples in addition to bits is the main conceptual novelty of our approach and allows learning distribution classes with infinite metric entropy.

The formal definition of a decoder follows.

\begin{definition}[decoder]
A \emph{decoder} for $\mathcal{F}$ is a deterministic function 
$\mathcal{J}:\bigcup_{n=0}^{\infty} Z^n \times \bigcup_{n=0}^{\infty} \{0,1\}^n 
\rightarrow \mathcal{F}$, which takes a finite sequence of elements of $Z$ and a finite sequence of bits, and outputs a member of $\mathcal{F}$. 
\end{definition}

The formal definition of a compression scheme follows.

\begin{definition}[robust compression schemes]
\label{def_robustcompression}
Let $\tau,t,m:(0,1)\rightarrow \mathbb{Z}_{\geq0}$ be functions, and let $r \geq 0$.
We say $\mathcal{F}$ admits $(\tau,t,m)$ $r$-robust compression if there exists a decoder $\mathcal{J}$ for $\mathcal{F}$ such that for any distribution $g \in \mathcal{F}$ and any distribution $q$ on $Z$ 
with $\|g-q\|_1\leq r$, the following holds:
\begin{quote}
For any $\eps \in (0,1)$, if a sample $S$ is drawn from $q^{m(\eps)}$, then, with probability at least $2/3$, there exists a sequence $L$ of at most $\tau(\eps)$ elements of $S$, and a sequence $B$ of at most $t(\eps)$ bits, such that $\|\mathcal{J}(L,B)-g\|_1\leq \eps$.
\end{quote}
\end{definition}

Note that $S$ and $L$ are sequences rather than sets;
in particular, they can contain repetitions.
Lastly, note that $m(\eps)$ is a \emph{lower bound} on the number of samples needed, whereas $\tau(\eps), t(\eps)$ are \emph{upper bounds} on the size of compression and the number of bits.

To summarize, the definition asserts that with probability $2/3$,
there is a (short) sequence $L$ of elements from $S$
and a (short) sequence $B$ of additional bits,
from which $g$ can be approximately reconstructed.
We emphasize that $L$ and $B$ depend on $g$.
Intuitively, the samples in $L$ together with the bits $B$ nearly represent the distribution $g$.
This is a notion similar to that of a sufficient statistic, but it is not equivalent, since here the representation is only approximate.
We say that the distribution $g$ is \emph{encoded} by the message $(L,B)$.
This compression scheme is called ``robust'' since 
it requires $g$ to be approximately reconstructed  from a sample generated from $q$ rather than $g$ itself. 
A 0-robust compression scheme is called a (non-robust) compression scheme.

\begin{remark}\label{remark_probability}
In the preceding definition we required that $L$ and $B$ exist with probability only $2/3$.
Naturally, one can boost this probability to $1-\delta$ by generating a sample of size
$m(\eps)\lceil \log_3(1/\delta)\rceil$: we partition this sample into $\lceil \log_3(1/\delta)\rceil$ parts of size $m(\eps)$. Each part has a desirable subsequence $L$ with probability at least $2/3$, so the probability that none of the parts has a desirable subsequence $L$ is not more than $(1/3)^{\lceil \log_3(1/\delta)\rceil} \leq \delta$.
\end{remark}

\subsection{Connection between compression and learning}

We now show that if a class of distributions has a (robust) compression scheme,
then it can be learned in the (agnostic) density estimation model.

The main idea is as follows.
An encoder cannot be implemented in the density estimation model because
she requires knowledge of the target distribution $g$.
However, since her interaction with the decoder only amounts to sending a
short message,
we can explore all possible behaviors of the encoder by a brute-force search over
all possible messages that she could have sent.
When any such message is provided as input to the decoder, he will output some distribution $f$.
Moreover, for \emph{at least one} such message,
namely the one that would have been produced by the encoder,
the decoder will output an $f$ that is guaranteed to be close to $g$. 
Thus, if we collect all possible distributions produced by the decoder on all possible input
messages, then the only remaining task is to select the distribution from that collection that is
closest to $g$.
Fortunately, this task has a known solution: the following result states that a finite class of size $M$ can be $3$-learned in the agnostic setting using $O(\log (M/\delta)/\eps^2)$ samples. 

\begin{theorem}\label{thm:candidates}
	There exists a deterministic algorithm that, given
    candidate distributions $f_1,\dots,f_M$,
    a parameter $\eps>0$,
    and $\lceil \log (3M^2/\delta)/2\eps^2 \rceil$ i.i.d.\ samples from an unknown distribution $g$,
     outputs an index $j\in[M]$ such that 
	\[
	\|f_j-g\|_1 ~\leq~ 3 \min_{i\in[M]} \|f_i-g\|_1 + 4\eps,
	\]
    with probability at least $1-\delta/3$.
\end{theorem}

This result is essentially proven in~\cite[Theorem~1]{Yatracos}.
It immediately follows from \cite[Theorem~6.3]{devroye_book} and 
Hoeffding's inequality~\cite[Theorem~2.1]{devroye_book}.

Our approach for relating compression schemes and density estimation,
described informally above, is made formal by the following theorem.
It uses \Theorem{candidates} to select the best distribution that 
the decoder could output.
Note that we assume the learner knows all the problem parameters, such as $k,d,\eps,\delta,\tau,t,m,$ and $r$, but is oblivious to the target distribution. 

\begin{theorem}[compression implies learning]
\label{thm:compression}
Suppose $\mathcal{F}$ admits $(\tau,t,m)$ $r$-robust compression. 
Let $\tau'(\eps)\coloneqq \tau(\eps  )+t(\eps  )$.
Then $\mathcal{F}$ can be $\max\{3,2/r\}$-learned in the agnostic setting using 
\begin{align*}
O\left(
m\Big(\frac \eps 6\Big) \log\Big(\frac{1}{\delta}\Big)
 + \frac{\tau'(\eps/6) \log (m(  \eps /6) \log_3(1/\delta)) + \log(1/\delta)}{\eps^2} 
\right) =
 \widetilde{O}
\left(
m\Big(\frac \eps 6\Big)  + \frac{\tau'(\eps/6)\log m(  \eps/ 6)}{\eps^2} 
\right)
\end{align*}
samples. If $\mathcal{F}$ admits $(\tau,t,m)$ non-robust compression,
then $\fcal$ can be learned in the realizable setting using the same number of samples.
\end{theorem}
\begin{proof}
We give the proof for the agnostic case; the proof for the realizable case is similar.
	Let $q$ be the target distribution from which the samples are being generated. Let $\alpha\coloneqq\inf_{f\in \mathcal{F}} \|f-q\|_{1}$ be the approximation error of $q$ with respect to $\mathcal{F}$. 
	The goal of the learner is to find a distribution $\hat{h}$ such that $\|\hat{h}-q \|_{1} \leq \max\{3,2/r\} \cdot \alpha + \eps$. 
	
	First, consider the case $\alpha < r$.
	In this case,
	we develop  a learner that finds a distribution $\hat{h}$ such that $\|\hat{h}-q \|_{1} \leq 3 \alpha + \eps$. 
	Let $g\in \mathcal{F}$ be a distribution such that
	\begin{equation}
	\label{eq:CompressionG}
	    \|g-q \|_{1} ~\leq~ \min\left\{\alpha + \frac{\eps}{12},r\right\}.
	\end{equation}
	Such a $g$ exists by the definition of $\alpha$. 
	By assumption, $\mathcal{F}$ admits $(\tau,t,m)$ $r$-robust compression.
	Let $\mathcal{J}$ denote the corresponding decoder.
	Given $\eps$, the learner first asks for an i.i.d.\  sample $S \sim q^{m(\eps/6) \cdot \log_3(2/\delta)}$.
	Recall the definition of robust compression and Remark~\ref{remark_probability}, which allows us to amplify the success probability of the decoder.
	Then, with probability at least $1-\delta/2$, 
	there exist $L \in S^{\tau(\eps/6)}$ and $B\in\{0,1\}^{t(\eps/6)}$ satisfying the following guarantee:
	letting $h^* \coloneqq \mathcal{J}(L, B)$, we have
    \begin{equation}
    \label{eq:CompressionH}
        \|h^*-g\|_1 ~\leq~ \frac{\eps}{6}.
    \end{equation}
	
	The learner is of course unaware of $L$ and $B$. However, given the sample $S$, it can try all of the possibilities for $L$ and $B$ and create a candidate set of distributions.
More concretely, let
$$ H ~=~ \{\: \mathcal{J}(L', B') \::\: L'\in S^{\tau(\eps/6)},\, 
	B' \in \{0,1\}^{t(\eps/6)}\: \}.
$$
Note that
\begin{align*}
	|H| 
    ~\leq~ \big(m(\eps/6)\log_3(2/\delta)\big)^{\tau(\eps/6)}2^{t(\eps/6)}
    ~\leq~ \big(m(\eps/6)\log_3(2/\delta)\big)^{\tau'(\eps/6)}.
	\end{align*}
Since $H$ is finite, we will use the algorithm of Theorem \ref{thm:candidates} to find a good candidate $\hat{h}$ from $H$. 
	In particular, we set the accuracy parameter in Theorem \ref{thm:candidates} to be $\eps/16$ and the confidence parameter to be $\delta/2$. 
	In this case, Theorem \ref{thm:candidates} requires 
	\begin{align*}
	\frac{\log(6|H|^2 /\delta)}{2(\eps/16)^2}
    ~=~
	O\left(\frac{ \tau'(\eps/6)\log(m(\frac \eps 6)\log_3(\frac 1 \delta)) +\log(\frac 1 \delta)}
	{\eps^2}\right)
    ~=~
    \widetilde{O}\left(\frac{ \tau'(\eps/6)\log m(\eps/6)}{\eps^2}\right)
	\end{align*} 
additional samples, and its output $\hat{h}$ satisfies the following guarantee:
	\begin{alignat*}{2}
	\|\hat{h} - q\|_{1} 
    &~\leq~	3\|h^* - q\|_{1} + 4\frac{\eps}{16}
        &&\qquad\text{(by Theorem \ref{thm:candidates})}\\
    &~\leq~ 3(\|h^* - g\|_{1} + \|g - q\|_{1}) + \frac{\eps}{4} &&\qquad\text{(by the triangle inequality)}\\
	&~\leq~ 3\Bigg(\frac{\eps}{6} + \Big(\alpha + \frac{\eps}{12}\Big)\Bigg)  + \frac{\eps}{4}
	    &&\qquad\text{(by \eqref{eq:CompressionG}
	    and \eqref{eq:CompressionH})}\\
    &~=~ 3\alpha + \eps.
	\end{alignat*}
Note that the above procedure uses
$\widetilde{O}\big(m(\eps/6) + \tau'(\eps/6)\log m(\eps/6)/\eps^2 \big)$ samples
and its failure probability is at most $\delta$: the probability of either $H$ not containing a good $h^*$ or the failure of Theorem~\ref{thm:candidates}, in choosing a good candidate among $H$, is bounded by $\delta /2 + \delta /2 = \delta$.

	The other case, $\alpha \geq r$, is trivial: the learner outputs some distribution $\widehat{h}$.
	Since $\widehat{h}$ and $q$ are density functions, we have
	\(
	\|\widehat{h}-q\|_1 \leq 2 \leq \frac{2}{r} \cdot \alpha \leq \max \{3, 2/r\} \cdot \alpha + \eps
	\).
	\end{proof}

\subsection{Combining compression schemes}

In the rest of this section,
we prove a few lemmata showing that compression schemes can be
combined in useful ways.
These results concern product distributions
(which will be useful for axis-aligned Gaussians)
and mixture distributions (which will be useful for mixtures of Gaussians).

First, \Lemma{product_compress} below states that if a class $\fcal$ of distributions can be robustly compressed,
then the class of distributions that are formed by taking products of members of $\fcal$ can also be robustly compressed.
If $p_1,\dots,p_d$ are distributions over domains $Z_1,\dots,Z_d$, then $\prod_{i=1}^{d} p_i$ denotes the standard product distribution over $\prod_{i=1}^{d} Z_i$.
For a class $\fcal$ of distributions, define
$$
\fcal^d ~\coloneqq~
\left\{~ \prod_{i=1}^{d} p_i \::\: p_1,\dots,p_d \in \fcal ~\right\}.
$$

\begin{lemma}[compressing product distributions]
\label{lem:product_compress}
For any $\tau,t,m,r,d$,
$$
\setlength{\arraycolsep}{0pt}
\begin{array}{rrcccll}
\text{if $\mathcal{F}$ admits~~}
    &\big(~
    &\tau(\eps),
    &t(\eps),
    &m(\eps) 
    &~\big)
    &\text{~~$r$-robust compression,}
\\
\text{then $\mathcal{F}^d$ admits~~}
    &\big(~
    &d \cdot \tau(\eps/d),~
    &d \cdot t(\eps/d),~
    &\log_3(3d) \cdot m(\eps/d)
    &~\big)
    &\text{~~$r$-robust compression.}
\end{array}
$$
\end{lemma}

For the proof, we need the following standard proposition, which can be proved, e.g., using the coupling characterization of the total variation distance.
\begin{proposition}[Lemma~3.3.7 in~\cite{Reiss}] \label{prop:TV} For $i\in [d]$, let $p_i$ and $q_i$ be probability distributions over the same domain $Z$. Then
	$\|\Pi_{i=1}^{d} p_i - \Pi_{i=1}^{d} q_i \|_{1} \leq \sum_{i=1}^{d} \|p_i - q_i\|_{1}$.
\end{proposition}

\begin{proof}[Proof of Lemma~\ref{lem:product_compress}]
	Let $Z$ be the domain of $\mathcal{F}$ and $G=\Pi_{i=1}^{d}g_i$ be an arbitrary element of $\mathcal{F}^d$, with all $g_i\in\fcal$. 
	Let $Q$ be an arbitrary distribution over $Z^d$ subject to $\|G-Q\|_1 \leq r$. 
	Let $q_1,\dots,q_d$ be the marginal distributions of $Q$ on the $d$ components. 
	Observe that $\|q_j-g_j\|_1\leq r$ for each $j\in[d]$, since Proposition~\ref{fact:dtv} implies that projection onto a coordinate cannot increase the total variation distance.
    
	The lemma's hypothesis is that $\mathcal{F}$ admits $(\tau,t,m)$ $r$-robust compression.
	Let $\mathcal{J}$ denote the corresponding decoder, let $m_0\coloneqq m(\eps/d)\log_3(3d)$ and $S\sim Q^{m_0}$.
	To prove the lemma we must encode an $\eps$-approximation of $G$ using $d \cdot \tau(\eps/d)$ elements of $S$ and $d \cdot t(\eps/d)$ bits.
	
	Since $S$ contains $m_0$ samples, each of which is a $d$-dimensional vector, we may think of $S$ as a $d \times m_0$ matrix over $Z$.
	Let $S_i$ denote the $i$th row of this matrix.
	That is, for $i\in[d]$, let $S_i \in Z^{m_0}$
	be the vector of the $i$th components of all elements of $S$.
	By definition of $q_i$, we have $S_i \sim q_i^{m_0}$ for each $i$.
	As observed above, we have $\|q_i-g_i\|_1\leq r$.
	
	Apply Remark~\ref{remark_probability} with parameters $\epsilon/d$ and $\delta=1/3d$
	for each $i\in[d]$.
	Then, for each $i$, the following statement holds with probability at least $1-1/3d$:
	there exists a sequence $L_i$ of at most $\tau(\eps/d)$ elements of $S_i$
	and a sequence $B_i$ of at most $t(\eps/d)$ bits
	such that $\|\mathcal{J}(L_i,B_i) - g_i\|_1 \leq \eps/d$.
	By the union bound, this statement holds simultaneously for all $i\in[d]$ with probability at least $2/3$.
	We encode these  $L_1,\dots,L_d,B_1,\dots,B_d$ 
	using $d \cdot \tau(\eps/d)$ samples from $S$
	and $d \cdot t(\eps/d)$ bits.
	Our decoder for $\fcal^d$ then extracts 
	$L_1,\dots,L_d,B_1,\dots,B_d$ 
	from these samples and bits, and then outputs 
	$\prod_{i=1}^{d}\mathcal{J}(L_i,B_i) \in \fcal^d$.
	Finally, Proposition~\ref{prop:TV} gives
	$$\left\|\Pi_{i=1}^d\mathcal{J}(L_i, B_i)-G\right\|_1 \leq \sum_{i=1}^d \left\|\mathcal{J}(L_i, B_i)-g_i\right\|_1 \leq d \cdot {\eps}/{d} = \eps,$$
	completing the proof.
\end{proof}

Our next lemma states that if a class $\fcal$ of distributions can be compressed,
then the class of distributions that are formed by taking mixtures of members of $\fcal$ can also be compressed.

\begin{lemma}[compressing mixtures, non-robustly]
\label{lem:compressmixtures}
For any $\tau,t,m,r,d$,
if class $\mathcal{F}$ admits
$\big(\tau(\eps),t(\eps),m(\eps)
\big)$ non-robust compression,
then the class of its $k$-mixtures,
{i.e., $\kmix(\fcal)$, admits}
$\displaystyle\big(
    k \cdot \tau\left(\frac \eps 3\right),~
    k \cdot t\left(\frac \eps 3\right) + k \log_2 \left(\frac{3k}{\eps}\right),~
    \frac{48k \log(6k)}{\eps} \cdot m\left(\frac\eps 3\right) 
    \big)$
{non-robust compression.}
\end{lemma}

\begin{proof}
Consider any $g \in \kmix(\fcal)$, so $g = \sum_{i\in[k]} w_i f_i$ for some distributions $f_1,\dots,f_k \in \fcal$
and mixing weights $w_1,\ldots, w_k$.
Define $m_0 \coloneqq 48 m(\eps/3) k \log(6k) / \eps$, and draw $S \sim g^{m_0}$.
Then $S$ has the same distribution as the process that performs $m_0$ independent trials as follows: 
select a component $i$ proportional to the weights $w_1,\ldots,w_k$, then draw a sample from $f_i$.
In the latter process, we define $S_i$ to be the sequence of samples that were generated using $f_i$.
Our encoder for $g$ will discretize the mixing weights and use the compression scheme for $\mathcal{F}$ to separately encode each $f_i$.

\paragraph{Encoding the mixing weights.}
We encode $w_1,\dots,w_k$ using bits as follows.
	Consider an $(\eps/3k)$-net in $\ell_{\infty}$ for $\Delta_k$
	of size $(3k/\eps)^k$ (see \Lemma{epsnetLinfsimplex}).
	Let $(\widehat{w}_1,\dots,\widehat{w}_k)$ be an element in the net that has 
	\begin{equation}
	\label{eq:CompressMixW}
	\|(\widehat{w}_1,\dots,\widehat{w}_k) - (w_1,\dots,w_k) \|_{\infty} ~\leq~ \eps / 3k.
	\end{equation}
	Encoding the element 
	$(\widehat{w}_1,\dots,\widehat{w}_k)$
	from the net requires only $k\log_2 ( 3k/\eps )$ bits.

\paragraph{Encoding $f_1,\dots,f_k$.}
	For any $i\in[k]$, we say that index $i$ is \emph{negligible} if $w_i \leq \eps/(6k)$.
	For any negligible index we will approximate $f_i$ by an arbitrary distribution $\widehat{f}_i$.
	For any non-negligible index we will likely have enough samples from $f_i$ to use the compression scheme for $\fcal$ to encode a distribution $\widehat{f}_i$ that approximates $f_i$.
	
	Define $m_1 = m(\eps/3) \log(6k)$.
	For each non-negligible index $i$, by the Chernoff bound (Theorem~\ref{thm:chernoff}), with probability at least $1-1/6k$, we have $|S_i| \geq m_1$.
	By a union bound, this statement holds simultaneously for all non-negligible $i \in [k]$ 
	with probability at least $5/6$.
	
	Apply Remark~\ref{remark_probability} with parameters $\epsilon/3$ and $\delta=1/6k$ for each non-negligible index $i$.
	Then, for each such $i$, the following statement holds with probability at least $1-1/6k$:
	there exist $\tau(\eps/3)$ samples from $S_i$ and $t(\eps/3)$ bits from which the decoder for $\fcal$ constructs a distribution $\widehat{f}_i$
	with 
	\begin{equation}
	\label{eq:CompressMixF}
	\|f_i-\widehat{f}_i\|_1 ~\leq~ \eps/3.
	\end{equation} 
	By the union bound, this statement holds simultaneously for all non-negligible indices with probability at least $5/6$.
	The encoding consists of these samples and bits for each non-negligible $i$, whereas for negligible $i$ we use the same number of samples and bits, chosen arbitrarily.

    By a union bound, the failure probability of the encoding is at most $2 \cdot (1-5/6) = 1/3$.

\paragraph{Complexity of the encoding.}
The discretized weights require
$k\log_2 ( 3k/\eps )$ bits.
For each index $i \in [k]$, we use $\tau(\eps/3)$ samples and $t(\eps/3)$ bits.
Thus, the total number of bits is $k \cdot t(\eps/3) + k\log_2 ( 3k/\eps )$, 
and the total number of samples is $k \cdot \tau(\eps/3)$.

\paragraph{Decoding.}
The decoder for $\kmix(\fcal)$ is given the discretized weights $\widehat{w}_1,\dots,\widehat{w}_k$.
It is also given, for each index $i$, $\tau(\eps/3)$ samples and $t(\eps/3)$ bits, which it provides to the decoder for $\fcal$, yielding the distribution $\widehat{f}_i$.
(Recall that, for a negligible index $i$, the distribution $\widehat{f}_i$ is arbitrary.)
The decoder outputs the distribution $\sum_i \widehat{w}_i \widehat{f}_i$.

	To complete the proof of the lemma, we will show that $\| \sum_i w_i f_i - \sum_i \widehat{w}_i \widehat{f}_i \|_1 \leq \eps$ with probability at least $2/3$.
	Let $N\subseteq [k] $ denote the set of negligible components.
	Recall that the encoder succeeds with probability at least $2/3$, in which case the decoded distributions $\widehat{f}_i$ will satisfy \eqref{eq:CompressMixF} for each $i \not\in N$.
So, we have
	\begin{alignat*}{2}
	\left\| \sum_{i\in [k]} (\widehat{w}_i \widehat{f}_i - w_i f_i)\right\|_1 
	&~\leq~
	\left\|\sum_{i\in [k]} {w}_i (\widehat{f}_i -  f_i) \right\|_1 
	+ \left\|\sum_{i\in [k]} (\widehat{w}_i-w_i) \widehat{f}_i \right\|_1
	\\&~\leq~
	\left\|\sum_{i\in N}{w_i} (\widehat{f}_i -  f_i) \right\|_1 
	+ \left\|\sum_{i\notin N}{w_i} (\widehat{f}_i - f_i) \right\|_1
	+ \sum_{i\in [k]} |\widehat{w}_i-w_i| \cdot \left\|\widehat{f}_i\right\|_1
	\\&~\leq~
	\sum_{i\in N}{w_i} \cdot 2
	~+~ \sum_{i\notin N}{w_i} \cdot \frac{\eps}{3}
	~+~ \sum_{i\in [k]} \frac{\eps}{3k} \cdot 1
	\qquad\qquad\text{~(by \eqref{eq:CompressMixW} and \eqref{eq:CompressMixF})}
	\\&~\leq~ k \cdot \frac{\eps}{6k} \cdot 2 ~+~ \frac{\eps}{3} ~+~ \frac{\eps}{3}
	~=~ \eps
	\qquad\qquad\qquad\quad\text{(by definition of $N$).}
	\end{alignat*}
This completes the analysis of the compression scheme for $\kmix(\fcal)$.
\end{proof}

The preceding lemma shows that
non-robust compression of $\fcal$
implies non-robust compression of $\kmix(\fcal)$.
We do not know whether an analogous statement holds for robust compression.
That is,
does robust compression of $\fcal$ imply robust compression of $\kmix(\fcal)$,
for a general class $\fcal$?
Nevertheless, in the next lemma we show that if $\fcal$ can be robustly compressed,
then $\kmix(\fcal)$ can be \emph{learned in the agnostic setting}.

\begin{lemma}[learning mixtures, robustly]
\label{lem:agnosticlearningmixture}
Suppose $\fcal$ admits $(\tau(\eps), t(\eps), m(\eps))$  $r$-robust compression,
and let $\tau'(\eps)\coloneqq \tau(\eps)+t(\eps)$.
Then $\kmix(\fcal)$ admits 
$3(1+ 2/r)$-agnostic learning with sample complexity
\[
\widetilde{O}
\left(
\frac{k m (\eps/10)}{\eps}
+ \frac{k\tau'(\eps/10) \log m(\eps/10)}{\eps^2}
\right).
\]
\end{lemma}

We first give a sketch of the proof.
Let $g$ be the target distribution and suppose there exists $\rho \geq 0$ and $f \in \kmix(\cF)$ such that $\norm{g - f}_1 \leq \rho$.
Since $f \in \kmix(\cF)$, we can write $f = \sum_{i \in [k]} w_i f_i$, where $f_i \in \cF$, $w_i \geq 0$, and $\sum_{i \in [k]} w_i = 1$.
A first attempt would be to try to write $g = \sum_{i \in [k]} w_i g_i$ such that each $\norm{g_i - f_i}_1 \leq r$;
if this were true, then given a sufficient number of samples from $g$, we would have sufficient samples from each $g_i$, and then we could use an $r$-robust compression scheme for $\fcal$ to encode, for each $i$,
some $\hat{f}_i$ close to $f_i$.
Alas, it is not clear whether we can ensure that $\norm{g_i - f_i}_1 \leq r$ for all $i$.
However,  Lemma~\ref{mixturelemma} below asserts that we can write $g = \sum_{i \in [k]} w_i g_i$ 
in such a way that, for each $i$, either $\norm{g_i - f_i}_1 \leq r$ or
$w_i$ is small (in fact, the \emph{sum} of all such weights is small) and, hence, their contribution to the TV distance is small.
Thus, we need only deal with the case where $\norm{g_i - f_i}_1 \leq r$, a task for which $r$-robust compression is well-suited.

\begin{lemma}\label{mixturelemma}
	Let $f = \sum_{i \in [k]} w_i f_i$ be a density with $(w_1, \ldots, w_k) \in \Delta_k$ and each $f_i\in\fcal$.
    Let $g$ be a density such that $\norm{g - f}_1 \leq \rho$.
    Then, we can write $g = \sum_{i \in [k]} w_i g_i$ such that each $g_i$ is a density and for any $r > 0$,
    \[
    \sum_{i \::\: \|g_i - f_i\|_1 > r} w_i \:<\: \rho/r.
    \]
\end{lemma}

The proof of this lemma is cumbersome and appears in Appendix~\ref{sec:lemmamixtureproof}.
We now prove Lemma~\ref{lem:agnosticlearningmixture}.	
\begin{proof}[Proof of Lemma~\ref{lem:agnosticlearningmixture}]
Let $g$ be the target distribution, and let
$f \in \kmix(\cF)$ be such that $\norm{f - g}_1 \leq \rho$.
To prove the lemma, we need to describe a learning algorithm
that outputs a distribution whose $L^1$ distance to $g$ is bounded by
$3\rho(1 + 2/r)+\eps$.

Let $g=\sum_{i \in [k]} w_i g_i$ be the representation given by Lemma~\ref{mixturelemma}.
The learner first takes $M=160 m(\eps/10) \log_3 (3k/\delta) k/\eps$ samples from $g$. Let $S$ be the set of these samples. We view $g$ as a mixture of the $g_i$, so $S$ can be partitioned into $k$ subsets such that the $i$th subset has distribution $g_i$. We learn each of the components individually. The learner does not know which sample point comes from which component, but it can try all possible ways of partitioning $S$ into $k$ subsets, hence generating several candidate distributions, such that at least one of them is close to $g$.
Moreover, the learner ``guesses'' the weights $w_i$: let $W$ be  an $(\eps/10k)$-net in $\ell_{\infty}$ for $\Delta_k$ of size $(10k/\eps)^k$
(see \Lemma{epsnetLinfsimplex}).
So, there exists some point $(\widehat{w}_1,\dots,\widehat{w}_k)\in W$ such that 
\begin{equation}\max_i |w_i -\widehat{w}_i| \leq \eps/10k.
	\label{eq:weights}
\end{equation}
	
	For each $i\in[k]$, component $i$ is called \emph{tiny} if $w_i < \eps/20k$, \emph{far} if $\|g_i-f_i\|_1 > r$, and \emph{nice} otherwise.
	The sum of weights of tiny components is at most $\eps/20$, and the sum of weights of far components is at most $\rho/r$ by Lemma~\ref{mixturelemma}.
	
	The number of samples from component $i$ is binomial with mean $M w_i$.
	By the Chernoff bound (Theorem~\ref{thm:chernoff}) and a union bound over nice components, with probability at least $1-\delta/3$, there are at least $m(\eps/10) \log_3(3k/\delta)$ points from each nice component.
	If this is the case, 
	then the definition of robust compression implies that, for each nice component $g_i$, with probability at least $1-\delta/3k$, there exists a sequence $L_i \in S^{\tau(\eps/10)}$ and a sequence $B_i \in \{0,1\}^{t(\eps/10)}$  such that $\|\jcal(L_i, B_i) - f_i\|_1 \leq \eps/10$, where $\jcal$ is the decoder for $\fcal$. By a union bound over nice components, this is simultaneously true for all nice components, with probability at least $1-\delta/3$.
	
	Thus far, we have proved that, with probability at least $1-2\delta/3$, there exist sequences 
	$L_1,\dots,L_k \in S^{\tau(\eps/10)}$ and
	$B_1,\dots,B_k\in\{0,1\}^{t(\eps/10)}$ such that
	\begin{equation}
		\|\jcal(L_i, B_i) - f_i\|_1 \leq \eps/10 \text{ for each nice component } i. \label{eqjlb}
	\end{equation}
	The learner builds the following set of candidate distributions:
	\[\mathcal C ~\coloneqq~
	\left\{~
	\sum_{i=1}^{k} {w}'_i  \jcal(L_i', B_i') ~:~
	 L_1',\dots, L_k' \in S^{\tau(\eps/10)},\,
	 B'_1,\dots, B'_k\in\{0,1\}^{t(\eps/10)},\,
	({w}'_1,\dots,{w}'_k)\in W
	~\right \}.
	\]
	We claim that, with probability at least $1-2\delta/3$, at least one of the distributions in $\mathcal C$ is $(3\eps/10 + 2\rho/r + \rho)$-close to $g$. 
	This corresponds to the ``correct'' sequences $L, B$, and $\widehat{w}$; that is, when $L'_i=L_i,B'_i=B_i,$ and $w'_i=\hat w_i$ for all $i\in[k]$.
	To prove the claim, let $T, F,$ and $N$ denote the set of tiny, far, and nice components, respectively. Then, we have
	\begin{align*}
	    \left\| \sum_{i\in[k]} \widehat{w}_i \jcal(L_i, B_i) - w_i g_i  \right\|_1
	&~\leq~
	    \left\| \sum_{i\in[k]} w_i (\jcal(L_i, B_i) -  f_i)  \right\|_1 + \left\| \sum_{i\in[k]} (\widehat{w}_i-w_i) \jcal(L_i, B_i) \right\|_1
	    + \left\|f-g\right\|_1 \\
	&~\leq~
	    \sum_{i\in T \cup F} w_i \left\| \jcal(L_i, B_i) -  f_i  \right\|_1 + \sum_{i\in N} w_i \left\|\jcal(L_i, B_i) -  f_i  \right\|_1 \\
	& \qquad ~~ \qquad \qquad +\sum_{i\in[k]} |\widehat{w}_i-w_i| \cdot \left\| \jcal(L_i, B_i)  \right\|_1 + \rho \\
	&~\leq~
	    \sum_{i\in T \cup F} w_i \cdot 2 + \sum_{i\in N} w_i \cdot (\eps/10) + \sum_{i\in[k]} (\eps/10k) + \rho \\
	&~\leq~
	    (\eps/10 + 2\rho/r) + \eps/10 + \eps/10 + \rho \\&~=~ 3\eps/10 + 2\rho/r + \rho, 
	\end{align*}
	where the first two inequalities follows from the triangle inequality and since $\|f-g\|_1\leq\rho$, 
	the third inequality follows from \eqref{eqjlb} and \eqref{eq:weights}, and
	the fourth inequality follows from the definition of tiny and  Lemma~\ref{mixturelemma}.
	This proves the claim.
	
	Next, the learner applies the algorithm of Theorem~\ref{thm:candidates}, with error parameter $\eps/40$, to obtain a member of $\mathcal C$ whose distance from $g$ is bounded by
	$3 \cdot (3\eps/10 + 2\rho/r + \rho) + 4 (\eps/40) \leq \eps + 3\rho(1 + 2/r)$, as required.
	The overall failure probability is bounded by $2\delta/3$ (failure probability of the claim) plus $\delta/3$ (failure probability of the algorithm of Theorem~\ref{thm:candidates}).
	
	The sample complexity of the algorithm is bounded as follows.
	The number of candidate distributions can be bounded by
	\[
	|\mathcal C| ~\leq~ \left( M^{\tau(\eps/10)} 2^{t(\eps/10)} \right)^k \cdot (10k/\eps)^k \leq M^{k\tau'(\eps/10)} \cdot (10k/\eps)^k,
	\]
	whence the total sample complexity can be bounded by
	\begin{align*}
	M & + \frac{\log(3|\mathcal C|^2/\delta)}{2\eps^2}
	\\&~=~
	O\left(
	m\Big(\frac{\eps}{10}\Big) \log\Big(\frac{k}{\delta}\Big)\frac{k}{\eps}
	+\frac{\log(1/\delta) + k \log(k/\eps) + k \tau'(\eps/10) \log \big(m({\eps}/{10} ) \log(k/\delta) k/\eps\big) }{\eps^2}
	\right)
	\\&~=~
	\widetilde{O}
	\left(
	\frac{k m (\eps/10)}{\eps}
	+ \frac{k\tau'(\eps/10) \log m(\eps/10)}{\eps^2}
	\right).\qedhere
	\end{align*}
\end{proof}

\section{Upper bounds: learning Gaussians mixtures via compression schemes}
\SectionName{upper_bound_simpler}
The main positive results of this paper are sample complexity bounds
for learning mixtures of Gaussians (Theorems~\ref{thm:upper_bound} and \ref{thm:upper_bound_axis_aligned}).
In this section, we prove these results by describing a compression scheme for a single Gaussian and then applying the tools developed in the previous section.
To begin, we illustrate the technique by analyzing the simpler problem of 
learning mixtures of axis-aligned Gaussians in the realizable setting.
To avoid the technical issue of $\log(1/\eps)$ getting extremely small, henceforth we will assume that $\eps$ is bounded away from 1, say $\eps\leq 0.99$.

\subsection{Warm-up: mixtures of axis-aligned Gaussians, non-robustly}

In this short section, we give an illustrative use of our compression framework to prove an upper bound of $\widetilde{O}(kd/\eps^2)$ for the sample complexity of learning mixtures of $k$ axis-aligned Gaussians in the realizable setting.
The next section gives a general argument for arbitrary Gaussians in the agnostic setting.

\begin{lemma}
\label{lem:single_simple}
The class of one-dimensional Gaussians admits a $(3, O(\log(1/\eps)), 3)$ non-robust compression scheme.
\end{lemma}

\begin{proof}
Let $0<c<1<C$ be such that
\(\Pr_{X\sim \ncal(0,1)}[c < |X| < C] \geq 0.99\).
Let $\ncal(\mu, \sigma^2)$ be the target distribution.
We first show how to encode $\sigma$.
Let $g_1, g_2 \sim \ncal(\mu, \sigma^2)$.
Then $g \coloneqq \frac{1}{\sqrt{2}}(g_1 - g_2) \sim \ncal(0, \sigma^2)$.
So, with probability at least $0.99$, we have $\sigma c < |g| < \sigma C$.
Conditioned on this event, we have $\lambda\coloneqq \sigma/g \in [-1/c, 1/c]$.
We now choose $\hat{\lambda} \in \{0, \pm \eps/2C^2, \pm 2\eps/2C^2, \pm 3\eps/2C^2 \ldots, \pm 1/c\}$ satisfying $|\hat{\lambda} - \lambda| \leq \eps/4C^2$,
and we encode the standard deviation by $(g_1, g_2, \hat{\lambda})$.
The decoder then estimates 
$\hat{\sigma} \coloneqq \hat{\lambda} (g_1-g_2)/\sqrt2$.
Note that $|\hat{\sigma} - \sigma| \leq |\hat{\lambda} - \lambda| |g| \leq \sigma \eps/4C$ and that the encoding requires two sample points and 
 $O(\log(C^2/c\eps))=O(\log(1/\eps))$ bits for encoding $\hat{\lambda}$.

Now we turn to encoding $\mu$.
Let $g_3 \sim \ncal(\mu, \sigma^2)$.
Then $|g_3 - \mu| \leq C\sigma$ with probability at least $0.99$.
We will condition on this event, which implies the existence of some $\eta \in [-C, C]$ satisfying $g_3 + \sigma \eta = \mu$.
We choose $\hat{\eta} \in \{0, \pm \eps/2, \pm 2 \eps/2, \pm 3 \eps/2 \ldots, \pm C\}$ such that $|\hat{\eta} - \eta| \leq \eps/4$,
and encode the mean by $(g_3,\hat{\eta})$.
The decoder estimates $\hat{\mu} \coloneqq g_3+\hat{\sigma} \hat{\eta}$.
Again, note that $|\hat{\mu} - \mu| = |\sigma \eta - \hat{\sigma}\hat{\eta}| \leq |\sigma\eta - \sigma\hat{\eta}| + |\sigma\hat{\eta} - \hat{\sigma}\hat{\eta}| \leq \sigma \eps/2$.
Moreover, encoding the mean requires one sample point and $O(\log(1/\eps))$ bits.

To summarize, the decoder has $|\hat{\mu} - \mu| \leq \sigma \eps/2$ and $|\hat{\sigma} - \sigma| \leq \sigma \eps/2$.
Plugging these bounds into Lemma~\ref{gaussianTV1d} gives
$\|\ncal(\mu, \sigma^2)-\ncal(\hat{\mu}, \hat{\sigma}^2)\|_1 \leq \eps$, as required.
\end{proof}

\begin{remark}\label{rem:compress}
In the above argument, the samples are not ``compressed'' in the usual sense of this verb. 
Nevertheless, our formal definition of compression,
Definition~\ref{def_robustcompression}, allows $m = \tau$.
\end{remark}

To complete the proof of \Theorem{upper_bound_axis_aligned} in the realizable setting,
we note that \Lemma{single_simple} combined with \Lemma{product_compress} implies that
the class of axis-aligned Gaussians in $\R^d$ admits an
$$\big(\: O(d) ,\, O(d \log(d/\eps)) ,\, O(\log(3d)) \:\big)$$
non-robust compression scheme.
(By noting that any axis-aligned Gaussian is a product of one-dimensional Gaussians.)
Then, by \Lemma{compressmixtures},
the class of mixtures of $k$ axis-aligned Gaussians admits an
$$\big(~ O(kd) ,\,
    O(kd\log(d/\eps) + k\log(k/\eps)) ,\,
    O(k\log(k)\log(d)/\eps)
    ~\big)
$$ non-robust compression scheme.
\Theorem{compression} now implies
that the class of $k$-mixtures of axis-aligned Gaussians in $\bR^d$ can be learned using $\widetilde{O}(kd/\eps^2)$ samples in the realizable setting.

\subsection{Learning axis-aligned and general Gaussian mixtures in the agnostic setting}
We now turn to the general case and prove an upper bound of $\widetilde{O}(kd^2/\eps^2)$ for the sample complexity of learning mixtures of $k$ Gaussians in $d$ dimensions,
and
an upper bound of $\widetilde{O}(kd/\eps^2)$ for the sample complexity of learning mixtures of $k$ axis-aligned Gaussians, both in the agnostic sense.
The heart of the proof is to show that Gaussians have robust compression schemes in any dimension.

\begin{lemma}
\label{lem:coreformixtures}
For any positive integer $d$, the class $\gcal_{d,1}$ of $d$-dimensional Gaussians admits an
$$
\big(~ O(d) ,\, O(d^2 \log(d/\eps)) ,\, O(d) ~\big)
$$
$2/3$-robust compression scheme.
\end{lemma}
\begin{remark}
\label{constantcompression}
In the special case $d=1$, there also exists a $(4,1,O(1/\eps))$ $0.773$-robust compression scheme
using completely different ideas.
The proof appears in Appendix~\ref{sec:constantcompression}.
Surprisingly, this compression scheme has constant size,
as the value of $\tau + t$ is independent of $\eps$
(unlike \Lemma{coreformixtures}).
This scheme could be used instead of \Lemma{coreformixtures} for proving \Theorem{upper_bound_axis_aligned}, although it would not improve the sample complexity bound asymptotically.
\end{remark}

\begin{remark}\label{rem:r1}
The proof of \Lemma{coreformixtures}
can be amended to give an $r$-robust compression scheme for any $r<1$, which will change the
constant 12 in 
the agnostic results of \Theorem{upper_bound} and \Theorem{upper_bound_axis_aligned} to any constant larger than $9$,
at the expense of worse constants for $\tau$, $t$ and $m$. 
This is straightforward but entails additional cumbersome notation, hence we omit the details.
\end{remark}

Before proving \Lemma{coreformixtures}, we show how to combine it with previous lemmata to prove our main upper bounds.

\begin{proof}[Proof of \Theorem{upper_bound}]
Combining \Lemma{coreformixtures} and \Lemma{agnosticlearningmixture} gives
 that the class of $k$-mixtures of $d$-dimensional Gaussians is $12$-agnostically learnable
with sample complexity $\widetilde{O}(kd^2/\eps^2)$.
\end{proof}

\label{oldpage18}
\begin{proof}[Proof of \Theorem{upper_bound_axis_aligned}]
	Recall that $\acal_{d,k}$ denotes the class of $k$-mixtures of $d$-dimensional axis-aligned Gaussian distributions. 
	Observe that any distribution in $\acal_{d,1}$ is a product of $d$ one-dimensional Gaussians, hence
    applying \Lemma{coreformixtures} for the case $d=1$ and
    then \Lemma{product_compress} 
    shows that  $\acal_{d,1}$ admits 
    $\big(O(d), O(d\log(d/\eps)), O(\log(3d))\big)$
    $2/3$-robust compression.
\Lemma{agnosticlearningmixture} then implies 
that $\acal_{d,k}$ is $12$-agnostically learnable
with sample complexity $\widetilde{O}(kd/\eps^2)$, completing the proof.
\end{proof}

\subsection{Proof of \Lemma{coreformixtures}}
\label{sec:lemma53}
We first provide an overview of the proof.
For simplicity, first assume that we 
are in the non-robust, zero-mean, full-rank scenario, i.e., we want to encode the distribution $\ncal(0, \Sigma)$,
where $\Sigma \in \R^{d \times d}$ has rank $d$.
Let $v_1, \ldots, v_d$ be an orthogonal set of vectors that satisfy $\Sigma = \sum_{i=1}^d v_i v_i\transpose$. 
(Such a representation can be obtained from the eigendecomposition of $\Sigma$ and noting that all its eigenvalues are positive. Note that the vectors $v_i$ are not normalized.)
Let $g_1, \ldots, g_d$ be i.i.d.\  samples from $\ncal(0, \Sigma)$.
As $\Span\{g_1, \ldots, g_d\} = \R^d$ with probability 1, a natural idea is to find,
for each $i$, coefficients $\lambda_{i,1}, \ldots, \lambda_{i,d}$ such that
$v_i = \sum_{j=1}^d \lambda_{i,j} g_j$.
The encoder sends $g_1, \ldots, g_d$ and a discretization of the values $\{\lambda_{i,j}\}_{i,j \in [d]}$ as her message and the decoder will recover $\Sigma$ approximately.
If the discretization of each $\lambda_{i,j}$ is accomplished  with $b$ bits, then we would obtain a $(d,d^2b,d)$ compression scheme.

The main difficulty is to control the bit complexity of a suitable discretization of $\lambda_{i,j}$---to achieve the optimal sample complexity bound, the bit complexity must be $\polylog(d/\eps)$ per coefficient.
The key to achieve a suitable discretization is the following fact from geometric functional analysis (\Lemma{litvak_newnew},  cf.~\cite[Corollary~4.1]{large_enclosed_ball}):
denote the convex hull of a set $T$ is by $\conv(T)$.
Given a sequence $g_1, \ldots, g_m$ of $m = {O}(d)$ i.i.d.\ samples from $\ncal(0, \Sigma)$, with high probability we have $\frac{1}{20} \cdot \cE \subseteq \conv\{\pm g_1, \ldots, \pm g_m\}$,
where $\cE$ is the ellipsoid centered at zero with principle axes $v_1, \ldots, v_d$.
This enables us to express each $v_i$ as $\sum_{j=1}^m \lambda_{i,j} g_j$, where each $\lambda_{i,j} \in [-20,20]$;
we then discretize by building an $\poly(\eps/d)$-net of size $\poly(d/\eps)$ on this interval, achieving the desired $\polylog(d/\eps)$ per coefficient
bit complexity.

Next, suppose the mean is not zero, say we want to encode the distribution $\ncal(\mu, \Sigma)$.
Note that if $g_1, g_2 \sim \ncal(\mu, \Sigma)$ then $\frac{g_1-g_2}{\sqrt{2}} \sim \ncal(0, \Sigma)$,
and thus we can use the same compression scheme as above to encode $v_1,\dots,v_d$.
To encode $\mu$, the idea is that a \emph{single} sample $g \sim \ncal(0, \Sigma)$ is unlikely to be too far
from $\mu$.
Specifically, if $\cE$ is the ellipsoid defined by $\Sigma$, centered at zero, then, with high probability,
$\mu \in g + O(\sqrt{d}) \cdot \cE$.
Thus, we build a net of the set $g + O(\sqrt{d}) \cdot \cE$ and the encoder sends
$g$ as well as the identity of the point in the net closest to $\mu$.

We now proceed to the formal proof.
We first prove a lemma that is similar to known results in random matrix theory (cf.~\cite[Corollary~4.1]{large_enclosed_ball}) but is tailored for our purposes.
Let us denote $a \cdot B_2^d\coloneqq\setst{ y \in \bR^d }{ \norm{y}_2 \leq a }$.

\begin{lemma}
	\LemmaName{litvak_newnew}
	Suppose that $q_1, \ldots, q_{m} \in \bR^d$ are i.i.d.~samples from some distribution $Q$
	that satisfies
	\(
	\DTV{Q}{\ncal(0,I_d)} \leq 2/3.
	\)
	Let $$T \:\coloneqq\: \{\: \pm q_i \::\: \|q_i\|_2 \leq 4\sqrt{d} \:\}.$$
	For a large enough absolute constant $C$, if $m\geq Cd$ then
	\[
	\prb{\: \frac 1 {20} B_2^d \subseteq \conv (T) \:} ~\geq~ 5/6.
	\]
\end{lemma}

\begin{proof}
	Let $S^{d-1} \coloneqq \setst{ y \in \bR^d }{ \norm{y}_2=1 }$.
	Consider the following statement:
	\begin{equation}
		\label{eq:enclosed_goal2new}
		\max_{q \in T} \, \abs{\inner{y}{q}} ~\geq~ \frac{1}{20}
		\qquad\forall y \in S^{d-1}.
	\end{equation}
	We first show that \eqref{eq:enclosed_goal2new} implies that
	$\frac{1}{20} B_2^d \subseteq \conv (T)$,
	which is the event which we wish to analyze.
	To see this, suppose, for the sake of contradiction, that $\frac{1}{20} B_2^d \not\subseteq \conv (T)$, so there exists some point $z \in \frac{1}{20} B_2^d \setminus \conv (T)$. 
	By the strict hyperplane separation theorem (e.g., \cite[Example~2.20]{boyd}), there exists some direction $y \in S^{d-1}$ such that $\inner{y}{z}>\inner{y}{q}$ for all $q\in T$. Since $\|z\|_2\leq1/20$, 
	we have $1/20\geq \inner{y}{z}$. This  means
	$1/20>\displaystyle\max_{q\in T}\inner{y}{q}=\max_{q\in T}|\inner{y}{q}|$,
	contradicting \eqref{eq:enclosed_goal2new}.
	
	Next, for each $y\in S^{d-1}$, let 
	$$H_y \coloneqq \setst{ x\in\R^d }{ \|x\|_2\leq 4\sqrt{d},\, |\inner{x}{y}|\geq\frac{1}{20} },$$  and let
	$\cH \coloneqq \left\{ H_y : {y\in S^{d-1}}\right\}.$ 
	Let $U\coloneqq \set{q_1,\dots,q_m}$ and observe that
	\eqref{eq:enclosed_goal2new} is equivalent to the event
	\[
	U \cap H \neq \emptyset \qquad \forall H\in\cH.
	\]
	So, to complete the proof we need only show that this event happens with probability at least $5/6$. Recalling that $U$ is an i.i.d.\ sample of size $m$ from $Q$, by the Vapnik-Chervonenkis inequality (see, e.g., \cite[Theorem 8.3.23]{hdp-vershynin}), for some absolute constant $c$,
	\begin{equation}\label{vcineq}
		\E \sup_{H\in\cH} \left| 
		Q(H) - \frac{|U\cap H|}{|U|}
		\right| ~\leq~ c \sqrt{\vc(\cH)/m}.
	\end{equation}
	Define $p \coloneqq \inf_{H\in\cH} Q(H) $.
	We claim that, to complete the proof of the lemma, it suffices to show that
	$\vc(\cH) = O(d)$
	and
	$p=\Omega(1)$. Indeed, if these statements are true, then we can choose 
	$m = 144c^2\cdot \vc(\cH)/p^2 = O(d)$ and the right-hand-side of
	\eqref{vcineq} becomes $p/12$.
	Markov's inequality would then imply that with probability at least $5/6$,
	for all $H\in\cH$ we have
	$Q(H) - \frac{|U\cap H|}{|U|} \leq p/2$, which implies
	\[
	\frac{|U\cap H|}{|U|} \geq Q(H) - p/2 \geq p/2 > 0,
	\]
	that is, $U\cap H$ is nonempty for all $H\in \cH$.
	Next we show that
	$\vc(\cH) = O(d)$
	and
	$p=\Omega(1)$.
	
	Since restricting the domain cannot increase the VC~dimension, the VC~dimension of $\cH$ is not more than that of 
	$\left\{
	 \{x\in\R^d : |\inner{x}{y}|\geq\frac{1}{20} \} :
	 {y\in S^{d-1}}
	\right\}$.
	Every set in this family is a union of two half-spaces (corresponding to $\inner{x}{y}\geq\frac{1}{20}$ and $\inner{x}{y}\leq-\frac{1}{20}$).
	The VC~dimension of the family of half-spaces is $d+1$ (see, e.g., \cite[Corollary~4.2]{devroye_book}), so, by \cite[Lemma~3.2.3]{Blumer:1989}, the VC~dimension of the family of pairwise unions of half-spaces is bounded by $4(d+1)\log_2(6)=O(d)$. 
	
	Finally, to show that $p=\Omega(1)$, let $g \sim \ncal(0,I_d)$ and note that for any $y\in S^{d-1}$, by the union bound,
	\begin{equation}
		\prb{g \in H_y} \geq
		\prb{\|g\|_2 \leq 4\sqrt{d}}
		-\prb{|\inner{g}{y}| < 1/20}
		\geq
		(1-e^{-3})-1/10 > 0.85,\label{eq85}
	\end{equation}
	where we have used that 
	$\prb{\|g\|_2 > 4\sqrt{d}}\leq \exp(-3)$ (by Corollary~\ref{cor:chi}) and that
	$\inner{g}{y}$ is a standard normal random variable so its probability density function is bounded by 1.
	Since 
	\(
	\DTV{Q}{\ncal(0,I_d)} \leq 2/3
	\),
	from \eqref{eq85} we obtain
	$Q(H_y)\geq 0.85-2/3 > 0.18$, completing the proof of the lemma.
\end{proof}

To prove \Lemma{coreformixtures}, we need to describe a $2/3$-robust compression scheme for $d$-dimensional Gaussians.
Accordingly, we consider a target distribution $Q$ 
such that there exists a Gaussian $\ncal(\mu,\Sigma)$
satisfying $\|Q-\ncal(\mu,\Sigma)\|_1\leq 2/3$.
Recall, from \eqref{eq:TVdef}, that this implies $\DTV{Q}{\ncal(\mu,\Sigma)}\leq 1/3$.

We may assume that $\Sigma$ has full rank, since there is a reduction from the case of rank-deficient $\Sigma$:
if the rank of $\Sigma$ is $\rho<d$,
then any $X \sim \ncal(\mu,\Sigma)$ lies in some affine subspace $\mathcal S$ of dimension $\rho$.
Thus, since $\DTV{Q}{\ncal(\mu,\Sigma)}\leq 1/3$,
any $X \sim Q$ lies in $\mathcal S$ with probability at least $2/3$.
By the Chernoff bound (Theorem~\ref{thm:chernoff}),
with high probability, after seeing $10d$ samples from $Q$, at least $\rho+1$ points
from $\mathcal S$  will appear in the sample.
We encode $\mathcal S$ using these samples, and for the rest of the process we work in this affine
subspace and discard outside points.

\paragraph{Definition of $v_1,\ldots,v_d,$ and $\Psi$.}
Since $\Sigma$ has full rank,
there exist an orthogonal set of vectors $v_1,\ldots,v_d$ 
satisfying $\Sigma = \sum_{i=1}^d v_i v_i\transpose$.
For convenience, let 
$\Psi = \Sigma^{1/2}$ be the unique positive definite square root of $\Sigma$.
Observe that
\begin{equation}
\label{eq:PsiAndSigma}
\Psi ~=~ \sum_{i=1}^d \frac{v_i v_i\transpose}{ \norm{v_i}_2 },
    \qquad\qquad
\Sigma^{-1} ~=~ \sum_{i=1}^d \frac{v_i v_i\transpose }{ \norm{v_i}_2^4 },
    \qquad\text{and}\qquad
\Psi^{-1} ~=~ \sum_{i=1}^d \frac{v_i v_i\transpose }{ \norm{v_i}_2^3 }.
\end{equation}

We next show how to encode the mean and the eigenvectors.

\begin{lemma}
\LemmaName{encodinggaussians}
Let $C$ be a sufficiently large absolute constant, and let $S$ be an i.i.d.\ sample of size
$2m=2Cd$ from $Q$, where
$\DTV{Q}{\ncal(\mu,\Sigma)}\leq 1/3$.
Then, with probability at least $2/3$,
one can encode vectors $\vhat_1,\dots,\vhat_d,\muhat\in\R^d$ satisfying
\begin{align}
\label{eq:EncodingGaussians1}
\|\Psi^{-1} (\vhat_j - v_j) \|_2 &~\leq~ \eps/24d^2 \qquad \forall j\in[d], \text{ and}
    \\
\label{eq:EncodingGaussians2}
\|\Psi^{-1} (\muhat - \mu) \|_2  &~\leq~ \eps/2,
\end{align}
using ${O}(d^2  \log(d/\eps))$ bits and the points in $S$.
\end{lemma}

\begin{proof}
The samples in $S$ will be denoted $X_1,\dots, X_{2m}$.

\paragraph{Encoding $\widehat{v}_1,\dots,\widehat{v}_d$.}
Define the ``standardized'' samples
$$Y_i ~\coloneqq~ \frac{1}{\sqrt2}\Psi^{-1} (X_{2i}-X_{2i-1})
    \qquad\forall i\in[m].$$
We claim that, for each $i\in[m]$,
$Y_i$ has TV distance at most $2/3$ from $\ncal(0,I)$. 
Indeed,    
If $X_{2i}$ and $X_{2i-1}$ had distribution $\ncal(\mu,\Sigma)$,
then $Y_i$ would have distribution $\ncal(0,I)$.
Instead, both $X_{2i}$ and $X_{2i-1}$ have TV distance at most $1/3$ from $\ncal(\mu,\Sigma)$, so, by Lemma~\ref{prop:TV}
and Proposition~\ref{fact:dtv}, $Y_i$ has TV distance at most $2/3$ from $\ncal(0,I)$.
Define the event 
$$
\cE \coloneqq
    \set{ \frac{1}{C} B_2^d \:\subseteq\: \conv \setst{ \pm Y_i }{ i \in \cI } },
    \text{ where }
\cI \coloneqq
    \setst{ i \in [m] }{ \|Y_i\|_2 \leq 4 \sqrt{d} }.
$$
Since $C$ is large, and in particular $C\geq20$, by \Lemma{litvak_newnew} we have $\prb{ \cE } \geq 5/6$.
Our encoding will assume  $\cE$ happens.

Fix some $j\in[d]$.
Referring to \eqref{eq:PsiAndSigma}, we see that
$\Psi^{-1} v_j = v_j/\|v_j\|_2$ has unit norm.
Since $\cE$ occurs, we can write
$$ \frac{\Psi^{-1} v_j}{C} \:=\: \sum_{i \in [m]} \theta_{j,i} Y_i$$
for some vector $\theta_j \in [-1,1]^m$ supported on $\cI$.
Applying $\Psi$ to both sides, we obtain
$$v_j = \frac {C}{\sqrt2} \sum_{i \in \cI} \theta_{j,i} (X_{2i}-X_{2i-1}).$$

To discretize $\theta_j$, consider the natural $(\eps/96 C m d^3)$-net for $[-1,1]^{m}$ in the
$\ell_{\infty}$ distance formed by the Cartesian product of one-dimensional nets,
namely,
$\{ n \eps/48 C m d^3 : 
- \lfloor 48 C m d^3 / \eps \rfloor \leq n \leq \lfloor 48 C m d^3 / \eps\rfloor \}^m$.
This net has size at most $(97 C md^3/\eps)^{m}$.
Recalling that $m = O(d)$, it follows that
any element of the net can be described using $O(d \log(d/\eps))$ bits.
Let $\widehat{\theta}_{j}$ be an element in the net that is closest to
$\theta_{j}$.
Since each $\theta_j$ is supported on $\cI$
and due to the structure of the net,
we may choose $\widehat{\theta}_{j}$ also to be supported on $\cI$.
Define
$$
\vhat_j \:\coloneqq\:
    \frac{C}{\sqrt 2} \sum_{i \in \cI}
        \widehat{\theta}_{j,i} (X_{2i}-X_{2i-1}).
$$
The vectors $\widehat{v}_1,\ldots,\widehat{v}_d$ are encoded using the points in $S$ and $\widehat{\theta}_1,\ldots,\widehat{\theta}_d$.
Encoding each $\widehat{\theta}_i$ requires $O(d \log(2d/\epsilon))$ bits,
so encoding all of them requires $O(d^2 \log(2d/\eps))$ bits.

The error of this encoding is
\begin{align}
\nonumber
\left\|\Psi^{-1} (\vhat_j - v_j) \right\|_2 
    &~=~
    \frac{C}{\sqrt 2}\left\| \sum_{i \in \cI}
    (\theta_{j,i}-\widehat{\theta}_{j,i}) \Psi^{-1} (X_{2i}-X_{2i-1})\right\|_2
\\\nonumber
    &~\leq~
    \frac{C}{\sqrt 2} \card{\cI}
        \big(\max_{i \in \cI} |\theta_{j,i}-\widehat{\theta}_{j,i}| \big)
        \big(\max_{i\in \cI} \sqrt 2 \|Y_i\|_2 \big)
\intertext{By the definition of $\widehat{\theta}_j$, we have
$\norm{\widehat{\theta}_j - \theta_j}_\infty \leq \eps/96 C m d^3$.
By the definition of $\cI$, we have $\norm{Y_i}_2 \leq 4 \sqrt{d}$,
leading to the bound
}
    \left\|\Psi^{-1} (\vhat_j - v_j) \right\|_2&~\leq~ \frac{C}{\sqrt 2} m 
        \left( \frac{\eps}{96 C m d^3} \right) 
        \big(4 \sqrt 2 \sqrt d \big)
\label{eq:PsiVError}
    ~\leq~ \frac{\eps}{24d^2},
\end{align}
establishing \eqref{eq:EncodingGaussians1}.

\paragraph{Encoding $\muhat$.}
Let $Z_i \coloneqq \Psi^{-1}(X_i-\mu)$ and observe that $Z_i$ has a distribution with TV distance at most $1/3$ to $\ncal(0,I)$.
Define the event
$$
\cE' \:\coloneqq\: \set{\, \min\{\|Z_1\|_2 , \|Z_2\|_2\} \leq 4 \sqrt d \,}.
$$
Corollary~\ref{cor:chi} implies that
$$\Pr [\|Z_i\|_2 \geq 4\sqrt d ] \:\leq\: \exp(-3)+1/3 \:<\: \sqrt{1/6}.$$
Thus $\prb{ \cE' } \geq 5/6$.
Our encoding will assume that  $\cE'$ happens.

By symmetry, assume that $\|Z_1\|_2 \leq 4 \sqrt d$,
and suppose that $Z_1 = \sum_{j\in[d]} \lambda_j v_j / \norm{v_j}_2$.
Thus we have $(\lambda_1,\dots,\lambda_d)\in4\sqrt{d}B_2^d$.
Also, by the definition of $\Psi$, we have
$$
    \mu
        = X_1 - \Psi Z_1
        = X_1 - \left(\sum_{i=1}^d \frac{v_i v_i\transpose}{ \norm{v_i}_2 }\right) \left(\sum_{j=1}^{d} \lambda_j \frac{v_j}{\norm{v_j}_2}\right)
        = X_1 - \sum_{j\in[d]} \lambda_j v_j.
$$
To discretize $(\lambda_1,\dots,\lambda_d)$,
consider an $(\eps/3d)$-net for $4\sqrt{d} B_2^d$ of size $O(d^{1.5}/\eps)^d$
(see \Lemma{epsnetL2}).
Let $\hat\lambda=(\widehat{\lambda}_1,\dots,\widehat{\lambda}_d)$ be the closest element to $(\lambda_1,\dots,\lambda_d)$ in this net.
The encoding is
$$\muhat \:\coloneqq\: X_1 -
    \sum_{j\in[d]} \widehat{\lambda_j} \widehat{v_j}.
$$
The error of this encoding is
\begin{align*}
\left\|\Psi^{-1}(\mu-\muhat)\right\|_2
    &~=~ \left\|\sum_{j \in [d]} \Psi^{-1} 
        (\lambda_j v_j -\widehat{\lambda}_j  \widehat{v}_j)\right\|_2
\\
    &~\leq~ \sum_{j \in [d]} \left\|
        \widehat{\lambda_j} (\Psi^{-1}v_j - \Psi^{-1}\widehat{v}_j) 
      + (\lambda_j-\widehat{\lambda}_j) \Psi^{-1}v_j \right\|_2
\\
    &~\leq~ d \cdot \max_{j \in [d]} \left\{
          \left| \widehat{\lambda_j}\right| \cdot
          \left\|\Psi^{-1}v_j - \Psi^{-1}\widehat{v}_j\right\|_2
        + \left|\lambda_j-\widehat{\lambda}_j\right| \cdot
          \left\|\Psi^{-1}v_j\right\|_2 \right\}.
\intertext{By the definition of $\hat\lambda$,
we have $\max_j |\widehat{\lambda}_j| \leq 4 \sqrt{d}$ and
$\max_j |\lambda_j-\widehat{\lambda}_j|\leq\eps/3d$.
From \eqref{eq:PsiAndSigma} we have $\norm{\Psi^{-1} v_j}_2 = 1$.
Lastly, using \eqref{eq:PsiVError} we have 
$ \norm{\Psi^{-1} (\vhat_j - v_j)}_2 \leq \eps/24 d^2$,
leading to the bound
}
    \left\|\Psi^{-1}(\mu-\muhat)\right\|_2&~\leq~ d \cdot \left( 4\sqrt{d} \cdot \frac{\eps}{24d^2}
        + \frac{\eps}{3d} \cdot 1 \right)
        ~\leq~ \eps/2,
\end{align*}
establishing \eqref{eq:EncodingGaussians2}.
The encoding for $\muhat$ consists of
$X_1$, the already encoded $\hat v_1,\dots,\hat v_d$, and  $\hat\lambda$.
Since $\widehat{\lambda}$ comes from a net of size
$O(d^{1.5}/\eps)^d$, the additional number of required bits to
encode $\muhat$ is $O(d \log (d/\eps))$.
Finally, note that all encodings succeed so long as both $\cE$ and $\cE'$ occur,
which happens with probability at least $2/3$.
\end{proof}

\Lemma{coreformixtures} now follows immediately from the following lemma.

\begin{lemma}\LemmaName{gaussianstv}
Suppose that the vectors $\vhat_1,\ldots,\vhat_d, \muhat \in \bR^d$ satisfy
\begin{align}
\label{eq:gaussianstv1}
\|\Psi^{-1} (\vhat_j - v_j) \|_2 &~\leq~ \rho \leq 1/6d 
    \qquad\forall j\in[d], \text{ and} \\
\label{eq:gaussianstv2}
\|\Psi^{-1} (\muhat - \mu) \|_2  &~\leq~ \zeta.
\end{align}
Then 
$$\displaystyle
\DTV{\: \ncal\big(\mu,   \smallsum{i \in [d]}{} v_i     v_i\transpose\big)}
    {\: \ncal\big(\muhat,\smallsum{i \in [d]}{} \vhat_i \vhat_i\transpose\big) \:}
~\leq~ \frac{\sqrt{9d^3\rho^2+\zeta^2}}{2}.
$$
\end{lemma}
\begin{proof}
Recall the definition of log-det divergence: $\DLD{A}{B} \coloneqq \trace(B^{-1}A - I) - \log \det( B^{-1} A )$.
Define $\widehat{\Sigma} \coloneqq \sum_i \vhat_i \vhat_i\transpose$.
We will show that
\begin{equation}
\label{ldclose}
\DLD{\widehat{\Sigma}}{\Sigma} ~\leq~ 9 d^3 \rho^2.
\end{equation}
If this is true, then \Lemma{tvupbound} and \eqref{eq:gaussianstv2} yield
$$
\DTV{\ncal(\mu,\Sigma)}{\ncal(\muhat,\widehat{\Sigma})}^2
    ~\leq~
        \frac{1}{4} \Big(\DLD{\widehat{\Sigma}}{\Sigma} +
        (\mu-\muhat) \transpose \Sigma^{-1} (\mu-\muhat)\Big)
    ~\leq~ \frac{1}{4}(9d^3\rho^2+\zeta^2),     
$$
which completes the proof of the lemma.

Thus, we need only prove \eqref{ldclose}.
Recall from \eqref{eq:PsiAndSigma} that $\Psi = \Sigma^{1/2}$ is positive definite.
Define $B \coloneqq \Sigma^{-1/2}\hat\Sigma\Sigma^{-1/2}=\Psi^{-1} \hat\Sigma\Psi.$
We will show that $\|B-I\|_s\leq 3d\rho$.
Then \Lemma{SpectralToDLD} will imply that
$\DLD{\widehat{\Sigma}}{\Sigma} \leq 9 d^3 \rho^2$,
which establishes \eqref{ldclose}.

To complete the proof, note that by the triangle inequality
\begin{align*}
\|B - I\|_s ~=~
\norm{\sum_{i=1}^d
    (\Psi^{-1} \vhat_i \vhat_i\transpose \Psi^{-1} - \Psi^{-1} v_i v_i\transpose \Psi^{-1}) }_s
&~\leq~
\sum_{i=1}^d
\norm{ \Psi^{-1} \vhat_i \vhat_i\transpose \Psi^{-1} - \Psi^{-1} v_i v_i\transpose \Psi^{-1} }_s
\\&~=~
\sum_{i=1}^d
\|x_i x_i \transpose - y_i y_i \transpose\|_s,
\end{align*}
with $x_i \coloneqq \Psi^{-1} \vhat_i$
and $y_i \coloneqq \Psi^{-1} v_i$. From~\eqref{eq:PsiAndSigma} we have $\|y_i\|_2=\|\Psi^{-1} v_i\|_2=1$.
By~\eqref{eq:gaussianstv1}, $\|x_i-y_i\|_2\leq\rho$.
By applying the following simple lemma, we conclude that $\norm{B-I}_s \leq 3 d \rho$, completing the proof of the lemma.
\end{proof}

\begin{lemma}
\LemmaName{EuclideantoOperator}
Let $x$ and $y$ be vectors satisfying $\|y\|_2=1$ and $\|x-y\|_2\leq \eps \leq 1$. Then we have
$\|xx\transpose-yy\transpose\|_s\leq 3\eps$.
\end{lemma}
\begin{proof}
Suppose $x = y + z$ with $\|z\|_2\leq \eps$.
Then, 
\[
\|xx\transpose-yy\transpose\|_s
=
\|y z\transpose + z y\transpose + z z\transpose\|_s
\leq 
\|y z\transpose\|_s + \|z y\transpose\|_s + \|z z\transpose\|_s
\leq \eps + \eps + \eps^2 \leq 3\eps.
\]
The first inequality is the triangle inequality for the operator norm.
The second inequality uses the facts that $\|AB\|_s \leq \|A\|_s \cdot \|B\|_s$ for any two size-compatible matrices $A$ and $B$ (see, e.g., \cite[Fact~7(c) in Section~24.4]{linearalgebrahandbook}) and that, for any vector $v$, the operator norm of $v$ as a matrix coincides with its Euclidean norm as a vector.
\end{proof}

\section{The lower bound for Gaussians and their mixtures}
\SectionName{lower_bound}
In this section, we establish a lower bound 
of $\widetilde{\Omega}(d^2/\eps^2)$ for learning a single Gaussian, and then lift it to obtain a lower bound 
of $\widetilde{\Omega}(kd^2/\eps^2)$ for learning mixtures of $k$ Gaussians in $d$ dimensions.
Both of our lower bounds are for the realizable setting and thus also hold in the agnostic setting.

The high-level strategy for our proof is similar to that adopted in earlier work for mixtures of spherical Gaussians~\cite{spherical}.
The idea is to create a large number of distributions that are pairwise close in KL divergence (roughly $\eps^2$) but pairwise far in TV distance (roughly $\eps$).
An application of the following lemma will then yield the desired sample complexity lower bound.

\begin{lemma}
\label{lem:fano_conc}
Let $\kappa:\R\rightarrow\R$ be a function and let $\fcal$ be a class of distributions such that, for all small enough $\eps>0$, there exist distributions $f_1,\dots,f_M \in \fcal$ with
\[
\DKL{f_i}{f_j} \:\leq\: \kappa(\eps) \quad\textnormal{ and }\quad
\DTV{f_i}{f_j} \:>\: 2\eps
\qquad
\forall i\neq j\in[M].\]	
Then any method that learns $\fcal$ to within total variation distance $\eps$ with success probability at least $2/3$ has sample complexity 
$\Omega\left( \frac{\log M}{\kappa(\eps) \log(1/\eps)} \right)$.
\end{lemma}

The preceding lemma is a straightforward consequence of the following result, which is a generalized form of Fano's inequality in information theory~\cite[Theorem~2.10.1]{cover}. 

\begin{lemma}[Generalized Fano's inequality \protect{\cite[Lemma 3]{bin_yu}}]\LemmaName{Fano}
	Let the distributions $f_1,\dots,f_M$ satisfy
	\[
	\DKL{f_i}{f_j} \:\leq\: \beta \quad\textnormal{ and }\quad
	\|f_i - f_j\|_1 \:>\: \alpha
	\qquad
	\forall i\neq j\in[M].\]
	Consider any density estimation method that
	has an explicit description of $f_1,\ldots,f_M$, receives $n$ i.i.d.\ samples from some $f_i$ without knowing $i$, then outputs an estimate $\widehat{f}$ for $f_i$.
	For each $i$, define $e_i \coloneqq \E\|f_i-\widehat{f}\|_1$ for the case in which the method receives samples from $f_i$.
	Then
	\[ \max_i e_i ~\geq~ \alpha \cdot \left( \frac{\log M - n\beta + \log 2}{2 \log M}\right).\]
\end{lemma}

\begin{proof}[Proof of \Lemma{fano_conc}]
Consider a distribution learning method $\acal$ that takes $m(\eps)$ samples and learns $\fcal$ to within total variation distance $\eps$ with success probability at least $2/3$.
Our goal is to prove 
$m(\eps) = \Omega(\log M / \kappa(\eps) \log(1/\eps))$.
Consider $M$ distributions $f_1,\dots,f_M$ satisfying the hypotheses.
We will design an estimator $\hat f$ for the finite class $\{f_1,\dots,f_M\}$ with sample complexity $n=km(\eps)$
and then apply~\Lemma{Fano} to this estimator.

Suppose we are given $km(\eps)$ samples from the unknown $f_j$.
We partition the sample into $k$ equal parts, and run $\acal$ on each part.
This gives us $k$ outputs $g_1,\dots,g_k$.
If some $f_{\ell}$ is within TV distance $\eps$ of more than half of these outputs, 
then we output $\hat f = f_{\ell}$;
otherwise, we output $\widehat{f}=f_1$.
Note that, since $f_1,\dots,f_M$ are $2\eps$-separated in the TV distance, there will be no ambiguity.
We now analyze the error of this estimator.

By the guarantee of $\acal$,
for each $i\in[k]$, with probability at least $2/3$ we have $\DTV{g_i}{f_j}\leq \eps$.
Let $\cE$ be the event that more than half of the $g_i$ satisfy $\DTV{g_i}{f_j}\leq \eps$.  
When $\cE$ happens, the estimate will be $\widehat{f}=f_j$, i.e., $\|f_j-\widehat{f}\|_1=0$.
On the other hand, by the Chernoff bound (Theorem~\ref{thm:chernoff}),
$\prob[\cE] \geq 1 - \exp(-\Omega(k))$.
Thus, the expected error is
$$e_j
~=~ \E\|f_j-\widehat{f}\|_1
~\leq~ \prob[\cE^c] \cdot 2
~\leq~ \exp(-\Omega(k))
\qquad\forall j \in [M].
$$
The total number of samples is $n=k m(\eps)$, so \Lemma{Fano} gives
\begin{align*}
	2\eps \cdot \left(\frac{\log M - (k m(\eps)) \kappa(\eps) + \log 2}{2 \log M}\right) ~\leq~ \exp(-\Omega(k)).
\end{align*}
Choosing $k = \Theta(\log (1/\eps))$ and
rearranging gives $m(\eps) = \Omega(\log M / \kappa(\eps) \log(1/\eps))$, as required.
\end{proof}

\subsection{The lower bound for learning a single Gaussian}

Our lower bound for learning a single Gaussian is the following theorem.

\begin{theorem}
\label{thm:lb}
Any algorithm that learns the class of $d$-dimensional Gaussians in the realizable setting within total variation distance $\eps$  and with success probability  $2/3$ has sample complexity
$\Omega\left( \frac{d^2}{\eps^2 \log(1/\eps)} \right)$.
\end{theorem}

\begin{proof}
To apply \Lemma{fano_conc}, we must create a large number $M$ of Gaussian distributions whose pairwise KL divergences are at most $\kappa(\eps)$ and whose pairwise TV distances are at least $2 \eps$.
We will accomplish this with parameters $M = 2^{\Omega(d^2)}$ and $\kappa = O(\eps^2)$, and \Lemma{fano_conc} yields the desired lower bound.

The existence of such $M$ distributions is shown using the probabilistic method.
Let us fix the parameters
$r = 9$ and $\lambda =  \Theta(\eps d^{-1/2} )$.
Assume, for simplicity, that $d/r$ is an integer.
For each $a \in [M]$, we pick $U_a$ to be a uniformly random $d \times d/r$ matrix with orthonormal columns. That is, 
$U_a$ consists of the first $d/r$ columns of a uniformly random $d\times d$ orthogonal matrix (a random matrix with Haar measure  on the group of $d\times d$ orthogonal matrices).
From this, we create the distribution
$$
f_a ~\coloneqq~ \ncal(0, \Sigma_a),
\qquad\text{where}\qquad
\Sigma_a = I_d + \lambda U_a U_a\transpose
\qquad\forall a \in [M].
$$
To apply \Lemma{fano_conc}, we must analyze the pairwise KL divergences and TV distances between $f_1,\ldots,f_M$.
Observe that, by the construction of $U_a$, for any fixed $d\times d$ orthogonal matrix $W$, the matrices $U_a$ and $WU_a$ have the same distribution, that is, $U_a$ is rotationally invariant (see, e.g., \cite[Section~1.2]{haar}).

\paragraph{Bounding the KL divergences.}
This analysis is straightforward since there is a closed-form expression for the KL divergence between Gaussians.
First, observe that all $\Sigma_a$ have the same spectrum:
there are $d/r$ eigenvalues equal to $1+\lambda$ and the remaining eigenvalues equal $1$.
Consequently, 
\begin{equation}
\label{eq:logdetzero}
\log \det (\Sigma_b \Sigma_a^{-1})
~=~ \log( \det \Sigma_b \cdot \det \Sigma_a^{-1})
~=~ 0.
\end{equation}
Next observe that
\begin{equation}
\label{eq:sigmainverse}
\Sigma_a^{-1}
~=~ I - \frac{\lambda}{1+\lambda} U_a U_a\transpose,
\end{equation}
which can be verified simply by multiplying by $\Sigma_a$.
Thus
\begin{align}
    2 \cdot \DKL{f_a}{f_b}
&~=~
    \trace(\Sigma_a^{-1}\Sigma_b - I)&
    \quad\text{(by \eqref{eq:logdetzero} and Lemma~\ref{lem:tvupbound})}
\notag\\&~=~
    \trace\Bigg( \Big(I - \frac{\lambda}{1+\lambda} U_a U_a\transpose\Big) (I + \lambda U_b U_b\transpose) - I \Bigg)
    &\quad\text{(by \eqref{eq:sigmainverse})}
\notag\\&~=~ 
    \trace\Bigg( \lambda U_b U_b\transpose - \frac{\lambda}{1+\lambda} U_a U_a\transpose  - \frac{\lambda^2}{1+\lambda}U_a U_a\transpose U_b U_b\transpose \Bigg)
\notag\\&~=~
    \lambda \cdot \frac{d}{r} \,-\, \frac{\lambda}{1+\lambda} \cdot \frac{d}{r} \,-\, 
    \frac{\lambda^2}{1+\lambda} \cdot \|U_a\transpose U_b\|_F^2
\notag\\&~\leq~
\lambda \cdot \frac{d}{r} \,-\, \frac{\lambda}{1+\lambda} \cdot \frac{d}{r} =
    \frac{\lambda^2 d}{(1+\lambda)r}
    ~\leq~ \frac{\lambda^2 d }{ r }
    ~=~O(\eps^2).\notag
\end{align}
This bound holds with probability $1$.

\paragraph{Bounding the TV distances.}
The remaining step is to show that $\DTV{f_a}{ f_b} = \Omega(\eps)$ for all $a \neq b$.
Then, by scaling $\eps$ by a constant factor,  \Lemma{fano_conc} completes the proof.

First we provide some intuition on why such a bound should hold.
Let $S_a$ be the subspace spanned by the columns of $U_a$.
A vector drawn from $\ncal(0, \Sigma_a)$ has a slightly larger projection onto $S_a$ than a vector drawn from $\ncal(0, \Sigma_b)$.
This reveals an event that has slightly higher probability under the former distribution than under the latter.
Recalling the definition of the TV distance as a supremum over events (see \eqref{eq:TVdef}), such an argument gives the desired lower bound on the TV distances up to logarithmic factors.
This approach was used in a preliminary version of this work~\cite[version 2]{gaussian_mixture_tr}, but it is fairly technical.

Here we use a simpler argument, formulated as \Lemma{frob_to_tv_lb} below, 
which shows that if $\|U_a\transpose U_b\|_F^2 \leq d/2r$ then $\DTV{f_a }{ f_b } = \Omega(\eps)$.
The condition $\|U_a\transpose U_b\|_F^2 \leq d/2r$ means that the columns of $U_a$ are nearly pairwise orthogonal to the columns of $U_b$, which intuitively should hold since $U_a$ and $U_b$ are chosen randomly.
This is formalized in \Lemma{UaUb_frob_bound} below, which states that, with positive probability, $\|U_a\transpose U_b \|_F^2 \leq d/2r$ for all $a\neq b$.
Then \Lemma{frob_to_tv_lb} implies that, by our choice of parameters, for all $a \neq b$ we have $\DTV{f_a}{ f_b} = \Omega\left( \min\{1,\lambda \sqrt{d/r} \}\right) = \Omega(\eps)$, completing the proof.
\end{proof}

The main technical lemma underlying our lower bound is \Lemma{UaUb_frob_bound}.

\begin{lemma}
\label{lem:UaUb_frob_bound}
Suppose $d\geq r \geq 9$.
There exists $M = 2^{\Omega(d^2/r^2)}$ such that the following holds.
Construct $M$ independent $d\times d/r$ matrices $\{U_a\}_{a=1}^{M}$ such that the columns of each $U_a$ are the first $d/r$ columns of a uniformly random $d\times d$ orthogonal matrix. 
Then, with positive probability, we have $\|U_a\transpose U_b \|_F^2 \leq d/2r$ for all distinct pairs $a, b\in[M]$.
\end{lemma}

For its proof, we will need a basic lemma about Gaussian matrices.

\begin{lemma}\LemmaName{lem:rotation_invariance}
	Let $G$ be a random matrix with i.i.d.\ Gaussian entries, and let $G= U \Sigma V\transpose$ be its singular value decomposition.
	(To make the decomposition unique, we assume that $0<\Sigma_{11}<\Sigma_{22}<\cdots$.)
	Then the matrices $U, \Sigma, V$ are mutually independent.
	Moreover, $U$ is a uniformly random orthogonal matrix.
\end{lemma}
\begin{proof}
	For any orthogonal matrix $A$, 
	by rotational invariance of the Gaussian distribution (see, e.g., \cite[Proposition~3.3.2]{hdp-vershynin}),
	we have $AG \equidist G$. Thus, letting $W,X$ be size-compatible uniformly random orthogonal matrices (i.e., having  Haar measure on the set of orthogonal matrices of appropriate sizes), independent of each other and of $G$, we have
	$G \equidist W G X\transpose
	= (WU) \Sigma (XV)\transpose$, where $WU, \Sigma$, and $XV$ are mutually independent, and $WU$ is a uniformly random orthogonal matrix, as required.
	(We have used the fact that, by the definition of Haar measure, e.g., \cite[Section~1.2]{haar}, if $W,X$ are independent uniformly random orthogonal matrices and $U,V$ are orthogonal matrices with compatible dimensions such that the products
$WU$ and $XV$ are well-defined, then $WU$ and $X V$ are also independent uniformly random orthogonal matrices.)
\end{proof}

\begin{proof}[Proof of Lemma~\ref{lem:UaUb_frob_bound}]
We will show that, for any two matrices $U_a$ and $U_b$ constructed independently as described,
with probability $1-O(\exp(-\Omega(d^2/r^2)))$ we have
$\|U_a\transpose U_b \|_F^2 \leq d/2r$.
The lemma then follows from the union bound.

Let $s=d/r$ and fix $a, b \in [M]$ with $a \neq b$.
By the rotational invariance of the distribution of $U_b$, we may assume without loss of generality that 
$U_a = \left[\begin{smallmatrix} I \\ 0 \end{smallmatrix}\right]$.
Thus $\|U_a\transpose U_b\|_F^2 \equidist \|U^{(s)}\|_F^2$,
where $U^{(s)}$ is an $s \times s$ principal submatrix of a uniformly random $d\times d$ orthogonal matrix $U$. (Alternatively, the columns of $U^{(s)}$ are the first $s$ coordinates of $s$ orthonormal vectors in $\R^d$ chosen uniformly at random.)
Hence, it suffices to show that $\|U^{(s)}\|_F^2 \leq d/2r$ with probability at least $1-O(\exp(-\Omega(d^2/r^2)))$.
The main difficulty is that $U^{(s)}$ does not have independent entries due to the orthonormality,
but intuitively it should behave similarly to a matrix with independent Gaussian entries.

Before proceeding, we review some useful facts about Gaussian matrices.
Let $G \in \R^{d \times s}$ be a Gaussian matrix with i.i.d.\ $\ncal(0, 1/d)$ entries.
Let $G = U_G \Sigma_G V_G\transpose$ be its singular value decomposition, where $U_G \in \R^{d \times s}$, $\Sigma_G, V_G \in \R^{s \times s}$, and the diagonal entries of $\Sigma_G$ are sorted ascendingly.
By \Lemma{lem:rotation_invariance},  $U_G$ is a uniformly random orthogonal matrix, so the top $s$ rows of $U_G$, denoted by $U_G^{(s)}$,
has the same distribution as $U^{(s)}$.
Let $G^{(s)}$ denote the top $s$ rows of $G$; then, $G^{(s)} = U_G^{(s)}\Sigma_G V_G\transpose$.
Moreover, by \Lemma{lem:rotation_invariance} again, $U_G$ is independent of $\Sigma_G,V_G$.

Let
$\sigma_{\min}(\cdot)$ and $\sigma_{\max}(\cdot)$ denote the smallest and largest singular values of a matrix, respectively.
Note that
$\sigma_{\min}(\Sigma_G)=\sigma_{\min}(G)$.
Theorem~2.13 in \cite{singularvalueconcentration} states that, for any $t>0$,
\begin{equation}
\prb{\sigma_{\min}(\Sigma_G)\leq 1 - 1/\sqrt{r} - t} \leq \exp(-t^2 d).
\label{eq:gordonnew}
\end{equation}

Finally, we will use that for any two size-compatible matrices $A,B$, we have \begin{equation}
\max\{ 
\sigma_{\min}(A) \|B\|_F,
\sigma_{\min}(B) \|A\|_F
\} \le
\|A B\|_F
\leq
\min\{ 
\sigma_{\max}(A) \|B\|_F,
\sigma_{\max}(B) \|A\|_F
\},\label{minF}
\end{equation}
see, e.g., \cite[Fact~7(c) in Section~24.4]{linearalgebrahandbook}.

We now proceed with the rest of the proof.
Recall that our goal is to show
that $\|U^{(s)}\|_F^2 \leq d/2r$ with probability at least $1-O(\exp(-\Omega(d^2/r^2)))$,
where $U^{(s)}$ is the $s \times s$ principal submatrix of a uniformly random orthogonal $d\times d$ matrix $U$.

The matrix $U$ is naturally related to $G$. Similarly, the matrix $U^{(s)}$ is naturally related to the Gaussian matrix $G^{(s)} \in \R^{s\times s}$.
More precisely, since $U_G$ is independent of $\Sigma_G, V_G$, we have
\begin{equation}
\label{eq:Gdr}
G^{(s)} ~=~ U_G^{(s)} \Sigma_G V_G \transpose ~\equidist~ U^{(s)} \Sigma_G V_G \transpose.
\end{equation}
Observe that
$\E \|G^{(s)}\|_F^2 = s^2/d = d/r^2$, so our goal is to show that $\|U^{(s)}\|_F^2$ is unlikely to exceed this by a multiplicative factor of $r/2$.

By the definition of the singular value decomposition, $V_G$ is orthogonal, so all its singular values are 1, hence using~\eqref{minF} we get
$\|U^{(s)} \Sigma_G V_G \transpose\|_F=\|U^{(s)} \Sigma_G \transpose\|_F$.
The Frobenius norms $\|G^{(s)}\|_F$ and $\|U^{(s)}\|_F$ can be related as
\begin{align}
\|G^{(s)}\|_F
\equidist \|U^{(s)} \Sigma_G V_G \transpose\|_F = \|U^{(s)} \Sigma_G \|_F  \geq \sigma_{\min}(\Sigma_G) \|U^{(s)}\|_F, \notag
\end{align}
the first equality is by \eqref{eq:Gdr} and the inequality follows from~\eqref{minF}.

Since $\Sigma_G$ and $U^{(s)}$ are independent, for any real $\alpha$ and $\beta$ we have
\begin{align*}
	\prb{\|G^{(s)}\|_F^2 \geq \alpha \beta}
& \geq 
\prb{\sigma_{\min}(\Sigma_G)^2 \geq \alpha \text{ and }\|U^{(s)}\|_F^2 \geq \beta}\\&=
\prb{\sigma_{\min}(\Sigma_G) \geq \sqrt \alpha}\prb{\|U^{(s)}\|_F^2 \geq \beta}.
\end{align*}
In particular, setting $\alpha = 1/3$ and $\beta = d/2r$ gives
\[
\prb{\|U^{(s)}\|_F^2 \geq d/2r}
\leq \frac{\prb{\|G^{(s)}\|_F^2 \geq d/6r}}
{\prb{\sigma_{\min}(\Sigma_G) \geq 1/\sqrt 3}}.
\]
For the numerator, since $d \|G^{(s)}\|_F^2 \equidist \chi_{s^2}$ we have
\[\prb{\|G^{(s)}\|_F^2 \geq d/6r}
= \prb{\chi_{s^2} \geq d^2/6r}
\leq \prb{\chi_{s^2} \geq 3s^2/2}
\leq \exp(-s^2/25) = \exp(-\Omega(d^2/r^2)),
\]
where we have used $r\geq9$ for the first inequality and \Lemma{ChiSquaredTail} for the second inequality.
For the denominator, setting $t = 1-1/\sqrt{r}-1/\sqrt{3}>0.08$ in~\eqref{eq:gordonnew} gives
\[
{\prb{\sigma_{\min}(\Sigma_G) \geq 1/\sqrt 3}} \geq 1 - \exp(-t^2d) > 0.006,
\]
hence
$\prb{\|U^{(s)}\|_F^2 \geq d/2r} = O\left(\exp\left({-\Omega(d^2/r^2)}\right)\right)$,
completing the proof.
\end{proof}

\begin{lemma}
\label{lem:frob_to_tv_lb}
Suppose that $\lambda \leq 1/4$.
If $U_a$ and $U_b$ are $d\times d/r$ matrices satisfying $\|U_a\transpose U_b\|_F^2 \leq d/2r$, then $\DTV{\ncal(0,I_d + \lambda U_a U_a\transpose) }{ \ncal(0,I_d + \lambda U_b U_b\transpose)  } = \Omega\left( \min\{1,\lambda \sqrt{d/r}\}\right)$.
\end{lemma}

\begin{proof}
We will use the following approximate characterization of the TV distance between two zero-mean Gaussians: for positive definite matrices $\Sigma_a$ and $\Sigma_b$,
\[\DTV{\ncal(0,\Sigma_a)}{\ncal(0,\Sigma_b)} ~=~ \Theta\left( \min \{1, \|\Sigma_a^{-1/2}\Sigma_b\Sigma_a^{-1/2} - I\|_F\} \right).\]
This result appears in \cite[Theorem~1.1]{tv_distance_gaussians}; see also \cite[Corollary~2]{ulyanov}.
Hence to complete the proof it suffices to show that 
$\|\Sigma_a^{-1/2}\Sigma_b\Sigma_a^{-1/2}-I\|_F \geq \frac45 \lambda \sqrt{d/r}$.
Since
$\Sigma_a^{-1/2}\Sigma_b\Sigma_a^{-1/2}-I
=
\Sigma_a^{-1/2}
(\Sigma_b-\Sigma_a)
\Sigma_a^{-1/2}$,
applying the left inequality in~\eqref{minF} twice gives
\[
\|\Sigma_a^{-1/2}\Sigma_b\Sigma_a^{-1/2}-I\|_F
\geq \sigma_{\min}(\Sigma_a^{-1/2})^2  \|\Sigma_b-\Sigma_a\|_F.
\]
Since the eigenvalues of $\Sigma_a$ are $1$ and $1+\lambda$, we have 
$\sigma_{\min}(\Sigma_a^{-1/2})=(1+\lambda)^{-1/2}\geq\sqrt{4/5}$ as $\lambda\leq1/4$.
Moreover,
$\|\Sigma_b-\Sigma_a\|_F
= \lambda \|U_bU_b\transpose - U_aU_a\transpose\|_F$,
and since $U_b U_b\transpose -
U_a U_a\transpose$ is symmetric, we have
\begin{align*}
\|U_b U_b\transpose -
U_a U_a\transpose\|_F^2 & = \trace ( (U_b U_b\transpose -
U_a U_a\transpose)(U_b U_b\transpose -
U_a U_a\transpose))\\
& = 
\trace(U_b U_b\transpose U_b U_b\transpose)
+
\trace(U_a U_a\transpose U_a U_a\transpose)
-
\trace(U_b U_b\transpose U_a U_a\transpose)
-
\trace(U_a U_a\transpose U_b U_b\transpose)\\
& = 
\trace(U_b U_b\transpose)
+
\trace(U_a U_a\transpose)
-
\trace(U_b\transpose U_a U_a\transpose U_b )
-
\trace(U_a\transpose U_b U_b\transpose U_a ) \\
& = 
d/r+
d/r-
\|U_a\transpose U_b\|_F^2
-
\|U_a\transpose U_b\|_F^2 \\ &\geq 2d/r - 2d/2r  = d/r,
\end{align*}
hence
$\|\Sigma_a^{-1/2}\Sigma_b\Sigma_a^{-1/2}-I\|_F \geq \frac45 \lambda \sqrt{d/r}$, completing the proof of the lemma.
\end{proof}

\subsection{The lower bound for learning Gaussian mixtures}
For proving our lower bound for mixtures, \Theorem{lower_bound},
we will need a standard result.
This lemma follows from the Gilbert-Varshamov bound in coding theory (see, e.g., \cite[Theroem~5.2.6]{codingtheory}); we include a proof for completeness.

\begin{lemma}
\LemmaName{packing}
Let $T \geq 4$ and $k \in \N$.
There exists a set of $k$-tuples $\cX \subseteq [T]^k$ such that $|\cX| \geq T^{k/4}$ and every distinct $x,y \in \cX$ differ in at least $k/4$ components.\end{lemma}

\begin{proof}
For any $x\in [T]^k$, the number of $k$-tuples in $[T]^k$ that differ with $x$ in at most $k/4$ components is bounded by
\[
\sum_{i=0}^{k/4} \binom{k}{i} (T-1)^i < T^{k/4} \sum_{i=0}^{k/4} \binom{k}{i} \leq (4eT)^{k/4} < T^{3k/4},
\]
where we have used the inequality $\sum_{i=0}^{m} \binom{n}{i} \leq (en/m)^m$, valid for all $1\leq m \leq n$ (see, e.g., \cite[Exercise~0.0.5]{hdp-vershynin}).
We give an iterative algorithm to build $\cX$: choose an arbitrary $k$-tuple from
$[T]^k$, put it in $\cX$, remove from $[T]^k$ the $k$-tuples that differ with the chosen $k$-tuple in at most $k/4$ components, and repeat.
By the above calculation, the size of the final set $\cX$ will be at least $T^{k/4}$. 
\end{proof}

\Theorem{lower_bound} follows immediately from the following result.

\begin{theorem}
\label{thm:lbmixture}
Any algorithm that learns the class of mixtures of $k$ Gaussians in $\R^d$ in the realizable setting within total variation distance $\eps$ and with success probability at least $2/3$ has sample complexity
$\Omega\left( \frac{kd^2}{\eps^2 \log(1/\eps)} \right)$.
\end{theorem}

\begin{proof}
Recall that $\gcal_{d,k}$ denotes the class of $k$-mixtures of $d$-dimensional Gaussian distributions. As we will use \Lemma{fano_conc} again to obtain the sample complexity lower bound,
it suffices to construct $2^{\Omega(kd^2)}$ distributions in $\gcal_{d,k}$ with pairwise KL divergences $O(\eps^2)$ and pairwise TV distances $\Omega(\eps)$.
Our family of distributions will use the covariance matrices constructed in \Theorem{lb}.
Some care is required to ensure that the TV distances are large, and we will adopt some ideas used in earlier work for mixtures of spherical Gaussians \cite[Appendix C.2]{spherical}.

The proof of \Theorem{lb} shows that there exists a family of symmetric positive definite matrices $\Sigma_1, \ldots, \Sigma_T$ with $T = 2^{\Omega(d^2)}$ satisfying
\begin{subequations}
\begin{alignat}{2}
    \label{eq:LBKL}
    \DKL{\ncal(0, \Sigma_i)}{\ncal(0, \Sigma_{j})} &~\leq~ O(\eps^2 )
    &&\qquad\forall i \neq j
    \\\label{eq:LBTV}
    \DTV{\ncal(0, \Sigma_i)}{ \ncal(0, \Sigma_{j})} &~\geq~ \Omega(\eps)
    &&\qquad\forall i \neq j
    \\\label{eq:LBNorm}
    \|\Sigma_i\|_s &~\leq~ 2 &&\qquad\forall i.
\end{alignat}
\end{subequations}

Next we will create a family of distributions in $\gcal_{d,k}$, where each Gaussian in each mixture uses one of these $\Sigma_i$ as its covariance matrix.
However, there is a tension.
On the one hand, we want any two of these mixture distributions to use disjoint sets of covariance matrices so that the TV distance between the mixtures is large.
On the other hand, this constraint reduces the number of mixture distributions we can create, while we want many distributions in order to maximize the lower bound.
This tension is resolved by a compromise obtained via error-correcting codes (\Lemma{packing}).

The construction proceeds as follows.
First, we pick $\mu_1,\ldots,\mu_k \in \R^d$, which will serve as the means for the Gaussians.
We choose them to be far apart: for some $\Delta$, to be chosen later, we pick them in such a way that $\| \mu_i - \mu_j \|_2 \geq \Delta$ for all $i \neq j$.
Each mixture distribution will be a uniform mixture of $k$ Gaussians, for which the $i$th Gaussian has mean $\mu_i$.
The choice of covariance matrices is determined using the error-correcting code.
Specifically, let $\cX \subset [T]^k$ be a set as in \Lemma{packing} above.
The family of mixture distributions is
\[
	\cF \coloneqq \left\{\: f_x \::\: x \in \cX \:\right\},
	\qquad\text{where}\qquad
	f_x ~\coloneqq~ \frac{1}{k} \Big( \ncal(\mu_1, \Sigma_{x_1}) + \cdots + \ncal(\mu_k, \Sigma_{x_k}) \Big).
\]
As desired, we have $\card{\cF} = T^{\Omega(k)} = 2^{\Omega(k d^2)}$.

To analyze $\cF$, the first task is to prove the pairwise KL divergence upper bound. This is straightforward.
Fix distinct $x,y \in \cX$.
For each $i$, \eqref{eq:LBKL} shows that
$$\DKL{\ncal(\mu_i, \Sigma_{x_i})}{\ncal(\mu_i, \Sigma_{y_i})} ~=~
\DKL{\ncal(0, \Sigma_{x_i})}{\ncal(0, \Sigma_{y_i})} 
~\leq~ O(\eps^2 ).$$
The convexity of KL divergence \cite[Theorem~2.7.2]{cover} then shows that $\DKL{f_x}{f_y} \leq O(\eps^2)$.

The remaining task is to prove 
$\DTV{f_x}{f_y} \geq \Omega(\eps)$ for all distinct $f_x, f_y \in \cF$.
The intuition is as follows.
Say that index $i \in [k]$ \emph{disagrees} if $x_i \neq y_i$. 
Whenever $i$ disagrees, the $i$th Gaussian in $f_x$ and $i$th Gaussian in $f_y$ have TV distance $\Omega(\eps)$ by \eqref{eq:LBTV}.
Moreover, the total mixture weight apportioned to disagreeing indices is at least $1/4$, since \Lemma{packing} ensures that the number of disagreements is at least $k/4$, and each mixture uses uniform weights on its components.
Thus, the disagreeing coordinates should suffice to imply that the TV distance is $\Omega(\eps)$.
Proving this formally requires some care because each Gaussian is supported on all of $\R^d$, so there is interaction between all Gaussians involved in the mixtures.
However, choosing a large enough $\Delta$ ensures that the means are far apart, so the interaction is negligible.

Formally, for each $j\in[k]$, let $A_j' \subseteq \R^d$ be such that
\begin{equation}\label{eq:AjTV}
 \Pr_{g \sim \ncal(\mu_j, \Sigma_{x_j})}[g \in A_j'] - \Pr_{g \sim \ncal(\mu_j, \Sigma_{y_j})}[g \in A_j'] 
~=~ \DTV{\ncal(\mu_j, \Sigma_{x_j})}{\ncal(\mu_j, \Sigma_{y_j})},
\end{equation}
and define
$$
A_j ~=~ A_j' \cap B_j,
\qquad\text{where}\qquad
B_j ~=~ \setst{ x \in \R^d }{ \norm{x-\mu_j}_2 < \Delta/2 }.
$$
Note that the separation of $\mu_1, \ldots, \mu_k$ implies that the balls $B_1, \ldots, B_k$ are disjoint.
Consequently, the sets $A_1, \ldots, A_k$ are also disjoint.

Several preliminary inequalities are required concerning these events.
First, for each $i\in[k]$,
\begin{alignat}{2}
\nonumber
\Pr_{g \sim \ncal(\mu_i, \Sigma_{x_i})}[g \not\in B_i ]
&~=~
    \Pr_{g \sim \ncal(\mu_i, \Sigma_{x_i})}[ \|g-\mu_i\|_2^2 \geq (\Delta/2)^2]
\\\nonumber&~=~
    \Pr_{g \sim \ncal(0, \Sigma_{x_i})}[ \|g\|_2^2 \geq (\Delta/2)^2]
    &&\qquad\text{(translating to zero-mean)}
\\\nonumber&~\leq~
    \Pr_{g \sim \ncal(0, I_d)}[ \|g\|_2^2 \geq \Delta^2/8]
    &&\qquad\text{(by \eqref{eq:LBNorm})}
\\&~\leq~
    \eps^2 / k^2,
\label{eq:gNotInBi}
\end{alignat}
by applying \Lemma{ChiSquaredTail}
with $t=2 \ln(k/\eps)$
and choosing $\Delta$ to satisfy $ \Delta^2/8 = d + 2 \sqrt{d t} + 2 t $.
Inequality \eqref{eq:gNotInBi} also holds replacing $x_i$ with $y_i$.
Since $A_i' \setminus A_i \subseteq B_i^c$, \eqref{eq:gNotInBi} shows that
\begin{equation}\label{eq:AiDiff}
\Big\lvert~ \Pr_{g \sim \ncal(\mu_i, \Sigma_{x_i})}[g \in A_i ] - \Pr_{g \sim \ncal(\mu_i, \Sigma_{x_i})}[g \in A_i' ] ~\Big\rvert 
    ~\leq~ \Pr_{g \sim \ncal(\mu_i, \Sigma_{x_i})}[g \not\in B_i ]
    ~\leq~ \eps^2/k^2.
\end{equation}
This inequality also holds using $y_i$ instead of $x_i$.
For $i \neq j$, we have $A_j \subseteq B_i^c$, so
\begin{equation}
    \Pr_{g \sim \ncal(\mu_i, \Sigma_{y_i})}[g \in A_j]
~\leq~
\Pr_{g \sim \ncal(\mu_i, \Sigma_{y_i})}[g \not\in B_i ]
~\leq~
    \eps^2 / k^2.
\label{eq:gInAj}
\end{equation}
Finally, by \eqref{eq:AjTV}, \eqref{eq:AiDiff}
and the triangle inequality,
\begin{equation}
\label{eq:AjVsTV}
\Pr_{g \sim \ncal(\mu_j, \Sigma_{x_j})}[g \in A_j] - \Pr_{g \sim \ncal(\mu_j, \Sigma_{y_j})}[g \in A_j] ~\geq~ \DTV{\ncal(\mu_j, \Sigma_{x_j})}{\ncal(\mu_j, \Sigma_{y_j})} - 2\eps^2/k^2.
\end{equation}

The total variation distance is lower bounded as follows.
Let $A \coloneqq A_1 \cup \cdots \cup A_k$.
Then
\begin{alignat*}{2}
&\DTV{f_x}{f_y}
\\&~\geq~
    \Pr_{g \sim f_x}[g \in A] - \Pr_{g \sim f_y}[g \in A]
\\&~=~
    \sum_{j = 1}^k \left( \Pr_{g \sim f_x}[g \in A_j] - \Pr_{g \sim f_y}[g \in A_j] \right)
    &&\qquad\text{(by disjointness of the $A_j$)}
\\&~=~
    \frac{1}{k} \sum_{j = 1}^k \sum_{i=1}^k \left( \Pr_{g \sim \ncal(\mu_i, \Sigma_{x_i})}[g \in A_j] - \Pr_{g \sim \ncal(\mu_i, \Sigma_{y_i})}[g \in A_j] \right)
    &&\qquad\text{(expanding $f_x$ and $f_y$)}
\\&~=~
    \frac{1}{k} \sum_{j = 1}^k \left( \Pr_{g \sim \ncal(\mu_j, \Sigma_{x_j})}[g \in A_j] - \Pr_{g \sim \ncal(\mu_j, \Sigma_{y_j})}[g \in A_j] \right)
    &&\qquad\text{(summands with $i=j$)}
    \\&\quad+~
    \frac{1}{k} \sum_{j = 1}^k \sum_{i \neq j} \Big(
    \underbrace{\Pr_{g \sim \ncal(\mu_i, \Sigma_{x_i})}[g \in A_j]}_{\geq 0} -
    \underbrace{\Pr_{g \sim \ncal(\mu_i, \Sigma_{y_i})}[g \in A_j]}_{\leq \eps^2/k^2 ~\text{by \eqref{eq:gInAj}}}
    \Big)
    &&\qquad\text{(summands with $i \neq j$)}
\\&~\geq~
    \frac{1}{k} \sum_{j = 1}^k \left( \Pr_{g \sim \ncal(\mu_j, \Sigma_{x_j})}[g \in A_j] - \Pr_{g \sim \ncal(\mu_j, \Sigma_{y_j})}[g \in A_j] \right) - \eps^2
\\&~\geq~
    \frac{1}{k} \sum_{j = 1}^k \Big( \DTV{\ncal(\mu_j, \Sigma_{x_j})}{\ncal(\mu_j, \Sigma_{y_j})} - 2\eps^2/k^2 \Big) - \eps^2
    &&\qquad\text{(by \eqref{eq:AjVsTV})}
\\&~\geq~ 
    \frac 1 k (k/4) \Omega(\eps) - 3\eps^2 
    ~=~ \Omega(\eps),
\end{alignat*}
where the last inequality is because $\DTV{\ncal(\mu_j, \Sigma_{x_j})}{\ncal(\mu_j, \Sigma_{y_j})} \geq \Omega(\eps)$ whenever $x_j \neq y_j$ (see~\eqref{eq:LBTV}), which is the case for at least $k/4$ of the indices $j$.
\end{proof}

\section{Discussion and open problems}\label{sec:open}

We have built a connection between distribution learning and compression.
Another concept related to compression is that of \emph{core-sets}. The idea of core-sets is to summarize the training data, using a small subset of them, in a way that \emph{any algorithm} minimizing the error on the subset will have small error on the whole set. 
Core-sets have been used in maximum likelihood estimation~\cite{lucic2017training}
and to solve clustering problems~\cite{sohler2018strong}.
There is one important distinction between compression and core-sets: compression can be more powerful since it can use complex class-specific decoders.

Our work opens several avenues for further research.

\subsubsection*{Eliminating the polylogarithmic factors}
Our sample complexity lower and upper bounds (in both axis-aligned and general cases) differ by multiplicative polylogarithmic factors.
Can one remove these factors? 
In this direction, we propose the following conjecture. (See Page~\pageref{rates} for the relevant definitions.)

\begin{conjecture}
	The minimax estimation rate for $\gcal_{d,k}$ and $\acal_{d,k}$ is $\Theta\left(\sqrt{kd^2/n}\right)$ and $\Theta\left(\sqrt{kd/n}\right)$, respectively.
\end{conjecture}

Note that this conjecture is true for the case $k=1$ (see~\cite{lower_bound_improved}).
For general $k$, the lower bound
for $\gcal_{d,k}$ follows from the proof of \Theorem{lower_bound} (see Theorem~\ref{thm:lbmixture}),
and the lower bound for $\acal_{d,k}$ follows from the lower bound proof in~\cite{spherical}.
The upper bounds hold up to polylogarithmic factors (see \Theorem{upper_bound}
and
\Theorem{upper_bound_axis_aligned}).

\subsubsection*{Learning mixtures of sparse Gaussians}
Our main results imply that the sample complexity for learning mixtures of \emph{axis-aligned} Gaussians is smaller than that of mixtures of general Gaussians by a factor of $\widetilde{\Theta}(d)$.
Can one interpolate between these two extremes by exploiting some notion of sparsity of the target distribution?

Consider the class of $d$-dimensional Gaussians whose inverse covariance matrices have at most $m$ off-diagonal nonzero entries.
Note that $m$ measures the amount of correlation between the Gaussian components:
if we build a probabilistic graphical model (also known as a Markov random field) whose nodes are the Gaussian components, then an axis-aligned Gaussian  corresponds to an empty graph with no correlation between the components, in which case $m=0$,
and a Gaussian with fully correlated components  corresponds to the complete graph, in which case $m=\binom{d}{2}$.
In general, $m$ counts the number of edges in this graph
(see~\cite[Proposition~5.2]{Lauritzen} for a proof).

The sample complexity of learning with respect to the class of $d$-dimensional Gaussians whose inverse covariance matrices have at most $m$ off-diagonal nonzero entries is $\widetilde{\Theta}((m+d)/\eps^2)$
(see~\cite{lower_bound_improved}).
This result on learning a single Gaussian and the fact that in some applications the underlying Gaussians are sparse motivates the following question: can one extend the bound of \cite{lower_bound_improved} to mixtures of Gaussians, obtaining sample complexity bounds that depend on some notion of sparsity of the mixture components?

\subsubsection*{Polynomial time algorithms for learning mixtures of Gaussians}
The running time of our density estimation algorithm is
$2^{kd^2 \polylog(d, k, 1/\eps,1/\delta)}$, which is not polynomial in the problem parameters. 
An important open question is whether there exists an algorithm for learning mixtures of Gaussians that runs in time $\poly(k, d, 1/\eps, 1/\delta)$ (see also \cite[Open Problem~15.5]{Diakonikolas2016}).
If the covariance matrices for all the Gaussians are multiples of the identity matrix (known as spherical Gaussians), \cite{spherical} gives an algorithm with running time that is polynomial in $d$ and $1/\eps$ but exponential in $k$.
On the other hand, for mixtures of general Gaussians, it is shown in \cite{gaussian_mixture} that
no polynomial time (in all the parameters) algorithm exists for the case that the learner has access to the distribution only via \emph{statistical queries}. (See~\cite{gaussian_mixture} for the definition of this model.)

\subsubsection*{What if $k$ is not known?}
Our density estimation algorithms assume that $k$ is given as input, while in some applications $k$ might be unknown.
One approach is to perform a binary search: run our algorithm for $k=1,2,4,8,\dots$,
and stop as soon as the output of our algorithm has total variation distance less than $\eps$ from the target distribution.
Unfortunately, it is not clear how to approximate this total variation distance. 
It is also not clear how to apply this approach to the robust learning scenario.

\subsubsection*{Is robust compression closed under taking mixtures?}
Lemma~\ref{lem:compressmixtures}
states that for any distribution class $\fcal$,
non-robust compression of $\fcal$
implies non-robust compression of $\kmix(\fcal)$.
Does an analogous statement hold for robust compression?
That is,
does robust compression of $\fcal$ imply robust compression of $\kmix(\fcal)$,
for a general class $\fcal$?

\subsubsection*{Sample complexity for learning with respect to the KL divergence}
Theorem~\ref{thm:kl_bad} states that there does not exist a function 
$g(k,d,\eps)$ such that there exists an algorithm that upon receiving $g(k,d,\eps)$ i.i.d.\ samples from an unknown $f\in\gcal_{d,k}$, outputs $\widehat{f}$ such that
$\DKL{f}{\widehat{f}} \leq \eps$ with probability more than 0.02.
Recalling that the KL divergence is asymmetric, we pose the following question:
what is the smallest function
$g(k,d,\eps,\delta)$ such that there exists an algorithm that, upon receiving $g(k,d,\eps,\delta)$ i.i.d.\ samples from an unknown $f\in\gcal_{d,k}$, outputs $\widehat{f}$ such that
$\DKL{\widehat{f}}{f} \leq \eps$ with probability at least $1-\delta$?

\subsubsection*{Characterizing the sample complexity of learning a  class of distributions}
A central open problem in distribution learning and density estimation is characterizing the sample complexity of learning a distribution class (\cite[Open Problem 15.1]{Diakonikolas2016}).
An insight from supervised learning theory is that the sample complexity of learning a class (of concepts, functions, or distributions) is typically proportional, up to logarithmic factors, to (some notion of) intrinsic dimension of that class divided by $\eps^2$, where $\eps$ is the error tolerance. 
For the case of  binary classification, the intrinsic dimension is captured by the VC~dimension of the concept class (see~\cite{vapnik2015uniform,Blumer:1989}). For the case of distribution learning with respect to ``natural'' parametric classes, we expect this dimension to be equal to the number of parameters. 
This is indeed true for the class of Gaussians (which have $d^2$ parameters) and axis-aligned Gaussians (which have $d$ parameters), 
and we showed in this paper that it holds for their mixtures as well 
(which have $kd^2$ and $kd$ parameters, respectively).

One may wonder if the VC~dimension of the Yatracos family associated with a class of distributions can characterize its sample complexity, but it is not hard to come up with examples where this VC~dimension is infinite while the class can be learned with finite samples.
Covering numbers do not characterize the sample complexity either: for instance, the class of Gaussians does not have a finite covering number in the TV metric, nevertheless it is learnable with finitely many samples.
Thus, we leave characterizing the sample complexity of learning a  class of distributions as an important open problem.

\subsubsection*{Do learnable classes have bounded compression schemes?}
In binary classification, the combinatorial notion of Littlestone-Warmuth compression has been shown to be sufficient \cite{littlestone1986relating} and necessary \cite{moran2016sample} for learning. In this work, we showed that the new but related notion of distribution compression is sufficient for distribution learning (Theorem~\ref{thm:compression}). Whether the existence of compression schemes is necessary for learning a class of distributions remains an intriguing open problem.
In this direction, we conjecture the following converse for Theorem~\ref{thm:compression}.
Let $m_{\fcal}(\eps,\delta)$ denote the sample complexity function associated with learning the class $\fcal$ of distributions (see
Definition~\ref{def:realizablelearning}).

\begin{conjecture}
There exists a universal constant $C$ such that 
any class $\fcal$ admits  an\\
$ ( 
C \eps^2m(\eps) \log(\eps^2m(\eps))^C,
C \eps^2m(\eps) \log(\eps^2m(\eps))^C,
C m(\eps) \log(m(\eps))^C)$
(non-robust) compression, where $m(\eps) \coloneqq m_{\fcal}(\eps,\frac12)$.
\end{conjecture}

The value $\eps^2m(\eps)$ is a candidate for the notion of ``intrinsic dimension'' of the class.
We also propose the following weaker conjecture.

\begin{conjecture}
There exists universal polynomials $P,Q$ and $R$ such that 
any class $\fcal$ admits 
a\\
$ ( 
P(\eps^2m(\eps), \log(1/\eps)),
Q(\eps^2m(\eps), \log(1/\eps)),
R(m(\eps), 1/\eps)
)$
(non-robust) compression, where $m(\eps) \coloneqq m_{\fcal}(\eps,\frac12)$.
\end{conjecture}

\begin{acks}
We thank the anonymous referees of the Journal of the ACM for their valuable comments which have substantially improved the presentation,
and for proposing the new proof of \Lemma{litvak_newnew} 
which resulted in logarithmic improvements in
\Lemma{coreformixtures}
and \Theorem{upper_bound}.
We are grateful to
Yaoliang Yu, for pointing out a mistake in an earlier version of this paper,
and Luc Devroye, for fruitful discussions.
Hassan Ashtiani and
Nicholas Harvey 
were supported by NSERC Discovery Grants.
Christopher Liaw was supported by an NSERC graduate award.
Abbas Mehrabian was supported by a CRM-ISM
postdoctoral fellowship and an IVADO-Apog\'ee-CFREF postdoctoral fellowship.
Yaniv Plan was supported by NSERC grant 22R23068.
This work started during the Foundations of Machine Learning program at The Simons Institute for the Theory of Computing in Spring 2017.
\end{acks}


\appendix

\section{Proofs of Theorem~\ref{thm:kl_bad} and Theorem~\ref{thm:lp_bad}}
\label{sec:parameterproofs}
\label{app:kl_lp_bad}
In this section, let $\nu$ denote the Lebesgue measure on $\bR$.
We will first prove Theorem~\ref{thm:kl_bad}.
To that end, we begin with a simple calculation that will be useful later.
\begin{lemma}
	\label{lem:small_kl}
    Suppose $I \subseteq \bR$ satisfies $\nu(I) \geq \gamma$.
    Moreover, let $f, h \colon \bR \to \bR_{\geq 0}$ be measurable density functions such that
    $f(x) \geq \beta$ and $h(x) \leq \alpha$ for all $x \in I$, and $f(x) > 0$ for all $x \in \bR$.
    Then $\DKL{f}{h} \geq \gamma \beta \log(\beta / \alpha) - 1/e$.
\end{lemma}
\begin{proof}
    Write
    \[
        \DKL{f}{h} =
        \integral{I}{}{f(x) \log \frac{f(x)}{h(x)}}{x} +
        \integral{I^c}{}{f(x) \log \frac{f(x)}{h(x)}}{x}.
    \]
    For the first integral, we have
    \[
        \integral{I}{}{f(x) \log \frac{f(x)}{h(x)}}{x}
        \geq 
        \integral{I}{}{\beta \log \frac{\beta}{\alpha}}{x}
        \geq \gamma \beta \log(\beta / \alpha).
    \]

    Next, we bound the second integral and show that it has value at least $-1/e$, which completes the proof.
    Let $F = \integral{I^c}{}{f(x)}{x}$ and $H = \integral{I^c}{}{h(x)}{x}$.
    Note that $F \geq 0$ as $f(x) > 0$ for all $x \in \bR$.
    If $H = 0$ then $h(x) = 0$ almost everywhere~on $I^c$ so the second integral is $+\infty$.
    So, assume that $H > 0$.
    Then $f/F$ and $h/H$ are densities on $I^c$.
    Hence, we have
    \begin{align*}
        \integral{I^c}{}{f(x) \log \frac{f(x)}{h(x)}}{x}
        =
        \underbrace{F \integral{I^c}{}{\frac{f(x)}{F} \log \frac{f(x) / F}{h(x) / H}}{x}}_{\geq 0} +
        F \integral{I^c}{}{\frac{f(x)}{F} \log \frac{F}{H}}{x}
        \geq F \log(F/H),
    \end{align*}
    where the inequality is because the KL divergence of two densities is always nonnegative.
    Since $H \leq 1$, we have $-\log(H) \geq 0$ so $F \log(F/H) = F \log F - F \log H \geq F \log F \geq - 1/e$, as required.
\end{proof}
\begin{proof}[Proof of Theorem~\ref{thm:kl_bad}]
We allow the algorithm to be randomized. Denote by $\mathcal{A}(X_1,\dots,X_m,R)$ the output of $\mathcal{A}$ given the sample $X_1,\dots,X_m$ and an independent source of randomness $R$.
We will first analyze the behavior of the algorithm when the true distribution is $\ncal(0,1)$ and show that there exists some $a' \in \bR$ for which the  algorithm's output puts almost no probability mass on around $a'$.
    We then show that if the true distribution is a carefully chosen mixture of $\ncal(0,1)$ and $\ncal(a',1)$, then the algorithm's output does not change with high probability, so it still  puts almost no mass on around $a'$; hence the KL divergence of the output and the true distribution is large.

    Define the parameters $\delta = \frac{0.01}{m}$, $\beta = \frac{\delta}{\sqrt{2\pi}} \exp(-1/32)$,
     and $\alpha = \beta \exp\left(\frac{-4\tau - 4/e}{\beta}\right)$.
    Let $X_1, \ldots, X_m \sim \ncal(0,1)$ and set $h = \acal(X_1, \ldots, X_m, R)$.
    Note that $h$ is random.
    Define the (random) set $H = \setst{x \in \bR}{h(x) \geq \alpha}$.
    Then $\nu(H) \leq 1/\alpha$.
    For any $a \in \bZ$, define $I_a = [a - 1/4, a + 1/4]$.
    Note that the $I_a$ are disjoint intervals.
    Hence $\sum_{a \in \bZ} \nu(I_a \cap H) \leq 1/\alpha$ deterministically so $\expect\left[\sum_{a \in \bZ} \nu(I_a \cap H)\right] \leq 1/\alpha$.
    Note that the left hand side of the inequality is an infinite sum while the right hand side is a finite number.
    Since expectation is linear, we can find $a' \in Z$ such that $\expect\left[\nu(I_{a'} \cap H)\right] \leq 1/400$.
    By Markov's Inequality, $\nu(I_{a'} \cap H) \leq 1/4$ with probability at least $0.99$.
    We condition on this event.

    Define $f = (1-\delta) \ncal(0,1) + \delta \cdot \ncal(a', 1)$ and
    note that $f$ is positive everywhere, and, for all $x \in I_{a'}$, we have $f(x) \geq \frac{\delta}{\sqrt{2\pi}} \exp(-1 / 32) = \beta$.
    Let $J_{a'} = I_{a'} \setminus H$.
    Then $\nu(J_{a'}) \geq \nu(I_{a'}) -\nu(I_{a'} \cap H) \geq 1/4$, and for all $x \in J_{a'}$ we have $f(x) \geq \beta$ and $h(x) < \alpha$.
    So,
    $$\DKL{f}{h} \geq \beta \log(\beta / \alpha)/4 - 1/e = \tau,$$ 
    where the inequality is by Lemma~\ref{lem:small_kl} and the equality is by the definition of $\alpha$.
    Hence, $\DKL{f}{h} \geq \tau$ with probability at least $0.99$.

    Note that $\DTV{f}{\ncal(0,1)} \leq \delta$.
    If $S = (X_1, \ldots, X_m)$ and $S' = (X_1', \ldots, X_m')$ where $X_i \sim \ncal(0,1)$ and $X_i' \sim f$ then Proposition~\ref{prop:TV} gives $\DTV{S}{S'} \leq m \delta = 0.01$.
    Hence, if $h = \acal(S,R)$ and $h' = \acal(S',R)$ then $\DTV{h}{h'} \leq 0.01$ so $\prob[\DKL{f}{h'} \geq \tau] \geq \prob[\DKL{f}{h} \geq \tau] - 0.01 \geq 0.98$, completing the proof.
\end{proof}

Next, we prove Theorem~\ref{thm:lp_bad}.
\begin{proof}[Proof of Theorem~\ref{thm:lp_bad}]
Define the parameters $\delta = \frac{0.01}{m}$, $\sigma^{p-1} = \frac{\delta^p}{\tau^p 6^p} \sqrt{\ln(9 / 2\pi)}$, and $M = 4 \sigma \sqrt{\ln(9/2\pi)}$.

Let $X_1, \ldots, X_m \sim \ncal(0,1)$ and set $h = \acal(X_1, \ldots, X_m, R)$, where, as in the proof of Theorem~\ref{thm:kl_bad},
$R$ is the algorithm's independent source of randomness.
Note that $h$ is random.
Next, let $H = \setst{x \in \bR}{h(x) \geq \delta / 6 \sigma}$.
Then $\nu(H) \leq 6\sigma / \delta$.
For $a \in \bZ$, define the intervals $I_a = [aM - M/4, aM + M/4]$ and note that $I_a$ are disjoint intervals.
Hence, $\sum_{a \in \bZ} \nu(I_a \cap H) \leq 6\sigma / \delta$ deterministically so $\expect\left[\sum_{a \in \bZ} \nu(I_a \cap H)\right] \leq 6\sigma / \delta$.
Note that the left hand side of the inequality is an infinite sum while the right hand side of the inequality is a finite number.
Since expectation is linear, there exists $a' \in Z$ such that $\expect\left[\nu(I_{a'} \cap H)\right] \leq M/400$.
By Markov's Inequality, $\nu(I_{a'} \cap H) \leq M/4$ with probability at least $0.99$.
We condition on this event.

Define $f = (1-\delta) \ncal(0,1) + \delta \cdot \ncal(a', \sigma^2)$, and,
note that, for all $x \in I_{a'}$, we have $f(x) \geq \delta \frac{1}{\sqrt{2\pi} \sigma} \exp(-(M/4)^2 / 2\sigma^2) = \delta/3\sigma$.

Let $J_{a'} = I_{a'} \setminus H$.
Then $\nu(J_{a'}) \geq M/2 - M/4 = M/4 = \sigma \sqrt{\ln(9/2\pi)}$, and, for all $x \in J_{a'}$, we have $f(x) \geq \delta/3\sigma$ and $h(x) \leq \delta/6\sigma$.
So,
\[
\norm{f - h}_p^p
 ~\geq~ \integral{J_{a'}}{}{\abs{f(x)-h(x)}^p}{x}
 ~\geq~ \frac{\delta^p}{(6\sigma)^p} \sigma \sqrt{\ln(9/2\pi)}
 ~=~ \frac{\delta^p}{6^p \sigma^{p-1}}  \sqrt{\ln(9/2\pi)}
 ~=~ \tau^p,
\]
where the last equality is by the  definition of $\sigma$.
Hence, $\norm{f - h}_p \geq \tau$ with probability at least $0.99$.

Note that $\DTV{f}{\ncal(0,1)} \leq \delta$.
If $S = (X_1, \ldots, X_m)$ and $S' = (X_1', \ldots, X_m')$ where $X_i \sim \ncal(0,1)$ and $X_i' \sim f$ then Proposition~\ref{prop:TV} gives $\DTV{S}{S'} \leq m \delta = 0.01$.
Hence, if $h = \acal(S,R)$ and $h' = \acal(S',R)$ then $\DTV{h}{h'} \leq 0.01$, so $\prob[\norm{f - h}_p \geq \tau] \geq \prob[\norm{f - h}_p \geq \tau] - 0.01 \geq 0.98$.
\end{proof}

\section{Proof of Lemma~\ref{mixturelemma}}\label{sec:lemmamixtureproof}
We restate the lemma for convenience.
\begin{lemma}
	Let $f = \sum_{i \in [k]} w_i f_i$ be a density with $(w_1, \ldots, w_k) \in \Delta_k$ and each $f_i\in\fcal$.
	Let $g$ be a density such that $\norm{g - f}_1 \leq \rho$.
	Then, we can write $g = \sum_{i \in [k]} w_i g_i$ such that each $g_i$ is a density and for any $r > 0$,
	\[
	\sum_{i \::\: \|g_i - f_i\|_1 > r} w_i \:<\: \rho/r.
	\]
\end{lemma}
Let $\cX\coloneqq\{x: g(x) < f(x)\}$.
Our goal is to ``transform'' each $f_i$ into another density $g_i$ so that $g = \sum_{i \in [k]} w_i g_i$.
Note that $\cX$ consists of the domain points on which $f$ exceeds $g$. Hence, to transform each $f_i$ into $g_i$, we will scale it down multiplicatively on points in $\cX$, and scale it up additively on points not in $\cX$.
These transformations need to be done carefully for each function $g_i$ to end up being nonnegative and integrate to 1.

To that end, we define 
\[
   g_i(x) ~\coloneqq~ \begin{cases}
       {f_i(x)g(x)}/{f(x)} &\quad\text{for $x\in \cX$,} \\
       f_i(x)+\Delta_i(x)  &\quad\text{for $x\notin \cX$,}
       \end{cases}
\]
where 
\begin{equation}
    \Delta_i(x) ~\coloneqq~ {\big(g(x)-f(x)\big)\left(\integral{\cX}{}{f_i(y) \cdot \frac{f(y)-g(y)}{f(y)}}{y} \right)} \bigg/ {\integral{\cX}{}{\big(f(y)-g(y)\big)}{y}}.\label{defDelta}
\end{equation}
Recall that $Z$ is the domain of $g$ and the densities in $\fcal$.
We now check that each $g_i$ is a density and that $g = \sum_{i \in [k]} w_i g_i$.
\begin{lemma}
    For all $i \in [k]$, $g_i$ is a density on $Z$.
\end{lemma}
\begin{proof}
    We first check that $g_i(x) \geq 0$ for all $x$.
    If $x \in \cX$, then $g_i(x) \geq 0$ because $f_i, g, f$ are all densities and hence nonnegative.
    If $x \notin \cX$, then $\Delta_i(x) \geq 0$ because $g(x) - f(x) \geq 0$ on $\cX^c$ and $f(x) - g(x) \geq 0$ on $\cX$.
    We now check that $\integral{Z}{}{g_i(x)}{x} = 1$.
    Since both $g$ and $f$ are densities, both integrate to $1$ over $Z$, and therefore
    \begin{equation}
    \label{eq:IntegralsEqual}
    \integral{\cX^c}{}{(g(x) - f(x))}{x} ~=~ \integral{\cX}{}{(f(x)-g(x))}{x}.
    \end{equation}
    The following calculation completes the proof.
    \begin{align*}
      \integral{\cX^c}{}{g_i(x)}{x}
    &=
      \integral{\cX^c}{}{(\Delta_i(x) + f_i(x))}{x}
    \\&=
        \frac{ \integral{\cX^c}{}{(g(x) - f(x))}{x} }{ \integral{\cX}{}{(f(y)-g(y))}{y} }
        \cdot \integral{\cX}{}{\left(f_i(y) \cdot \frac{f(y)-g(y)}{f(y)}\right)}{y}
        &\text{(by~\eqref{defDelta})}
        \\&\qquad+ \integral{\cX^c}{}{f_i(x)}{x} 
    \\&=
        \integral{\cX}{}{\left(f_i(y) \cdot \frac{f(y)-g(y)}{f(y)}\right)}{y}+\integral{\cX^c}{}{f_i(x)}{x}
        &\text{(by \eqref{eq:IntegralsEqual})}
    \\&=
        \integral{\cX}{}{f_i(y) \cdot \left(1 - \frac{g(y)}{f(y)}\right)}{y} + \integral{\cX^c}{}{f_i(y)}{y}
    \\&~=~
        \integral{\cX}{}{f_i(y) }{y} - \integral{\cX}{}{f_i(y) \cdot \left( \frac{g(y)}{f(y)}\right)}{y} + \integral{\cX^c}{}{f_i(y)}{y}
    \\&=
        1 - \integral{\cX}{}{f_i(y) \cdot \frac{g(y)}{f(y)}}{y}
        &\text{($f_i$ is a density)}
    \\&=
        1 - \integral{\cX}{}{g_i(y)}{y}
        & \qedhere
    \end{align*}
\end{proof}
\begin{lemma}
    \label{GSumGi}
    We have $g = \sum_{i\in [k]} w_ig_i$.
\end{lemma}
\begin{proof}
    For $x \in \cX$, since $\sum_{i \in [k]} w_i f_i = f$, we have
    \[
        \sum_{i \in [k]} w_i g_i(x)
        ~=~ \sum_{i \in [k]} w_if_i(x) \frac{g(x)}{f(x)}
        ~=~ g(x).
    \]
    On the other hand, for $x \notin \cX$ we have
    \begin{align*}
        \sum_{i \in [k]} w_i g_i(x)
    &~=~
        \sum_{i \in [k]} w_i \Delta_i(x) + w_i f_i(x) &\text{(definition of $g_i$)}
    \\&~=~
        \sum_{i \in [k]} w_i \left( \frac{ (g(x) - f(x))}{\integral{\cX}{}{\left(f(y) - g(y)\right)}{y}} \cdot \integral{\cX}{}{\left(f_i(y) \cdot \frac{f(y) - g(y)}{f(y)}\right)}{y}\right) 
        & \text{\qquad(by~\eqref{defDelta})}
        \\&\qquad+ \sum_{i \in [k]} w_i f_i(x)
    \\&~=~
        \frac{ (g(x) - f(x))}{\integral{\cX}{}{\left(f(y) - g(y)\right)}{y}} \cdot \integral{\cX}{}{\Bigg(\sum_{i \in [k]} w_i f_i(y) \cdot \frac{f(y) - g(y)}{f(y)}\Bigg)}{y}\\ &\qquad+f(x)&\text{$\left(\textstyle \sum_{i \in [k]} w_i f_i = f\right)$}
    \\&~=~
        \frac{ (g(x) - f(x))}{\integral{\cX}{}{\left(f(y) - g(y)\right)}{y}} \cdot \integral{\cX}{}{\left({f(y) - g(y)}\right)}{y} +f(x)&\text{$\left(\textstyle \sum_{i \in [k]} w_i f_i = f\right)$}
    \\&~=~
        g(x) - f(x) +f(x)
    ~=~ g(x).
    \qedhere
    \end{align*}
\end{proof}

Let $I \coloneqq \{\, i \in[k] \,:\, \|f_i-g_i\|_1 > r \,\}$.
It remains to show that $\sum_{i \in I} w_i < \rho / r$.
Observe from the definition of the $g_i$ that we also have $\cX = \setst{ x }{ g_i(x) < f_i(x) }$ for each $i\in[k]$. Thus, using Lemma~\ref{GSumGi},
\[
\|f-g\|_1
~=~
2 \integral{\cX}{}{(f(x)-g(x))}{x}
~=~
2 
\sum_{i\in[k]}w_i\integral{\cX}{}{(f_i(x)-g_i(x))}{x}
~=~
\sum_{i\in[k]} w_i \|f_i-g_i\|_1.
\]
Thus, from the hypothesis of the lemma,
\[
\rho 
~\geq~
\|f-g\|_1
~=~ \sum_{i \in [k]} w_i \|f_i-g_i\|_1
~\geq~ \sum_{i \in I} w_i \|f_i-g_i\|_1
~>~ \sum_{i \in I} w_i r,
\]
by definition of $I$.
This gives $\sum_{i \in I} w_i<\rho/r$, as required.

{
	\section{An efficient algorithm for learning a single Gaussian}
	\label{secupperboundsingle}
	In this section, we give a simple algorithm for learning a single $d$-dimensional Gaussian in the realizable setting. The algorithm outputs a distribution within $L^1$ distance $\eps$ of the target distribution with probability $1-\delta$ and has sample complexity $O((d^2+d\log(1/\delta))/\eps^2)$ and computational complexity $O((d^4+d^3\log(1/\delta))/\eps^2)$.
	
	The algorithm samples $2m$ points $v_1,\dots,v_{2m}$ from the unknown target distribution $\ncal(\mu,\Psi)$
	and outputs the Gaussian distribution with parameters
	\begin{equation}
	\muhat \coloneqq \sum_{i\in[m]} v_i/m \textnormal{ and }
	\Psihat \coloneqq \sum_{i\in[m]} (v_{2i}-v_{2i-1})(v_{2i}-v_{2i-1})\transpose/2m.
	\label{muhatpsihat}
	\end{equation}
	The running time is clearly $O(d^2 m)$.
	In the rest of this section we prove the following theorem.
	
	\begin{theorem}\label{b1}
		There exists an absolute constant $C$ such that 
		if we take $2m = 2C(d^2+d\log(1/\delta))/\eps^2$ samples 
		$v_1,\dots,v_{2m} \sim \ncal(\mu,\Psi)$
		and define $\muhat$ and $\Psihat$ by
		\eqref{muhatpsihat},
		then, with probability $1-\delta$,
		\[\DTV{\ncal(\muhat,\Psihat)}{\ncal(\mu,\Psi)}\leq \eps/2.\]
	\end{theorem}
	
	We start by proving that $\muhat$ is close to $\mu$. 
	
	\begin{lemma}
		\label{GaussianConcentrationWithSigma}
		If $m\geq (2d + 6 \sqrt{d \log(2/\delta)})/\eps^2$
		then we have $$
		\prob[ (\muhat-\mu) \transpose \Psi^{-1} (\muhat-\mu) \geq \eps^2/2 ] 
		\leq \delta/2.
		$$
	\end{lemma}
	\begin{proof}
		Let $g_i \coloneqq \Psi^{-1/2} (v_i - \mu)$;
		it is easy to check that $g_1,\ldots,g_m$ are independent samples from $\ncal(0,I)$,
		and
		\[
		(\muhat-\mu) \transpose \Psi^{-1} (\muhat-\mu) = \norm{\smallfrac{1}{m} \smallsum{i=1}{m} g_i}_2^2.
		\]
		Note that $\norm{\smallfrac{1}{\sqrt{m}} \smallsum{i=1}{m} g_i}_2^2$ 
		has the chi-squared distribution with parameter $d$.
		Let $\theta = (m\eps^2-2d)/2d$.
		Applying \Lemma{ChiSquaredTail} with $t=\theta^2 d/9$ shows that
		$$ \prob\Big[\: \norm{\smallfrac{1}{m} \smallsum{i=1}{m} g_i}_2^2 \geq (1+\theta)d/m \:\Big] 
		~\leq~ \exp(-\theta^2 d/9).
		$$
		So, we conclude
		\begin{align*}
			\prob\big[ (\muhat-\mu) \transpose \Psi^{-1} (\muhat-\mu) \geq \eps^2/2 \big] 
			&~=~   \prob\big[ \norm{\smallfrac{1}{m} \smallsum{i=1}{m} g_i}_2^2 \geq (1+\theta)d/m \big] \\
			&~\leq~ \exp(-\theta^2 d/9)
			= \exp(-(m\eps^2-2d)^2/36d)
			\\&\leq \exp (-36 d \log(2/\delta)/36d)=\delta/2
			,
		\end{align*}
		as required.
	\end{proof}
	We now show that $\Psihat$ is close to $\Psi$.
	Let $\alpha \coloneqq \eps/\sqrt{2d}$.
	
	\begin{lemma}\label{covarianceclose}
		There is an absolute constant $C$ such that
		if $m\geq C(d^2 + d \log(1/\delta))/\eps^2$, then
		with probability at least $1-\delta/2$ we have
		$\|\Psi^{-1/2} \Psihat \Psi^{-1/2} - I\|_s \leq \eps/\sqrt{2d}.$
	\end{lemma}
	\begin{proof}
		Let $g_1,\dots,g_{2m} \sim \ncal(0,I_d)$ be  independent.
		An alternative way to generate the $v_i$ is letting $v_i \equidist \Psi^{1/2} g_i + \mu$.
		For each $i\in[m]$, let $h_i = g_{2i}-g_{2i-1} / \sqrt{2}$.
		Thus
		\[\Psihat \equidist \sum_{i\in[m]} \Psi^{1/2} (g_{2i}-g_{2i-1})(g_{2i}-g_{2i-1})\transpose \Psi^{1/2}/2m
		\equidist \sum_{i\in[m]} \Psi^{1/2} h_{i}h_i\transpose \Psi^{1/2}/m
		= \Psi^{1/2} \left(\sum_{i\in[m]}  h_{i}h_i\transpose/m\right) \Psi^{1/2},
		\]
		where $h_i \sim \ncal(0,I_d)$ and the $h_i$ are independent.
		Since
		$\Psi^{-1/2}\Psihat\Psi^{-1/2}=\sum_{i\in[m]}  h_{i}h_i\transpose/m $,
		the lemma's conclusion is equivalent to
		$$\prb{\left\| \sum_{i\in[m]}  h_{i}h_i\transpose/m - I\right\|_s \geq \eps/\sqrt{2d}} \leq \delta/2.$$
		This follows from \Lemma{covariance_estimation} below
		by choosing $t=1 + \sqrt{\log(4/\delta)/d}$
		and since 
		$m\geq (Cd^2 + Cd \log(1/\delta))/\eps^2$.
	\end{proof}
	
	\begin{lemma}[Corollary~5.50 in \cite{vershynin_2012}]
		\LemmaName{covariance_estimation}
		There exists an absolute constant $C$ with the following property.
		Let $X_1,\dots,X_m \sim \ncal (0, I_d)$,
		and let $0 < \eps < 1 < t$.
		If $m \geq C (t/\eps)^2 d$, then we have
		\[
		\prb{
			\left\|
			\frac{1}{m} \sum_{i=1}^{m} X_i X_i\transpose - I_d
			\right\|_s
			> \eps } ~<~ 2\exp(-t^2 d).
		\]
	\end{lemma}
	
	\begin{proof}[Proof of Theorem~\ref{b1}]
		We will show that, with probability at least
		$1-\delta$,
		$
		\DLD{\Psihat}{\Psi} +
		(\mu-\muhat) \transpose \Psi^{-1} (\mu-\muhat) \leq \eps^2$
		and the theorem will follow from Lemma~\ref{lem:tvupbound}.
		
		First, since
		$m=C(d^2+d\log(1/\delta))/\eps^2\geq (2d + 6 \sqrt{d \log(2/\delta)})/\eps^2$,
		Lemma~\ref{GaussianConcentrationWithSigma} gives that, with probability at least $1-\delta/2$, we have $
		(\muhat-\mu) \transpose \Psi^{-1} (\muhat-\mu) \leq \eps^2/2$.
		Second,
		Lemma~\ref{covarianceclose} gives that, with probability at least $1-\delta/2$
		we have
		$\|\Psi^{-1/2} \Psihat \Psi^{-1/2} - I\|_s \leq \eps/\sqrt{2d},$
		which, by \Lemma{SpectralToDLD} implies
		$\DLD{\Psihat}{\Psi} \leq \eps^2/2$, completing the proof.
	\end{proof}
}

{
\section{Proof of Remark~\ref{constantcompression}}
\label{sec:constantcompression}
In this appendix we prove the following result.
\begin{theorem}
	\label{thm:1gaussiancompression}
	The class of Gaussian distributions over the real line admits $(4,1,O(1/\eps))$ $0.773$-robust compression.
\end{theorem}

We first prove an auxiliary lemma.
Note that any vector $(p_1,\dots,p_n)\in \Delta_n$ induces a discrete probability distribution over $[n]$ defined by $\Pr (i) \coloneqq p_i$. Let $x\vee y \coloneqq \max\{x,y\}$.

\begin{lemma}\label{lem:ballsinbins}
	Let $(p_1,\dots,p_{2n+1})\in\Delta_{2n+1}$ and
	$(q_1,\dots,q_{2n+1})\in\Delta_{2n+1}$ be discrete probability distributions 
	satisfying $\sum_{i\in[2n+1]} |p_i-q_i| \leq t$.
	Suppose we have $2n+1$ bins, numbered 1 to $2n+1$.
	We throw $m$ balls into these bins, where each ball independently chooses a bin according to $q_i$.
	We pair bin 1 with bin 2,
	bin 3 with bin 4,
	\ldots, and bin $2n-1$ with bin $2n$;
	bin $2n+1$ is unpaired.
	The probability that, for each pair of bins, at most one them gets a ball, is not more than
	\[
	2^{n} \left(t/2 + p_{2n+1}+\sum_{i=1}^{n} \max \{ p_{2i-1},p_{2i} \}\right)^m.
	\]
\end{lemma}
\begin{proof}
	Let $P_1\coloneqq\{1,2\}$, $P_2\coloneqq\{3,4\}$, ..., $P_n\coloneqq\{2n-1,2n\}$, and let 
	\(\mathcal{A} \coloneqq \{ 
	A\subseteq [2n] : |A \cap P_i| = 1 \ \forall i\in[n]
	\}.\)
	Clearly $|\mathcal A| = 2^n$. 
	For any $A\in \mathcal A$, let $E_A$ be the event that, the first ball does not choose a bin in $A$,
	and let $F_A$ be the event that, none of the balls choose a bin in $A$.
	Then,
	\begin{align*}
	\Pr[E_A]  = \sum_{i\in [2n+1]\setminus A} q_i 
	\leq \DTV{p}{q} + \sum_{i\in [2n+1]\setminus A} p_i
	\leq t/2 + \sum_{i\notin A} p_i
	\leq t/2 + p_{2n+1} + \sum_{i=1}^{n} \max \{ p_{2i-1},p_{2i} \},
	\end{align*}
	and so
	\(
	\Pr[F_A] = \Pr[E_A]^m
	\leq (t/2 + p_{2n+1} + \sum_{i=1}^{n} ( p_{2i-1} \vee p_{2i} ))^m.
	\)
	Finally, observe that, if for each pair of bins, at most one them gets a ball, then there exists at least one $A\in \mathcal A$ such that 
	none of the balls chooses a bin in $A$.
	The lemma is thus proved by applying the union bound over all events $\{F_A\}_{A\in \mathcal A}$.
\end{proof}

\begin{proof}[Proof of Theorem~\ref{thm:1gaussiancompression}]
	Let $q$ be a distribution such that there exists a Gaussian $g = \mathcal{N}(\mu, \sigma^2)$ with $\|q-g\|_{1} \leq r \leq 0.773$. 
	Our goal is to encode $g$ using  samples generated from $q$.
	Let $m=C/\eps$ for a large constant $C$ to be determined, and let $S\sim q^m$ be an i.i.d.\ sample. The goal is to approximately encode $\mu$ and $\sigma$ using only four elements of $S$ and a single bit. 
	
	Our proposed decoder, $\mathcal{J}$, takes as input four points $x_1, x_2, y_1,y_2 \in \R$ and one bit $b\in \{0,1\}$ and outputs a Gaussian distribution based on the following rule:
	\[   
	\mathcal{J}(x_1,x_2,y_1,y_2,b) = \left\{
	\begin{array}{@{}l@{\thinspace}l}
	\mathcal{N}\left(\frac{x_1 + x_2}{2}, \frac{|y_1 - y_2|^2}{9}\right)  &: \text{ if } b = 1,\\ \\
	\mathcal{N}\left(\frac{x_1 + x_2}{2}, |y_1 - y_2|^2\right) &: \text{ if } b = 0.
	\end{array}
	\right.
	\]
	Our goal is thus to show that, with probability at least $2/3$, there exist ${x}_1,{x}_2,{y}_1,{y}_2 \in S$ and $b\in\{0,1\}$ such that
	$\|\mathcal{J}({x}_1,{x}_2,y_1,y_2,b) - g\|_1 \leq \eps$. 
	
	Let $M\coloneqq 1/\eps$ and partition the interval $[\mu-2\sigma,\mu+2\sigma)$ into $4M$ subintervals of length $\eps\sigma$. 
	Enumerate these intervals as $I_1$ to $I_{4M}$, i.e., $I_i = [\mu-2\sigma + (i-1)(\eps\sigma), \mu-2\sigma + i(\eps\sigma))$.
	Also let $I_{4M+1} = \R\setminus \bigcup_{i=1}^{4M} I_i$.
	We state two claims which will imply the theorem.
	
	\emph{Claim 1}. With probability at least $5/6$, there exist $y_1,y_2\in S$ such that at least one of the following two conditions holds: (a) $y_1\in I_i$ and $y_2\in I_{i + M}$ for some  $i\in \{M+1,2M+2,...,2M\}$. 
	In this case, we let $b=0$, and so 
	$\mathcal{J}({x}_1,{x}_2,y_1,y_2,b)$ will have standard deviation $ |y_1 - y_2|  $. \\(b) $y_1\in I_i$ and $y_2\in I_{i + 3M}$ for some  $i\in [M]$. 
	In this case, we let $b=1$, and so 
	$\mathcal{J}({x}_1,{x}_2,y_1,y_2,b)$ will have standard deviation 
	$\frac{|y_1 - y_2|}{3}$.\\
	Also, if both of (a) and (b) happen, we will go with the first rule. Note that if Claim 1 holds, then the standard deviation of $\mathcal{J}({x}_1,{x}_2,y_1,y_2,b)$, written $\hat{\sigma}$,
	satisfies $|\hat{\sigma}-\sigma|\leq \eps \sigma$. 
	
	\emph{Claim 2}. With probability at least $5/6$, there exist $x_1,x_2\in S$ such that $x_1\in I_i$ and $x_2\in I_{4M - i + 1}$ for some  $i\in [2M]$. If so, 
	$\mathcal{J}({x}_1,{x}_2,y_1,y_2,b)$ will have mean $\frac{x_1+x_2}{2} \eqqcolon \hat{\mu}$.
	
	Also note that if Claim 2 holds, then $|\hat{\mu} - \mu| \leq  \eps \sigma$. 
	Therefore, if both claims hold, then Lemma \ref{gaussianTV1d} gives that $\mathcal{J}({x}_1,{x}_2,y_1,y_2,b)=\mathcal{N}(\hat{\mu}, \hat{\sigma}^2)$ is a $2\eps$-approximation for $\mathcal{N}(\mu, \sigma^2)=g$.
	In other words, $g$ can be approximately reconstructed, up to error $2\eps$, using the four data points $x_1,x_2,y_1,y_2$ from a sample $S$ of size $O(1/\eps)$ and a single bit $b$. (The definition of robust compression requires error at most $\eps$. For getting this, one just needs to refine the partition by a constant factor, which multiplies $M$ by a constant factor, and this will only multiply $m$ by a constant factor.)
	Note also that the probability of existence of such four points is at least $1-(1-5/6)-(1-5/6) \geq 2/3$. Therefore, it remains to prove Claim 1 and Claim 2.

    We start with Claim 1.
	View the sets $I_1,\dots,I_{4M},I_{4M+1}$ as bins,
	and consider the i.i.d.\ samples as balls landing in these bins according to $q$. 
	Let $p_i \coloneqq \int_{I_i}g(x) \mathrm{d}x $ and $q_i \coloneqq \int_{I_i} q(x) \mathrm{d}x$
	for $i\in[4M+1]$.
	Note that, by the triangle inequality,
	the $\ell_1$ distance between
	$(p_1,\dots,p_{4M+1})$ and $(q_1,\dots,q_{4M+1})$ is not more than the $L^1$ distance between $g$ and $q$, which is at most $r$.
	
	We pair the bins as follows:
	$I_i$ is paired with $I_{i+M}$ for $i\in\{M+1,\dots,2M\}$,
	and 
	$I_i$ is paired with $I_{i+3M}$ for $i\in[M]$.
	Therefore, by Lemma~\ref{lem:ballsinbins},
	the probability that
	Claim 1 does not hold can be bounded by
	\begin{equation}
	2^{2M} \!\left(\sum_{i=M+1}^{2M}  \!\!(p_i \!\vee\!  p_{i+M}) \!+\! \sum_{i=1}^{M}  (p_i \!\vee\! p_{i+3M}) \!+\! p_{4M+1}\!+\!\frac r 2\! \right)^m  = 
	2^{2M} \left(\sum_{i=\frac 3 2 M+1}^{\frac 5 2 M} \!p_i \!+\! \sum_{i=\frac M 2+1}^{M} \!p_i\!
	+\! \sum_{3M+1}^{\frac {7}{2}M} p_{i}\! +\! p_{4M+1}+\! \frac  r 2\! \right) ^ m,\label{prob}
	\end{equation}
	where  we have used the fact that the $p_i$ are coming from a Gaussian, and thus
	$p_1  \leq \dots \leq p_{2M} = p_{2M+1}  \geq \dots \geq p_{4M}
	$
	(we have also assumed, for simplicity, that $M$ is even).  Let $X\sim \mathcal N(\mu, \sigma^2)$ and $\Phi(A) \coloneqq \Pr[ N(0,1)\in A]$.  Then using known numerical bounds for $\Phi$, we obtain
	\begin{align*}
	&\sum_{i=1.5 M+1}^{2.5 M} p_i  + \sum_{i=M/2+1}^{M} p_i
	+ \sum_{3M+1}^{3.5M} p_{i} + p_{4M+1}+r/2 \\
	&= \Pr[X\in [\mu - \sigma/2, \mu + \sigma /2]] + 
	2\Pr[X\in [\mu - 3\sigma/2, \mu - \sigma]]  + \Pr[X\notin [\mu - 2\sigma, \mu + 2\sigma]] + r/2   \\& = \Phi ([-0.5, 0.5]) \!+\! 2\Phi ([-1.5, -1]) \!+\! 
	2 \Phi ((-\infty, -2]) \!+\! \frac r 2 < 0.383 + 0.184 + 0.046 + \frac r 2  \leq 0.9995.
	\end{align*}
	Therefore, since $M=1/\eps$,
	by choosing $m=C/\eps$ for a large enough $C$, we can make  \eqref{prob} arbitrarily small, completing the proof of Claim 1. 
	
	Via a similar argument, the probability that Claim 2 does not hold can be bounded by
	\begin{align*}
	& 2^{2M} \left ( \sum_{i=1}^{2M} \max \{ p_i, p_{4M- i + 1}\} + p_{4M+1} +r/2 \right) ^ m
	=
	2^{2M} \left ( \sum_{i=1}^{2M} p_i + p_{4M+1} +r/2\right) ^ m \\
	& =
	2^{2M} \left ( \Phi([-\infty,0]) + \Phi([2,\infty)) +r/2 \right) ^ m  <
	2^{2M} \left ( 0.5 + 0.023 +r/2 \right) ^ m
	< 2^{2M} \left ( 0.91 \right) ^ m < 1/6,
	\end{align*}
	for $m = C/\eps$ with a large enough $C$.
\end{proof}

\begin{remark}
	By using more bits and adding more scales, one can show that 1-dimensional Gaussians admit $(4,b(r),O(1/\eps))$ $r$-robust compression for any fixed $r<1$ (the number of required bits and the implicit constant in the $O$ will depend on the value of $r$).
\end{remark}
}


\begin{thebibliography}{50}


\ifx \showCODEN    \undefined \def \showCODEN     #1{\unskip}     \fi
\ifx \showDOI      \undefined \def \showDOI       #1{#1}\fi
\ifx \showISBNx    \undefined \def \showISBNx     #1{\unskip}     \fi
\ifx \showISBNxiii \undefined \def \showISBNxiii  #1{\unskip}     \fi
\ifx \showISSN     \undefined \def \showISSN      #1{\unskip}     \fi
\ifx \showLCCN     \undefined \def \showLCCN      #1{\unskip}     \fi
\ifx \shownote     \undefined \def \shownote      #1{#1}          \fi
\ifx \showarticletitle \undefined \def \showarticletitle #1{#1}   \fi
\ifx \showURL      \undefined \def \showURL       {\relax}        \fi
\providecommand\bibfield[2]{#2}
\providecommand\bibinfo[2]{#2}
\providecommand\natexlab[1]{#1}
\providecommand\showeprint[2][]{arXiv:#2}

\bibitem[\protect\citeauthoryear{Acharya, Diakonikolas, Hegde, Li, and
  Schmidt}{Acharya et~al\mbox{.}}{2015}]%
        {ADHLS}
\bibfield{author}{\bibinfo{person}{Jayadev Acharya}, \bibinfo{person}{Ilias
  Diakonikolas}, \bibinfo{person}{Chinmay Hegde}, \bibinfo{person}{Jerry~Zheng
  Li}, {and} \bibinfo{person}{Ludwig Schmidt}.}
  \bibinfo{year}{2015}\natexlab{}.
\newblock \showarticletitle{Fast and Near-Optimal Algorithms for Approximating
  Distributions by Histograms}. In \bibinfo{booktitle}{\emph{Proceedings of the
  34th ACM SIGMOD-SIGACT-SIGAI  Symposium on Principles of Database Systems}}
  (Melbourne, Victoria, Australia) \emph{(\bibinfo{series}{PODS ’15})}.
  \bibinfo{publisher}{Association for Computing Machinery},
  \bibinfo{address}{New York, NY, USA}, \bibinfo{pages}{249–263}.
\newblock
\showISBNx{9781450327572}
\urldef\tempurl%
\url{https://doi.org/10.1145/2745754.2745772}
\showDOI{\tempurl}


\bibitem[\protect\citeauthoryear{Anthony and Bartlett}{Anthony and
  Bartlett}{1999}]%
        {AB99}
\bibfield{author}{\bibinfo{person}{Martin Anthony} {and}
  \bibinfo{person}{Peter~L. Bartlett}.} \bibinfo{year}{1999}\natexlab{}.
\newblock \bibinfo{booktitle}{\emph{Neural network learning: theoretical
  foundations}}.
\newblock \bibinfo{publisher}{Cambridge University Press, Cambridge}. xiv+389
  pages.
\newblock
\showISBNx{0-521-57353-X}
\urldef\tempurl%
\url{https://doi.org/10.1017/CBO9780511624216}
\showDOI{\tempurl}


\bibitem[\protect\citeauthoryear{Arora and Kannan}{Arora and Kannan}{2005}]%
        {arora}
\bibfield{author}{\bibinfo{person}{Sanjeev Arora} {and} \bibinfo{person}{Ravi
  Kannan}.} \bibinfo{year}{2005}\natexlab{}.
\newblock \showarticletitle{Learning mixtures of separated nonspherical
  Gaussians}.
\newblock \bibinfo{journal}{\emph{Ann. Appl. Probab.}} \bibinfo{volume}{15},
  \bibinfo{number}{1A} (\bibinfo{date}{02} \bibinfo{year}{2005}),
  \bibinfo{pages}{69--92}.
\newblock
\urldef\tempurl%
\url{https://doi.org/10.1214/105051604000000512}
\showDOI{\tempurl}


\bibitem[\protect\citeauthoryear{Ashtiani, Ben-David, Harvey, Liaw, Mehrabian,
  and Plan}{Ashtiani et~al\mbox{.}}{2018b}]%
        {gaussian_mixture_conference_version}
\bibfield{author}{\bibinfo{person}{Hassan Ashtiani}, \bibinfo{person}{Shai
  Ben-David}, \bibinfo{person}{Nicholas~J.A. Harvey},
  \bibinfo{person}{Christopher Liaw}, \bibinfo{person}{Abbas Mehrabian}, {and}
  \bibinfo{person}{Yaniv Plan}.} \bibinfo{year}{2018}\natexlab{b}.
\newblock \showarticletitle{Nearly tight sample complexity bounds for learning
  mixtures of {G}aussians via sample compression schemes}.
\newblock In \bibinfo{booktitle}{\emph{Advances in Neural Information
  Processing Systems 31}}, \bibfield{editor}{\bibinfo{person}{S.~Bengio},
  \bibinfo{person}{H.~Wallach}, \bibinfo{person}{H.~Larochelle},
  \bibinfo{person}{K.~Grauman}, \bibinfo{person}{N.~Cesa-Bianchi}, {and}
  \bibinfo{person}{R.~Garnett}} (Eds.). \bibinfo{publisher}{Curran Associates,
  Inc.}, \bibinfo{pages}{3412--3421}.
\newblock
\urldef\tempurl%
\url{https://papers.nips.cc/paper/7601-nearly-tight-sample-complexity-bounds-for-learning-mixtures-of-gaussians-via-sample-compression-schemes}
\showURL{%
\tempurl}


\bibitem[\protect\citeauthoryear{Ashtiani, Ben-David, Harvey, Liaw, Mehrabian,
  and Plan}{Ashtiani et~al\mbox{.}}{2020}]%
        {gaussian_mixture_tr}
\bibfield{author}{\bibinfo{person}{Hassan Ashtiani}, \bibinfo{person}{Shai
  Ben-David}, \bibinfo{person}{Nicholas~J.A. Harvey},
  \bibinfo{person}{Christopher Liaw}, \bibinfo{person}{Abbas Mehrabian}, {and}
  \bibinfo{person}{Yaniv Plan}.} \bibinfo{year}{2020}\natexlab{}.
\newblock \bibinfo{title}{Nearly tight sample complexity bounds for learning
  mixtures of {G}aussians via sample compression schemes}.
\newblock
\newblock
\showeprint[arxiv]{1710.05209}~[cs.LG]
\urldef\tempurl%
\url{https://arxiv.org/abs/1710.05209}
\showURL{%
\tempurl}


\bibitem[\protect\citeauthoryear{Ashtiani, Ben-David, and Mehrabian}{Ashtiani
  et~al\mbox{.}}{2018a}]%
        {ashtiani2017sample}
\bibfield{author}{\bibinfo{person}{Hassan Ashtiani}, \bibinfo{person}{Shai
  Ben-David}, {and} \bibinfo{person}{Abbas Mehrabian}.}
  \bibinfo{year}{2018}\natexlab{a}.
\newblock \showarticletitle{Sample-Efficient Learning of Mixtures}. In
  \bibinfo{booktitle}{\emph{Proceedings of the Thirty-Second AAAI Conference on
  Artificial Intelligence}} \emph{(\bibinfo{series}{AAAI'18})}.
  \bibinfo{publisher}{AAAI Publications}, \bibinfo{pages}{2679--2686}.
\newblock
\urldef\tempurl%
\url{https://arxiv.org/abs/1706.01596}
\showURL{%
\tempurl}


\bibitem[\protect\citeauthoryear{{Barsov} and {Ul'yanov}}{{Barsov} and
  {Ul'yanov}}{1987}]%
        {ulyanov}
\bibfield{author}{\bibinfo{person}{S.~S. {Barsov}} {and} \bibinfo{person}{V.~V.
  {Ul'yanov}}.} \bibinfo{year}{1987}\natexlab{}.
\newblock \showarticletitle{{Estimates of the proximity of Gaussian measures.}}
\newblock \bibinfo{journal}{\emph{{Sov. Math., Dokl.}}}  \bibinfo{volume}{34}
  (\bibinfo{year}{1987}), \bibinfo{pages}{462--466}.
\newblock
\showISSN{0197-6788}


\bibitem[\protect\citeauthoryear{Belkin and Sinha}{Belkin and Sinha}{2010}]%
        {Belkin}
\bibfield{author}{\bibinfo{person}{Mikhail Belkin} {and}
  \bibinfo{person}{Kaushik Sinha}.} \bibinfo{year}{2010}\natexlab{}.
\newblock \showarticletitle{Polynomial Learning of Distribution Families}. In
  \bibinfo{booktitle}{\emph{Proceedings of the 2010 IEEE 51st Annual Symposium
  on Foundations of Computer Science}} \emph{(\bibinfo{series}{FOCS '10})}.
  \bibinfo{publisher}{IEEE Computer Society}, \bibinfo{address}{Washington, DC,
  USA}, \bibinfo{pages}{103--112}.
\newblock
\showISBNx{978-0-7695-4244-7}
\urldef\tempurl%
\url{https://doi.org/10.1109/FOCS.2010.16}
\showDOI{\tempurl}


\bibitem[\protect\citeauthoryear{Blumer, Ehrenfeucht, Haussler, and
  Warmuth}{Blumer et~al\mbox{.}}{1989}]%
        {Blumer:1989}
\bibfield{author}{\bibinfo{person}{Anselm Blumer}, \bibinfo{person}{Andrzej
  Ehrenfeucht}, \bibinfo{person}{David Haussler}, {and}
  \bibinfo{person}{Manfred~K. Warmuth}.} \bibinfo{year}{1989}\natexlab{}.
\newblock \showarticletitle{Learnability and the {V}apnik-{C}hervonenkis
  Dimension}.
\newblock \bibinfo{journal}{\emph{J. ACM}} \bibinfo{volume}{36},
  \bibinfo{number}{4} (\bibinfo{date}{Oct.} \bibinfo{year}{1989}),
  \bibinfo{pages}{929--965}.
\newblock
\showISSN{0004-5411}
\urldef\tempurl%
\url{https://doi.org/10.1145/76359.76371}
\showDOI{\tempurl}


\bibitem[\protect\citeauthoryear{Boyd and Vandenberghe}{Boyd and
  Vandenberghe}{2004}]%
        {boyd}
\bibfield{author}{\bibinfo{person}{Stephen Boyd} {and} \bibinfo{person}{Lieven
  Vandenberghe}.} \bibinfo{year}{2004}\natexlab{}.
\newblock \bibinfo{booktitle}{\emph{Convex optimization}}.
\newblock \bibinfo{publisher}{Cambridge University Press, Cambridge}. xiv+716
  pages.
\newblock
\showISBNx{0-521-83378-7}
\urldef\tempurl%
\url{https://doi.org/10.1017/CBO9780511804441}
\showDOI{\tempurl}


\bibitem[\protect\citeauthoryear{Chan, Diakonikolas, Servedio, and Sun}{Chan
  et~al\mbox{.}}{2014}]%
        {onedimensional}
\bibfield{author}{\bibinfo{person}{Siu-On Chan}, \bibinfo{person}{Ilias
  Diakonikolas}, \bibinfo{person}{Rocco~A. Servedio}, {and}
  \bibinfo{person}{Xiaorui Sun}.} \bibinfo{year}{2014}\natexlab{}.
\newblock \showarticletitle{Efficient Density Estimation via Piecewise
  Polynomial Approximation}. In \bibinfo{booktitle}{\emph{Proceedings of the
  Forty-sixth Annual ACM Symposium on Theory of Computing}} (New York, New
  York) \emph{(\bibinfo{series}{STOC '14})}. \bibinfo{publisher}{ACM},
  \bibinfo{address}{New York, NY, USA}, \bibinfo{pages}{604--613}.
\newblock
\showISBNx{978-1-4503-2710-7}
\urldef\tempurl%
\url{https://doi.org/10.1145/2591796.2591848}
\showDOI{\tempurl}


\bibitem[\protect\citeauthoryear{Cover and Thomas}{Cover and Thomas}{2006}]%
        {cover}
\bibfield{author}{\bibinfo{person}{Thomas~M. Cover} {and}
  \bibinfo{person}{Joy~A. Thomas}.} \bibinfo{year}{2006}\natexlab{}.
\newblock \bibinfo{booktitle}{\emph{Elements of information theory}
  (\bibinfo{edition}{second} ed.)}.
\newblock \bibinfo{publisher}{Wiley-Interscience [John Wiley \& Sons], Hoboken,
  NJ}. xxiv+748 pages.
\newblock
\showISBNx{978-0-471-24195-9; 0-471-24195-4}


\bibitem[\protect\citeauthoryear{Dasgupta}{Dasgupta}{1999}]%
        {dasgupta1999learning}
\bibfield{author}{\bibinfo{person}{Sanjoy Dasgupta}.}
  \bibinfo{year}{1999}\natexlab{}.
\newblock \showarticletitle{Learning mixtures of {G}aussians}.
\newblock In \bibinfo{booktitle}{\emph{40th {A}nnual {S}ymposium on
  {F}oundations of {C}omputer {S}cience ({N}ew {Y}ork, 1999)}}.
  \bibinfo{publisher}{IEEE Computer Soc., Los Alamitos, CA},
  \bibinfo{pages}{634--644}.
\newblock
\urldef\tempurl%
\url{https://doi.org/10.1109/SFFCS.1999.814639}
\showDOI{\tempurl}


\bibitem[\protect\citeauthoryear{Daskalakis and Kamath}{Daskalakis and
  Kamath}{2014}]%
        {DK14}
\bibfield{author}{\bibinfo{person}{Constantinos Daskalakis} {and}
  \bibinfo{person}{Gautam Kamath}.} \bibinfo{year}{2014}\natexlab{}.
\newblock \showarticletitle{Faster and Sample Near-Optimal Algorithms for
  Proper Learning Mixtures of Gaussians}. In
  \bibinfo{booktitle}{\emph{Proceedings of The 27th Conference on Learning
  Theory}} \emph{(\bibinfo{series}{Proceedings of Machine Learning Research},
  Vol.~\bibinfo{volume}{35})}, \bibfield{editor}{\bibinfo{person}{Maria~Florina
  Balcan}, \bibinfo{person}{Vitaly Feldman}, {and} \bibinfo{person}{Csaba
  Szepesvári}} (Eds.). \bibinfo{publisher}{PMLR}, \bibinfo{address}{Barcelona,
  Spain}, \bibinfo{pages}{1183--1213}.
\newblock
\urldef\tempurl%
\url{http://proceedings.mlr.press/v35/daskalakis14.html}
\showURL{%
\tempurl}


\bibitem[\protect\citeauthoryear{Davidson and Szarek}{Davidson and
  Szarek}{2001}]%
        {singularvalueconcentration}
\bibfield{author}{\bibinfo{person}{Kenneth~R. Davidson} {and}
  \bibinfo{person}{Stanislaw~J. Szarek}.} \bibinfo{year}{2001}\natexlab{}.
\newblock \showarticletitle{Local operator theory, random matrices and {B}anach
  spaces}.
\newblock In \bibinfo{booktitle}{\emph{Handbook of the geometry of {B}anach
  spaces, {V}ol. {I}}}. \bibinfo{publisher}{North-Holland, Amsterdam},
  \bibinfo{pages}{317--366}.
\newblock
\urldef\tempurl%
\url{https://doi.org/10.1016/S1874-5849(01)80010-3}
\showDOI{\tempurl}


\bibitem[\protect\citeauthoryear{Devroye}{Devroye}{1987}]%
        {devroye_density_estimation_first}
\bibfield{author}{\bibinfo{person}{Luc Devroye}.}
  \bibinfo{year}{1987}\natexlab{}.
\newblock \bibinfo{booktitle}{\emph{A course in density estimation}}.
  \bibinfo{series}{Progress in Probability and Statistics},
  Vol.~\bibinfo{volume}{14}.
\newblock \bibinfo{publisher}{Birkh\"auser Boston, Inc., Boston, MA}. xx+183
  pages.
\newblock
\showISBNx{0-8176-3365-0}


\bibitem[\protect\citeauthoryear{Devroye and Lugosi}{Devroye and
  Lugosi}{2001}]%
        {devroye_book}
\bibfield{author}{\bibinfo{person}{Luc Devroye} {and} \bibinfo{person}{G\'abor
  Lugosi}.} \bibinfo{year}{2001}\natexlab{}.
\newblock \bibinfo{booktitle}{\emph{Combinatorial methods in density
  estimation}}.
\newblock \bibinfo{publisher}{Springer-Verlag, New York}. xii+208 pages.
\newblock
\showISBNx{0-387-95117-2}
\urldef\tempurl%
\url{https://doi.org/10.1007/978-1-4613-0125-7}
\showDOI{\tempurl}


\bibitem[\protect\citeauthoryear{Devroye, Mehrabian, and Reddad}{Devroye
  et~al\mbox{.}}{2018}]%
        {tv_distance_gaussians}
\bibfield{author}{\bibinfo{person}{Luc Devroye}, \bibinfo{person}{Abbas
  Mehrabian}, {and} \bibinfo{person}{Tommy Reddad}.}
  \bibinfo{year}{2018}\natexlab{}.
\newblock \bibinfo{title}{The total variation distance between high-dimensional
  {G}aussians}.
\newblock
\newblock
\showeprint[arxiv]{1810.08693}~[math.ST]
\urldef\tempurl%
\url{https://arxiv.org/abs/1810.08693}
\showURL{%
\tempurl}


\bibitem[\protect\citeauthoryear{Devroye, Mehrabian, and Reddad}{Devroye
  et~al\mbox{.}}{2020}]%
        {lower_bound_improved}
\bibfield{author}{\bibinfo{person}{Luc Devroye}, \bibinfo{person}{Abbas
  Mehrabian}, {and} \bibinfo{person}{Tommy Reddad}.}
  \bibinfo{year}{2020}\natexlab{}.
\newblock \showarticletitle{The minimax learning rates of normal and {I}sing
  undirected graphical models}.
\newblock \bibinfo{journal}{\emph{Electron. J. Statist.}} \bibinfo{volume}{14},
  \bibinfo{number}{1} (\bibinfo{year}{2020}), \bibinfo{pages}{2338--2361}.
\newblock
\showISSN{1935-7524}
\urldef\tempurl%
\url{https://doi.org/10.1214/20-EJS1721}
\showDOI{\tempurl}


\bibitem[\protect\citeauthoryear{Diakonikolas}{Diakonikolas}{2016}]%
        {Diakonikolas2016}
\bibfield{author}{\bibinfo{person}{Ilias Diakonikolas}.}
  \bibinfo{year}{2016}\natexlab{}.
\newblock \showarticletitle{Learning structured distributions}.
\newblock In \bibinfo{booktitle}{\emph{Handbook of big data}}.
  \bibinfo{publisher}{CRC Press, Boca Raton, FL}, Chapter~15,
  \bibinfo{pages}{267--283}.
\newblock
\newblock
\shownote{Available at
  \url{http://www.iliasdiakonikolas.org/distribution-learning-survey.pdf}.}


\bibitem[\protect\citeauthoryear{Diakonikolas, Grigorescu, Li, Natarajan, Onak,
  and Schmidt}{Diakonikolas et~al\mbox{.}}{2017a}]%
        {DGLNOS}
\bibfield{author}{\bibinfo{person}{Ilias Diakonikolas}, \bibinfo{person}{Elena
  Grigorescu}, \bibinfo{person}{Jerry Li}, \bibinfo{person}{Abhiram Natarajan},
  \bibinfo{person}{Krzysztof Onak}, {and} \bibinfo{person}{Ludwig Schmidt}.}
  \bibinfo{year}{2017}\natexlab{a}.
\newblock \showarticletitle{Communication-Efficient Distributed Learning of
  Discrete Distributions}.
\newblock In \bibinfo{booktitle}{\emph{Advances in Neural Information
  Processing Systems 30}}, \bibfield{editor}{\bibinfo{person}{I.~Guyon},
  \bibinfo{person}{U.~V. Luxburg}, \bibinfo{person}{S.~Bengio},
  \bibinfo{person}{H.~Wallach}, \bibinfo{person}{R.~Fergus},
  \bibinfo{person}{S.~Vishwanathan}, {and} \bibinfo{person}{R.~Garnett}}
  (Eds.). \bibinfo{publisher}{Curran Associates, Inc.},
  \bibinfo{pages}{6391--6401}.
\newblock
\urldef\tempurl%
\url{http://papers.nips.cc/paper/7218-communication-efficient-distributed-learning-of-discrete-distributions}
\showURL{%
\tempurl}


\bibitem[\protect\citeauthoryear{Diakonikolas, Kane, and Stewart}{Diakonikolas
  et~al\mbox{.}}{2017b}]%
        {logconcave}
\bibfield{author}{\bibinfo{person}{Ilias Diakonikolas},
  \bibinfo{person}{Daniel~M. Kane}, {and} \bibinfo{person}{Alistair Stewart}.}
  \bibinfo{year}{2017}\natexlab{b}.
\newblock \showarticletitle{Learning Multivariate Log-concave Distributions}.
  In \bibinfo{booktitle}{\emph{Proceedings of the 2017 Conference on Learning
  Theory}} \emph{(\bibinfo{series}{Proceedings of Machine Learning Research},
  Vol.~\bibinfo{volume}{65})}, \bibfield{editor}{\bibinfo{person}{Satyen Kale}
  {and} \bibinfo{person}{Ohad Shamir}} (Eds.). \bibinfo{publisher}{PMLR},
  \bibinfo{address}{Amsterdam, Netherlands}, \bibinfo{pages}{711--727}.
\newblock
\urldef\tempurl%
\url{http://proceedings.mlr.press/v65/diakonikolas17a.html}
\showURL{%
\tempurl}


\bibitem[\protect\citeauthoryear{Diakonikolas, Kane, and Stewart}{Diakonikolas
  et~al\mbox{.}}{2017c}]%
        {gaussian_mixture}
\bibfield{author}{\bibinfo{person}{Ilias Diakonikolas},
  \bibinfo{person}{Daniel~M. Kane}, {and} \bibinfo{person}{Alistair Stewart}.}
  \bibinfo{year}{2017}\natexlab{c}.
\newblock \showarticletitle{Statistical query lower bounds for robust
  estimation of high-dimensional {G}aussians and {G}aussian mixtures (extended
  abstract)}.
\newblock In \bibinfo{booktitle}{\emph{58th {A}nnual {IEEE} {S}ymposium on
  {F}oundations of {C}omputer {S}cience---{FOCS} 2017}}.
  \bibinfo{publisher}{IEEE Computer Soc., Los Alamitos, CA},
  \bibinfo{pages}{73--84}.
\newblock
\urldef\tempurl%
\url{https://doi.org/10.1109/FOCS.2017.16}
\showDOI{\tempurl}
\newblock
\shownote{Available at \url{https://arxiv.org/abs/1611.03473}.}


\bibitem[\protect\citeauthoryear{Feldman, Servedio, and O'Donnell}{Feldman
  et~al\mbox{.}}{2006}]%
        {axis_aligned}
\bibfield{author}{\bibinfo{person}{Jon Feldman}, \bibinfo{person}{Rocco~A.
  Servedio}, {and} \bibinfo{person}{Ryan O'Donnell}.}
  \bibinfo{year}{2006}\natexlab{}.
\newblock \showarticletitle{{P}{A}{C} Learning Axis-aligned Mixtures of
  {G}aussians with No Separation Assumption}. In
  \bibinfo{booktitle}{\emph{Proceedings of the 19th Annual Conference on
  Learning Theory}} (Pittsburgh, PA) \emph{(\bibinfo{series}{COLT'06})}.
  \bibinfo{publisher}{Springer-Verlag}, \bibinfo{address}{Berlin, Heidelberg},
  \bibinfo{pages}{20--34}.
\newblock
\showISBNx{3-540-35294-5, 978-3-540-35294-5}
\urldef\tempurl%
\url{https://doi.org/10.1007/11776420_5}
\showDOI{\tempurl}


\bibitem[\protect\citeauthoryear{Hogben}{Hogben}{2014}]%
        {linearalgebrahandbook}
\bibfield{editor}{\bibinfo{person}{Leslie Hogben}} (Ed.).
  \bibinfo{year}{2014}\natexlab{}.
\newblock \bibinfo{booktitle}{\emph{Handbook of linear algebra}
  (\bibinfo{edition}{second} ed.)}.
\newblock \bibinfo{publisher}{CRC Press, Boca Raton, FL}. xxx+1874 pages.
\newblock
\showISBNx{978-1-4665-0728-9}


\bibitem[\protect\citeauthoryear{Ibragimov}{Ibragimov}{2001}]%
        {ibragimov}
\bibfield{author}{\bibinfo{person}{Ildar Ibragimov}.}
  \bibinfo{year}{2001}\natexlab{}.
\newblock \showarticletitle{Estimation of analytic functions}.
\newblock In \bibinfo{booktitle}{\emph{State of the art in probability and
  statistics ({L}eiden, 1999)}}. \bibinfo{series}{IMS Lecture Notes Monogr.
  Ser.}, Vol.~\bibinfo{volume}{36}. \bibinfo{publisher}{Inst. Math. Statist.,
  Beachwood, OH}, \bibinfo{pages}{359--383}.
\newblock
\urldef\tempurl%
\url{https://doi.org/10.1214/lnms/1215090078}
\showDOI{\tempurl}


\bibitem[\protect\citeauthoryear{Kalai, Moitra, and Valiant}{Kalai
  et~al\mbox{.}}{2012}]%
        {KMV}
\bibfield{author}{\bibinfo{person}{Adam~Tauman Kalai}, \bibinfo{person}{Ankur
  Moitra}, {and} \bibinfo{person}{Gregory Valiant}.}
  \bibinfo{year}{2012}\natexlab{}.
\newblock \showarticletitle{Disentangling Gaussians}.
\newblock \bibinfo{journal}{\emph{Commun. ACM}} \bibinfo{volume}{55},
  \bibinfo{number}{2} (\bibinfo{date}{Feb.} \bibinfo{year}{2012}),
  \bibinfo{pages}{113–120}.
\newblock
\showISSN{0001-0782}
\urldef\tempurl%
\url{https://doi.org/10.1145/2076450.2076474}
\showDOI{\tempurl}


\bibitem[\protect\citeauthoryear{Kearns, Mansour, Ron, Rubinfeld, Schapire, and
  Sellie}{Kearns et~al\mbox{.}}{1994}]%
        {Kearns}
\bibfield{author}{\bibinfo{person}{Michael Kearns}, \bibinfo{person}{Yishay
  Mansour}, \bibinfo{person}{Dana Ron}, \bibinfo{person}{Ronitt Rubinfeld},
  \bibinfo{person}{Robert~E. Schapire}, {and} \bibinfo{person}{Linda Sellie}.}
  \bibinfo{year}{1994}\natexlab{}.
\newblock \showarticletitle{On the Learnability of Discrete Distributions}. In
  \bibinfo{booktitle}{\emph{Proceedings of the Twenty-sixth Annual ACM
  Symposium on Theory of Computing}} (Montreal, Quebec, Canada)
  \emph{(\bibinfo{series}{STOC '94})}. \bibinfo{publisher}{ACM},
  \bibinfo{address}{New York, NY, USA}, \bibinfo{pages}{273--282}.
\newblock
\showISBNx{0-89791-663-8}
\urldef\tempurl%
\url{https://doi.org/10.1145/195058.195155}
\showDOI{\tempurl}


\bibitem[\protect\citeauthoryear{Kullback and Leibler}{Kullback and
  Leibler}{1951}]%
        {kldivergence}
\bibfield{author}{\bibinfo{person}{S. Kullback} {and} \bibinfo{person}{R.~A.
  Leibler}.} \bibinfo{year}{1951}\natexlab{}.
\newblock \showarticletitle{On information and sufficiency}.
\newblock \bibinfo{journal}{\emph{Ann. Math. Statistics}}  \bibinfo{volume}{22}
  (\bibinfo{year}{1951}), \bibinfo{pages}{79--86}.
\newblock
\showISSN{0003-4851}
\urldef\tempurl%
\url{https://doi.org/10.1214/aoms/1177729694}
\showDOI{\tempurl}


\bibitem[\protect\citeauthoryear{Laurent and Massart}{Laurent and
  Massart}{2000}]%
        {ML00}
\bibfield{author}{\bibinfo{person}{B. Laurent} {and} \bibinfo{person}{P.
  Massart}.} \bibinfo{year}{2000}\natexlab{}.
\newblock \showarticletitle{Adaptive estimation of a quadratic functional by
  model selection}.
\newblock \bibinfo{journal}{\emph{Ann. Statist.}} \bibinfo{volume}{28},
  \bibinfo{number}{5} (\bibinfo{year}{2000}), \bibinfo{pages}{1302--1338}.
\newblock
\showISSN{0090-5364}
\urldef\tempurl%
\url{https://doi.org/10.1214/aos/1015957395}
\showDOI{\tempurl}


\bibitem[\protect\citeauthoryear{Lauritzen}{Lauritzen}{1996}]%
        {Lauritzen}
\bibfield{author}{\bibinfo{person}{Steffen~L. Lauritzen}.}
  \bibinfo{year}{1996}\natexlab{}.
\newblock \bibinfo{booktitle}{\emph{Graphical models}}. \bibinfo{series}{Oxford
  Statistical Science Series}, Vol.~\bibinfo{volume}{17}.
\newblock \bibinfo{publisher}{The Clarendon Press, Oxford University Press, New
  York}. x+298 pages.
\newblock
\showISBNx{0-19-852219-3}
\newblock
\shownote{Oxford Science Publications.}


\bibitem[\protect\citeauthoryear{Ling and Xing}{Ling and Xing}{2004}]%
        {codingtheory}
\bibfield{author}{\bibinfo{person}{San Ling} {and} \bibinfo{person}{Chaoping
  Xing}.} \bibinfo{year}{2004}\natexlab{}.
\newblock \bibinfo{booktitle}{\emph{Coding theory: A first course}}.
\newblock \bibinfo{publisher}{Cambridge University Press, Cambridge}. xii+222
  pages.
\newblock
\showISBNx{0-521-52923-9}
\urldef\tempurl%
\url{https://doi.org/10.1017/CBO9780511755279}
\showDOI{\tempurl}


\bibitem[\protect\citeauthoryear{Littlestone and Warmuth}{Littlestone and
  Warmuth}{1986}]%
        {littlestone1986relating}
\bibfield{author}{\bibinfo{person}{Nick Littlestone} {and}
  \bibinfo{person}{Manfred Warmuth}.} \bibinfo{year}{1986}\natexlab{}.
\newblock \bibinfo{booktitle}{\emph{Relating data compression and
  learnability}}.
\newblock \bibinfo{type}{{T}echnical {R}eport}. \bibinfo{institution}{Technical
  report, University of California, Santa Cruz}.
\newblock
\newblock
\shownote{Available at \url{https://users.soe.ucsc.edu/~manfred/pubs/T1.pdf}.}


\bibitem[\protect\citeauthoryear{Litvak, Pajor, Rudelson, and
  Tomczak-Jaegermann}{Litvak et~al\mbox{.}}{2005}]%
        {large_enclosed_ball}
\bibfield{author}{\bibinfo{person}{Alexander~E. Litvak}, \bibinfo{person}{Alain
  Pajor}, \bibinfo{person}{Mark Rudelson}, {and} \bibinfo{person}{Nicole
  Tomczak-Jaegermann}.} \bibinfo{year}{2005}\natexlab{}.
\newblock \showarticletitle{Smallest singular value of random matrices and
  geometry of random polytopes}.
\newblock \bibinfo{journal}{\emph{Adv. Math.}} \bibinfo{volume}{195},
  \bibinfo{number}{2} (\bibinfo{year}{2005}), \bibinfo{pages}{491--523}.
\newblock
\showISSN{0001-8708}
\urldef\tempurl%
\url{https://doi.org/10.1016/j.aim.2004.08.004}
\showURL{%
\tempurl}


\bibitem[\protect\citeauthoryear{Lucic, Faulkner, Krause, and Feldman}{Lucic
  et~al\mbox{.}}{2018}]%
        {lucic2017training}
\bibfield{author}{\bibinfo{person}{Mario Lucic}, \bibinfo{person}{Matthew
  Faulkner}, \bibinfo{person}{Andreas Krause}, {and} \bibinfo{person}{Dan
  Feldman}.} \bibinfo{year}{2018}\natexlab{}.
\newblock \showarticletitle{Training {G}aussian Mixture Models at Scale via
  Coresets}.
\newblock \bibinfo{journal}{\emph{Journal of Machine Learning Research}}
  \bibinfo{volume}{18}, \bibinfo{number}{160} (\bibinfo{year}{2018}),
  \bibinfo{pages}{1--25}.
\newblock
\urldef\tempurl%
\url{http://jmlr.org/papers/v18/15-506.html}
\showURL{%
\tempurl}


\bibitem[\protect\citeauthoryear{Meckes}{Meckes}{2019}]%
        {haar}
\bibfield{author}{\bibinfo{person}{Elizabeth~S. Meckes}.}
  \bibinfo{year}{2019}\natexlab{}.
\newblock \bibinfo{booktitle}{\emph{The random matrix theory of the classical
  compact groups}}. \bibinfo{series}{Cambridge Tracts in Mathematics},
  Vol.~\bibinfo{volume}{218}.
\newblock \bibinfo{publisher}{Cambridge University Press, Cambridge}. xi+212
  pages.
\newblock
\showISBNx{978-1-108-41952-9}
\urldef\tempurl%
\url{https://doi.org/10.1017/9781108303453.009}
\showDOI{\tempurl}


\bibitem[\protect\citeauthoryear{Mitzenmacher and Upfal}{Mitzenmacher and
  Upfal}{2017}]%
        {mitzenmacher}
\bibfield{author}{\bibinfo{person}{Michael Mitzenmacher} {and}
  \bibinfo{person}{Eli Upfal}.} \bibinfo{year}{2017}\natexlab{}.
\newblock \bibinfo{booktitle}{\emph{Probability and computing: Randomization
  and probabilistic techniques in algorithms and data analysis}
  (\bibinfo{edition}{second} ed.)}.
\newblock \bibinfo{publisher}{Cambridge University Press, Cambridge}. xx+467
  pages.
\newblock
\showISBNx{978-1-107-15488-9}


\bibitem[\protect\citeauthoryear{Moitra and Valiant}{Moitra and
  Valiant}{2010}]%
        {moitravaliant}
\bibfield{author}{\bibinfo{person}{Ankur Moitra} {and} \bibinfo{person}{Gregory
  Valiant}.} \bibinfo{year}{2010}\natexlab{}.
\newblock \showarticletitle{Settling the Polynomial Learnability of Mixtures of
  {G}aussians}. In \bibinfo{booktitle}{\emph{Proceedings of the 2010 IEEE 51st
  Annual Symposium on Foundations of Computer Science}}
  \emph{(\bibinfo{series}{FOCS '10})}. \bibinfo{publisher}{IEEE Computer
  Society}, \bibinfo{address}{Washington, DC, USA}, \bibinfo{pages}{93--102}.
\newblock
\showISBNx{978-0-7695-4244-7}
\urldef\tempurl%
\url{https://doi.org/10.1109/FOCS.2010.15}
\showDOI{\tempurl}


\bibitem[\protect\citeauthoryear{Moran and Yehudayoff}{Moran and
  Yehudayoff}{2016}]%
        {moran2016sample}
\bibfield{author}{\bibinfo{person}{Shay Moran} {and} \bibinfo{person}{Amir
  Yehudayoff}.} \bibinfo{year}{2016}\natexlab{}.
\newblock \showarticletitle{Sample compression schemes for {VC} classes}.
\newblock \bibinfo{journal}{\emph{J. ACM}} \bibinfo{volume}{63},
  \bibinfo{number}{3} (\bibinfo{year}{2016}), \bibinfo{pages}{Art. 21, 10}.
\newblock
\showISSN{0004-5411}
\urldef\tempurl%
\url{https://doi.org/10.1145/2890490}
\showDOI{\tempurl}


\bibitem[\protect\citeauthoryear{Rasmussen and Williams}{Rasmussen and
  Williams}{2006}]%
        {Rasmussen}
\bibfield{author}{\bibinfo{person}{Carl~Edward Rasmussen} {and}
  \bibinfo{person}{Christopher K.~I. Williams}.}
  \bibinfo{year}{2006}\natexlab{}.
\newblock \bibinfo{booktitle}{\emph{Gaussian processes for machine learning}}.
\newblock \bibinfo{publisher}{MIT Press, Cambridge, MA}. xviii+248 pages.
\newblock
\showISBNx{978-0-262-18253-9}
\newblock
\shownote{Available at \url{http://www.gaussianprocess.org/gpml/chapters/}.}


\bibitem[\protect\citeauthoryear{Reiss}{Reiss}{1989}]%
        {Reiss}
\bibfield{author}{\bibinfo{person}{Rolf-Dieter Reiss}.}
  \bibinfo{year}{1989}\natexlab{}.
\newblock \bibinfo{booktitle}{\emph{Approximate distributions of order
  statistics with applications to nonparametric statistics}}.
\newblock \bibinfo{publisher}{Springer-Verlag, New York}. xii+355 pages.
\newblock
\showISBNx{0-387-96851-2}
\urldef\tempurl%
\url{https://doi.org/10.1007/978-1-4613-9620-8}
\showURL{%
\tempurl}


\bibitem[\protect\citeauthoryear{Silverman}{Silverman}{1986}]%
        {silverman}
\bibfield{author}{\bibinfo{person}{Bernard~W. Silverman}.}
  \bibinfo{year}{1986}\natexlab{}.
\newblock \bibinfo{booktitle}{\emph{Density estimation for statistics and data
  analysis}}.
\newblock \bibinfo{publisher}{Chapman \& Hall, London}. x+175 pages.
\newblock
\showISBNx{0-412-24620-1}


\bibitem[\protect\citeauthoryear{Sohler and Woodruff}{Sohler and
  Woodruff}{2018}]%
        {sohler2018strong}
\bibfield{author}{\bibinfo{person}{Christian Sohler} {and}
  \bibinfo{person}{David~P. Woodruff}.} \bibinfo{year}{2018}\natexlab{}.
\newblock \showarticletitle{Strong coresets for {$k$}-median and subspace
  approximation: goodbye dimension}.
\newblock In \bibinfo{booktitle}{\emph{59th {A}nnual {IEEE} {S}ymposium on
  {F}oundations of {C}omputer {S}cience---{FOCS} 2018}}.
  \bibinfo{publisher}{IEEE Computer Soc., Los Alamitos, CA},
  \bibinfo{pages}{802--813}.
\newblock
\urldef\tempurl%
\url{https://doi.org/10.1109/FOCS.2018.00081}
\showDOI{\tempurl}


\bibitem[\protect\citeauthoryear{Suresh, Orlitsky, Acharya, and
  Jafarpour}{Suresh et~al\mbox{.}}{2014}]%
        {spherical}
\bibfield{author}{\bibinfo{person}{Ananda~Theertha Suresh},
  \bibinfo{person}{Alon Orlitsky}, \bibinfo{person}{Jayadev Acharya}, {and}
  \bibinfo{person}{Ashkan Jafarpour}.} \bibinfo{year}{2014}\natexlab{}.
\newblock \showarticletitle{Near-Optimal-Sample Estimators for Spherical
  {G}aussian Mixtures}.
\newblock In \bibinfo{booktitle}{\emph{Advances in Neural Information
  Processing Systems 27}}, \bibfield{editor}{\bibinfo{person}{Z.~Ghahramani},
  \bibinfo{person}{M.~Welling}, \bibinfo{person}{C.~Cortes},
  \bibinfo{person}{N.~D. Lawrence}, {and} \bibinfo{person}{K.~Q. Weinberger}}
  (Eds.). \bibinfo{publisher}{Curran Associates, Inc.},
  \bibinfo{pages}{1395--1403}.
\newblock
\urldef\tempurl%
\url{http://papers.nips.cc/paper/5251-near-optimal-sample-estimators-for-spherical-gaussian-mixtures}
\showURL{%
\tempurl}


\bibitem[\protect\citeauthoryear{Tsybakov}{Tsybakov}{2009}]%
        {Tsybakov}
\bibfield{author}{\bibinfo{person}{Alexandre~B. Tsybakov}.}
  \bibinfo{year}{2009}\natexlab{}.
\newblock \bibinfo{booktitle}{\emph{Introduction to nonparametric estimation}}.
\newblock \bibinfo{publisher}{Springer, New York}. xii+214 pages.
\newblock
\showISBNx{978-0-387-79051-0}
\urldef\tempurl%
\url{https://doi.org/10.1007/b13794}
\showDOI{\tempurl}
\newblock
\shownote{Revised and extended from the 2004 French original, Translated by
  Vladimir Zaiats.}


\bibitem[\protect\citeauthoryear{Vapnik and Chervonenkis}{Vapnik and
  Chervonenkis}{1971}]%
        {vapnik2015uniform}
\bibfield{author}{\bibinfo{person}{Vladimir~N. Vapnik} {and}
  \bibinfo{person}{Alexey~Ya. Chervonenkis}.} \bibinfo{year}{1971}\natexlab{}.
\newblock \showarticletitle{On the Uniform Convergence of Relative Frequencies
  of Events to Their Probabilities}.
\newblock \bibinfo{journal}{\emph{Theory of Probability \& Its Applications}}
  \bibinfo{volume}{16}, \bibinfo{number}{2} (\bibinfo{year}{1971}),
  \bibinfo{pages}{264--280}.
\newblock
\urldef\tempurl%
\url{https://doi.org/10.1137/1116025}
\showDOI{\tempurl}


\bibitem[\protect\citeauthoryear{Vershynin}{Vershynin}{2012}]%
        {vershynin_2012}
\bibfield{author}{\bibinfo{person}{Roman Vershynin}.}
  \bibinfo{year}{2012}\natexlab{}.
\newblock \showarticletitle{Introduction to the non-asymptotic analysis of
  random matrices}.
\newblock In \bibinfo{booktitle}{\emph{Compressed sensing}}.
  \bibinfo{publisher}{Cambridge Univ. Press, Cambridge},
  \bibinfo{pages}{210--268}.
\newblock
\newblock
\shownote{Available at \url{https://arxiv.org/abs/1011.3027}.}


\bibitem[\protect\citeauthoryear{Vershynin}{Vershynin}{2018}]%
        {hdp-vershynin}
\bibfield{author}{\bibinfo{person}{Roman Vershynin}.}
  \bibinfo{year}{2018}\natexlab{}.
\newblock \bibinfo{booktitle}{\emph{High-dimensional probability: An
  introduction with applications in data science}}. \bibinfo{series}{Cambridge
  Series in Statistical and Probabilistic Mathematics},
  Vol.~\bibinfo{volume}{47}.
\newblock \bibinfo{publisher}{Cambridge University Press, Cambridge}. xiv+284
  pages.
\newblock
\showISBNx{978-1-108-41519-4}
\urldef\tempurl%
\url{https://doi.org/10.1017/9781108231596}
\showDOI{\tempurl}
\newblock
\shownote{Available at
  \url{https://www.math.uci.edu/~rvershyn/papers/HDP-book/HDP-book.html}.}


\bibitem[\protect\citeauthoryear{Yatracos}{Yatracos}{1985}]%
        {Yatracos}
\bibfield{author}{\bibinfo{person}{Yannis~G. Yatracos}.}
  \bibinfo{year}{1985}\natexlab{}.
\newblock \showarticletitle{Rates of convergence of minimum distance estimators
  and {K}olmogorov's entropy}.
\newblock \bibinfo{journal}{\emph{Ann. Statist.}} \bibinfo{volume}{13},
  \bibinfo{number}{2} (\bibinfo{year}{1985}), \bibinfo{pages}{768--774}.
\newblock
\showISSN{0090-5364}
\urldef\tempurl%
\url{https://doi.org/10.1214/aos/1176349553}
\showDOI{\tempurl}


\bibitem[\protect\citeauthoryear{Yu}{Yu}{1997}]%
        {bin_yu}
\bibfield{author}{\bibinfo{person}{Bin Yu}.} \bibinfo{year}{1997}\natexlab{}.
\newblock \showarticletitle{Assouad, {F}ano, and {L}e {C}am}.
\newblock In \bibinfo{booktitle}{\emph{Festschrift for {L}ucien {L}e {C}am}}.
  \bibinfo{publisher}{Springer, New York}, \bibinfo{pages}{423--435}.
\newblock
\newblock
\shownote{Available at
  \url{https://www.stat.berkeley.edu/~binyu/ps/LeCam.pdf}.}


\end{thebibliography}
\end{document}